\PassOptionsToPackage{hyphens}{url}
\documentclass[sigconf,nonacm]{aamas}

\usepackage{algorithm}
\usepackage{algorithmic}
\usepackage{longtable}
\usepackage{placeins}

\setcopyright{none}

\title{Constrained Multi-Agent Reinforcement Learning in Imagination}

\author{Raz Shmueli}
\email{razshmue@post.bgu.ac.il}
\author{Shahaf S. Shperberg}
\email{shperbsh@bgu.ac.il}
\affiliation{%
  \institution{Ben-Gurion University of the Negev}
  \city{Beer-Sheva}
  \country{Israel}}

\begin{abstract}
Constrained multi-agent reinforcement learning pays for every policy update with
interactions with the real system, which are slow, costly, and can be unsafe.
We show that a world model, a learned simulator of the system, can move most of
this learning into imagination, onto episodes the model simulates.  \emph{Cost-Aware
MAMBA} (CA-MAMBA) extends MAMBA, a multi-agent world-model learner, to predict
each agent's cost, trains per-agent policies on imagined reward and cost,
and prices cost for the whole team with a Lagrange multiplier, a weight on cost
set from the cost of real episodes.  On two benchmarks, one with discrete and one with continuous control, constrained methods without a world model need at least an order of magnitude
more interactions on most tasks to match CA-MAMBA's cost and task performance.  An
exact decomposition of the gap between imagined and real cost, with a bound
that splits across the agents' communication structure, explains what decides
whether a constraint learned in imagination holds on the real system, and why
the price should be set from measured cost.  In controlled tests on one task per
benchmark, weighting cost as much as reward more than halved a cost-blind learner's cost on one task and met a tight limit that the learned price missed on
the other, without losing task performance.  This suggests the model's cost
estimates support tighter limits than our learned price reaches.  CA-MAMBA cuts
the interactions that constrained multi-agent learning needs.

\end{abstract}


\begin{document}

\maketitle
\thispagestyle{plain}
\pagestyle{plain}

\section{Introduction}
Cooperative agents that share a safety cost must learn how much task return to
give up for it.  Existing constrained multi-agent methods learn this trade-off
without a model \cite{DBLP:journals/ai/GuKCDYKY23}, so every policy update is
paid for with interactions with the team's real system.  These interactions are
slow and expensive, and each one can incur the very cost the constraint is meant
to limit.  Learned world models change this accounting.  A world model is a
learned simulator of the environment; once it is accurate, most policy updates
can be taken on imagined rollouts instead of collected ones
\cite{DBLP:conf/iclr/HafnerLB020}, and MAMBA extends this to cooperative agents
that each see only part of the environment \cite{DBLP:conf/atal/EgorovS22}.  A
world model that predicts no cost, however, cannot teach a policy to respect a
constraint.

This paper introduces \emph{Cost-Aware MAMBA} (CA-MAMBA), a multi-agent method
that moves the constraint into the world model and so learns it in imagination.
CA-MAMBA predicts each agent's cost from the model's summary of its history, as
SafeDreamer does for one agent \cite{DBLP:conf/iclr/0008JXZ024}, trains
per-agent policies on imagined reward and cost, and prices cost for the team
with one Lagrange multiplier set from measured cost, where SafeDreamer, training
in imagination, sets it from imagined cost.  Its components are known;
we show what combining them buys, where it falls short, and why.

Learning in imagination raises a question that model-free methods do not face:
does a constraint met in imagination hold in the real system?  We answer it
with an exact decomposition of the imagined-to-real cost gap
(Theorem~\ref{thm:cost-transfer}).  The gap is a sum of two model biases, one in
the predicted cost and one in its next-step prediction of the remaining
cost, both averaged over the states the real system visits, and the bound on its
transition part splits across agents by who communicates with whom.  The
decomposition shapes the design: a safety margin sized from model error would
need error bounds that training data cannot usefully estimate, whereas a price
set from measured cost targets the real constraint and is unaffected by a
uniform rescaling of the predicted cost.

Learning in imagination pays off in interaction efficiency.  We test CA-MAMBA on two benchmarks that
differ in cost structure and action space: sparse, irreversible unit losses in
StarCraft battles under discrete control, and dense, reversible speed violations under continuous
control.  At equal interaction budgets it has lower cost than two model-free
constrained baselines on 23 of 24 discrete-control configurations (pairs of a
task and a cost limit) and a higher win rate, the fraction of battles won, on
22; under continuous control it earns $2.8$ to $12$ times their reward while all
three methods end within the limit.  On the median configuration the stronger of the two
baselines needs ten to twenty times as many interactions to match CA-MAMBA's cost
and win rate.

Meeting the limit itself is harder, for every method.  At the matched budget no
method meets a zero-casualty limit, and CA-MAMBA meets 5 of the 12 nonzero ones,
which the stronger baseline matches only with ten times the interactions.
Controlled experiments on one task per benchmark suggest that the model's cost
estimates are not what limits cost.  At parity, the price at which cost and
reward weigh equally in the policy update, these estimates more than halve the
cost of a learner that ignores cost, at a higher win rate, and on continuous
control they meet a tight limit that our learned price misses, at equal reward.
Our learned price nearly matches parity under the zero-casualty limit and meets
the looser continuous-control limit; under the tight one it rises too slowly,
and a tenfold larger step meets that limit at lower reward.  Too high a price
fails the other way: at five times parity the agents avoid combat and win no
battle.  The paper contributes:
\begin{itemize}
    \item CA-MAMBA, which composes MAMBA with SafeDreamer's cost prediction so
    that per-agent cost predictions form one team constraint, priced from
    measured cost;
    \item an exact decomposition of the imagined-to-real cost gap, extended to
    the policy's returns, whose bound splits across agents to expose
    communication and inter-agent correlation;
    \item evidence on both benchmarks that model-free constrained methods need at
    least an order of magnitude more interactions on most tasks to match
    CA-MAMBA;
    \item controlled multiplier experiments, including a cost-blind learner on both
    benchmarks, that separate the price from the model, and evidence that acting
    without messages can cost reward under continuous control.
\end{itemize}

\section{Related Work}
\label{sec:related}
CA-MAMBA joins constrained multi-agent learning, which has no model, and
multi-agent world models, which have no cost.

\noindent \textbf{Constrained multi-agent learning.}
Constrained reinforcement learning maximizes return subject to a limit on
expected cumulative cost.  Constrained policy optimization enforces the limit in
every policy update while keeping the new policy close to the current one, a
\emph{trust region} \cite{DBLP:conf/icml/AchiamHTA17}.
\citet{DBLP:journals/ai/GuKCDYKY23} carry this to teams with two methods, which
are our baselines: MACPO (multi-agent constrained policy optimization) keeps the
trust-region update, and MAPPO-Lagrangian instead subtracts a learned multiple
of the cost from the reward inside MAPPO, the multi-agent form of proximal
policy optimization (PPO), which keeps updates small by clipping the ratio of
new to old action probabilities.  Other work changes the objective, constraining
peak as well as average cost \cite{DBLP:conf/pkdd/LiuGALYX21} or optimizing the
mean of each agent's worst returns \cite{DBLP:conf/nips/QiuWYWHAOR21};
\citet{DBLP:journals/pami/GuYDCWWK24} survey
the field.  All of these methods learn without a model, so every policy update
is paid for with environment interactions, which are what we aim to reduce.  A
complementary approach, shielding, checks each action before it is executed,
in training and deployment alike, and replaces it when a safety model predicts a
violation \cite{DBLP:conf/atal/XiaoLD23}; its guarantees are only as good as that
model.

\noindent \textbf{Learning in imagination.}
World-model methods supply what these methods lack: they learn the
environment's dynamics from collected experience
and train the policy on trajectories the model imagines, so most policy updates
need no new interaction.  Dreamer learns these dynamics in a compact learned
state space with a recurrent network \cite{DBLP:conf/iclr/HafnerLB020} and now
trains agents across many domains \cite{DBLP:journals/nature/HafnerPBL25}, but
has no notion of cost.  Two single-agent methods add a constraint: LAMBDA uses
the uncertainty of a Bayesian world model to stay pessimistic about cost
\cite{DBLP:conf/iclr/AsUCK22}, and SafeDreamer adds a cost predictor, a cost
critic that estimates future cost, and a Lagrange multiplier to Dreamer; it plans
actions online or trains a policy in imagination, and in the latter its multiplier
follows imagined cost \cite{DBLP:conf/iclr/0008JXZ024}.  For a team, the model must also capture how
agents affect one another.  MAMBA, on which we build, gives each agent its own
dynamics and couples them through attention, a learned weighting with which
each agent reads the other agents' states and actions
\cite{DBLP:conf/atal/EgorovS22}; other designs use interacting local models
\cite{DBLP:conf/aaai/WuYCHZ23} or per-agent transformers, stacks of attention
layers, whose representations are aggregated centrally \cite{DBLP:journals/tmlr/ZhangBZYLL25}.  To
our knowledge, none of these multi-agent models predicts cost.  CA-MAMBA
combines MAMBA with SafeDreamer's cost prediction, which raises questions
neither line answers: how per-agent cost predictions form one team constraint,
how the choice of whom each agent attends to and the correlations between agents
enter the gap between imagined and real cost, which signal should drive the
shared multiplier, and, when the cost limit is missed, whether the model or the
multiplier is at fault.

\noindent \textbf{Benchmarks.}
Testing these questions needs constraints that differ in kind.  For continuous
control we use the multi-agent MuJoCo (MAMuJoCo) velocity tasks
of the Safety-Gymnasium benchmark, in which agents jointly move one simulated
robot and incur cost while it exceeds a speed limit; its library SafePO
supplies our baselines \cite{DBLP:conf/nips/JiZZP0SGZD023}.
For discrete control we add a casualty cost to the StarCraft Multi-Agent
Challenge (SMAC) \cite{DBLP:conf/atal/SamvelyanRWFNRH19}, a standard benchmark
in which agents control allied units fighting the game's scripted enemy.  This
cost is constructed, not standard.

\section{CA-MAMBA}
\label{sec:method}
After stating the problem, this section builds CA-MAMBA in three steps: a world
model that also predicts cost (Section~\ref{sec:world-model}), the messages
through which agents' models interact (Section~\ref{sec:message}), and the
imagined actor (policy) and critics (value estimators) with a price on cost set from measured cost
(Section~\ref{sec:actor-critic}).

\subsection{Constrained Dec-POMDP with Messages}
\label{sec:problem}

We model the team as follows.  Let
\[
\begin{aligned}
\mathcal M=\langle&
\mathcal I,\mathcal S,\{\mathcal A_i\}_{i=1}^{N},P,R,
\{C_i\}_{i=1}^{N},\\
&\{\mathcal O_i\}_{i=1}^{N},\Omega,\mu_0,H,\gamma,d\rangle
\end{aligned}
\]
be a finite-horizon decentralized partially observable Markov decision process
(Dec-POMDP) \cite{DBLP:journals/mor/BernsteinGIZ02} extended with costs.  Here
$\mathcal I=\{1,\ldots,N\}$ is the agent set, $\mathcal S$ the state space,
$\mathcal A_i$ and $\mathcal O_i$ agent $i$'s action and observation spaces,
$P(s_{t+1}\mid s_t,\mathbf a_t)$ the transition distribution, $R(s_t,\mathbf
a_t)$ the shared reward, $C_i(s_t,\mathbf a_t)$ agent $i$'s expected cost,
$\Omega=\{\Omega_i\}_{i=1}^{N}$ the observation distributions
$\Omega_i(o_t^i\mid s_t)$, $\mu_0$ the initial-state distribution, $H$ the
episode horizon, $\gamma\in[0,1]$ the discount, and $d\geq0$ the
team-cost limit.  The joint action is $\mathbf a_t=(a_t^1,\ldots,a_t^N)$.

Agent $i$ has local history $\tau_t^i=(o_0^i,a_0^i,\ldots,a_{t-1}^i,o_t^i)$
and, before acting, may receive a message $m_t^i$ from teammates, whose form is
a design choice (Section~\ref{sec:message}).  The joint policy
factorizes as
\[
\pi_{\phi}(\mathbf a_t\mid\boldsymbol\tau_t,\mathbf m_t)
=\prod_{i=1}^{N}
\pi_{\phi^i}^i(a_t^i\mid\tau_t^i,m_t^i),
\]
where $\phi=(\phi^1,\ldots,\phi^N)$ collects the agents' policy parameters,
$\boldsymbol\tau_t=(\tau_t^1,\ldots,\tau_t^N)$, and
$\mathbf m_t=(m_t^1,\ldots,m_t^N)$.  Execution is decentralized in that no
agent uses the state $s_t$; whether it also needs communication depends on the
message function.

Each agent incurs a realized cost $c_t^i$, and the team cost is their mean
$c_t=\tfrac1N\sum_{i=1}^{N}c_t^i$.  The $c_t^i$ differ when each agent is
charged for its own events, such as its own collisions, and all equal $c_t$
when the environment charges every agent the same team cost.  With $r_t$ the
realized shared reward, the team maximizes expected discounted return subject to
a limit on expected undiscounted episode cost:
\[
\begin{aligned}
\max_{\phi}\;\;& J_R(\pi_\phi)=\mathbb E_{\pi_\phi,P}\Bigl[\textstyle\sum_{t=0}^{H-1}\gamma^t r_t\Bigr]\\
\text{s.t.}\;\;& J_C(\pi_\phi)=\mathbb E_{\pi_\phi,P}\Bigl[\textstyle\sum_{t=0}^{H-1}c_t\Bigr]\leq d.
\end{aligned}
\]

\subsection{Cost-Aware Multi-Agent World Model}
\label{sec:world-model}

To solve the constrained problem of Section~\ref{sec:problem} with few
interactions, CA-MAMBA learns a world model, a learned simulator of the team: once it is accurate, policy updates can use generated
trajectories instead of collected ones.  We build on
the recurrent state-space model (RSSM)
\cite{DBLP:conf/icml/HafnerLFVHLD19,DBLP:conf/iclr/HafnerLB020} and add a
per-agent cost predictor (Figure~\ref{fig:training}).  The RSSM compresses each
agent's history into a latent state, learns to predict the next latent state
without the next observation, and trains that prediction against latent states
inferred from observations stored in a \emph{replay} buffer of collected
episodes; imagination then rolls it forward without observations.

\begin{figure}[t]
  \centering
  \includegraphics[width=\columnwidth]{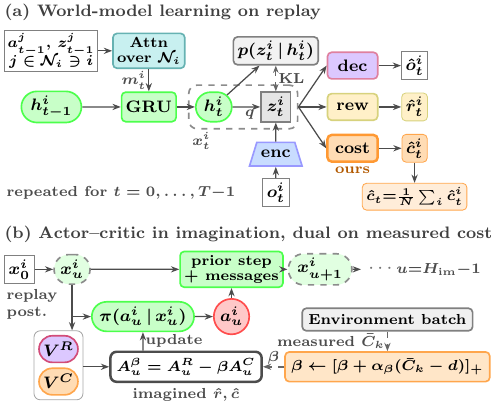}
  \Description{Diagram of the two training phases: world-model learning on
  replay, and actor-critic learning in imagination with a multiplier updated
  from measured cost.}
  \caption{The two training phases for agent $i$; Sections~\ref{sec:world-model}
  to~\ref{sec:actor-critic} define each component (GRU: gated recurrent unit, the recurrent network
  $f_\theta$; enc: the observation encoder; dec: the decoder).
  \textbf{(a)}~On replay, attention over $\mathcal N_i$ turns the previous
  actions and stochastic states into a message $m_t^i$ that drives the
  recurrent update; the posterior combines $h_t^i$ with the encoded
  observation, and a Kullback--Leibler (KL) penalty ties the prior, which predicts $z_t^i$ from $h_t^i$ alone, to it.  The
  decoder reads $x_t^i=(h_t^i,z_t^i)$; the reward and cost heads read $z_t^i$ with
  $h_{t+1}^i$, which carries the agents' actions $a_t$, and the per-agent costs
  are averaged into $\hat c_t$.  \textbf{(b)}~From replay posteriors, the actor
  (policy) and prior roll out $H_{\mathrm{im}}$ imagined steps; the critics
  (value estimators) form $A_u^\beta$ for the actor, while the price $\beta$ steps
  on the \emph{measured} cost $\bar C_k$ of the latest real episode.}
  \label{fig:training}
\end{figure}

\noindent \textbf{Latent state.}
Agent $i$'s latent state $x_t^i=(h_t^i,z_t^i)$ has two parts.  The
\emph{deterministic} state $h_t^i$ is the hidden state of a recurrent network
and accumulates history,
\[
h_t^i=f_\theta\!\left(h_{t-1}^i,m_t^i\right),
\]
where the message $m_t^i$ of Eq.~\eqref{eq:message} carries the previous action
and stochastic state of agent $i$ and of the teammates it attends to.  The
\emph{stochastic} state $z_t^i$ is sampled at each step and represents what
history alone does not determine; as in DreamerV2, Dreamer's discrete-latent version, it is a vector
of categorical variables \cite{DBLP:conf/iclr/HafnerL0B21}.  Together, $x_t^i$
summarizes agent $i$'s history, so the actor and critics read $x_t^i$ instead
of $\tau_t^i$.

\noindent \textbf{Prior and posterior.}
Imagination must predict $z_t^i$ without an observation, so the model keeps two
distributions over it.  The \emph{posterior}
$q_\theta(z_t^i\mid h_t^i,o_t^i)$ also sees the current observation, through an
observation encoder, so it is available only on replay.  The \emph{prior}
$p_\theta(z_t^i\mid h_t^i)$ predicts $z_t^i$ from history alone; imagined steps
have no observation, so imagination samples from the prior.

\noindent \textbf{Heads.}
From the latent state, small networks called heads predict the observable
quantities that training and imagination need.  The observation head, or decoder, reconstructs $o_t^i$ from
$x_t^i$; it is unused in imagination and only trains the latent to retain what
the observation shows.  The reward and cost heads read $z_t^i$ together with the
next deterministic state $h_{t+1}^i$.  That state already contains the action
$a_t^i$ and, through the message, the actions of the teammates it attends to,
so when it attends to all of them each head sees the joint action that team reward and cost
depend on.  The cost head is our addition.  Without it the latent has no
incentive to encode cost, and a cost critic would read a representation blind
to the quantity it estimates.  Each cost head regresses its own $c_t^i$ and
each reward head the shared $r_t$; imagination uses the team means
$\hat c_t=\tfrac1N\sum_{i}\hat c_t^i$ and $\hat r_t=\tfrac1N\sum_{i}\hat r_t^i$.

\noindent \textbf{Objective.}
Training fits the heads to recorded data and pulls the prior toward the
posterior.  For a length-$T$ replay sequence, write
$q_t^i=q_\theta(z_t^i\mid h_t^i,o_t^i)$ and
$p_t^i=p_\theta(z_t^i\mid h_t^i)$.  The parameters $\theta$ minimize
\begin{align}
\mathcal L_{\mathrm{wm}}(\theta)
={}&\sum_{t=0}^{T-1}\sum_{i=1}^{N}
\bigl(\mathcal L_o(o_t^i,\hat o_t^i)
 +\mathcal L_r(r_t,\hat r_t^i)
 +\mathcal L_c(c_t^i,\hat c_t^i)\bigr) \nonumber\\
&+\omega_{\mathrm{KL}}\sum_{t=0}^{T-1}\sum_{i=1}^{N}
D_{\mathrm{KL}}\!\left(q_t^i\Vert p_t^i\right).
\label{eq:world-model-loss}
\end{align}
This is the RSSM training objective, derived from a tractable lower bound on the
log-likelihood of the recorded sequence
\cite{DBLP:conf/icml/HafnerLFVHLD19,DBLP:conf/iclr/HafnerLB020}, with each
log-likelihood term replaced by a smooth-$\ell_1$ regression loss on
observation, reward, and cost ($\mathcal L_o,\mathcal L_r,\mathcal L_c$:
squared error for small residuals, absolute error for large ones).  The
Kullback--Leibler (KL) divergence $D_{\mathrm{KL}}$, weighted by
$\omega_{\mathrm{KL}}\geq0$, penalizes the gap between prior and posterior, so
that rollouts from the prior alone stay close to the latent states the
posterior infers from real observations.  KL balancing
\cite{DBLP:conf/iclr/HafnerL0B21} weights the part of this penalty that moves
the prior more heavily than the part that moves the posterior, so the
posterior is not pulled toward a poor prior.  Three further terms come from
MAMBA: heads that predict whether the episode continues
and which actions are available, and, for discrete actions, a classifier that
recovers the agent's action from its latent state (Appendix~\ref{app:settings}).

\subsection{The Message Function and Its Neighborhood}
\label{sec:message}

The world model couples agents only through the message $m_t^i$, which reads
the model features of a chosen set of teammates:
\begin{equation}
m_t^i=g_\theta\!\left(i,\bigl(j,\,a_{t-1}^j\oplus z_{t-1}^j\bigr)_{j\in\mathcal N_i}\right),
\quad i\in\mathcal N_i\subseteq\mathcal I,
\label{eq:message}
\end{equation}
where $\mathcal N_i$ is agent $i$'s \emph{attention neighborhood} and $\oplus$
concatenates.  The function $g_\theta$ is a Transformer encoder
\cite{DBLP:conf/nips/VaswaniSPUJGKP17}: a stack of attention layers, each of
which updates every agent's input vector by a learned weighting of the other
agents' vectors.  Agent $j$'s input is an embedding of
$a_{t-1}^j\oplus z_{t-1}^j$ plus a fixed sine--cosine code of its index $j$,
and agent $i$ may weight only the agents in $\mathcal N_i$.  The message
$m_t^i$ is the encoder's output at agent $i$'s own position, so no averaging
over agents is needed.  The encoder sits inside the recurrent transition and is
the only input of $f_\theta$ besides $h_{t-1}^i$, as in MAMBA \cite{DBLP:conf/atal/EgorovS22}.
Agents exchange their previous action and stochastic state, not their current
action, so an agent can condition on what its teammates did and believe but not
on what they are about to do.  The actor, one multilayer perceptron shared by
all agents, reads each $x_t^i$ with no attention of its own, so $\mathcal N_i$
is the \emph{only} channel between agents.

Choosing $\mathcal N_i$ changes only whom each agent may attend to.  The
\emph{full} neighborhood $\mathcal N_i=\mathcal I$ needs a reliable channel at
deployment; the \emph{self-only} neighborhood $\mathcal N_i=\{i\}$ needs no
communication, and nearest-neighbor choices lie in between.  The neighborhood
used when acting can differ from the one used to learn the model and to
imagine, which lets a policy trained with messages act without them
(Appendix~\ref{app:execution} diagrams the execution path).

\subsection{Imagined Actor--Critic and Measured-Cost Multiplier}
\label{sec:actor-critic}

With the model and its messages in place, the policy is trained in imagination
under a price on cost.  We relax the constraint $J_C(\pi_\phi)\leq d$ with a multiplier $\beta\geq0$
and train against the Lagrangian $J_R-\beta(J_C-d)$, where $\beta$ is the
price of cost.  This is the standard Lagrangian treatment of constrained
problems \cite{DBLP:conf/icml/StookeAA20}, applied to a world model as
in SafeDreamer \cite{DBLP:conf/iclr/0008JXZ024}.  With factorized policies the
relaxation need not be tight: the best policy at the best price can still miss
the constrained optimum in cooperative multi-agent learning
\cite{DBLP:conf/iclr/ChenZH24}.  We use it as a training signal and claim
nothing that depends on its being tight.  One multiplier serves the team
(Algorithm~\ref{alg:ca-mamba}) because the problem of
Section~\ref{sec:problem} has one constraint, on the team cost $J_C$;
per-agent constraints would need one multiplier each.

\begin{algorithm}[t]
\caption{CA-MAMBA training loop (quantities defined below)}
\label{alg:ca-mamba}
\begin{algorithmic}[1]
\FOR{$k=0,1,2,\ldots$}
\STATE Add the next episode finished by one of the $W$ workers (each acts
with parameters last sent to it and messages over $\mathcal N_i$) to replay;
$\bar C_k\gets$ its team cost; skip the rest in warm-up
\STATE Update $\theta$ on replay sequences by
Eq.~\eqref{eq:world-model-loss}
\FOR{$e=1,\ldots,E$}
\STATE Freeze $\theta$; from replay posteriors, roll out $H_{\mathrm{im}}$
imagined steps under $\pi_\phi$; $\hat r_u,\hat c_u$ average the agents' heads
\STATE Compute $A_u^{\beta,i}=A_u^{R,i}-\beta A_u^{C,i}$; for several passes,
step $\phi$ on $\mathcal J_{\mathrm{clip}}$ and $V^R,V^C$ on $\lambda$-returns
per minibatch
\STATE $\beta\gets[\beta+\alpha_\beta(\bar C_k-d)]_+$, using \emph{measured}
cost only
\ENDFOR
\STATE Send the updated $\theta,\phi$ to the worker that finished
\ENDFOR
\end{algorithmic}
\end{algorithm}

\noindent \textbf{Critics.}
The policy update needs critics, which estimate the return expected from a
state.  From posterior states
sampled from replay, the frozen model rolls out imagined joint latents
$\mathbf x_0,\ldots,\mathbf x_{H_{\mathrm{im}}}$, with
$\mathbf x_u=(x_u^1,\ldots,x_u^N)$.  The first $H_{\mathrm{im}}$ steps,
$u\in\{0,\ldots,H_{\mathrm{im}}-1\}$, carry imagined reward and cost; beyond
them, the critics' own estimate at $\mathbf x_{H_{\mathrm{im}}}$ stands in for
the rest of the return.  One critic network, an attention layer over all agents'
latent states feeding two output heads, estimates for each agent the imagined
reward return $V^R$ and team-cost return $V^C$; separate heads let a changed
$\beta$ reweight two estimates that stay valid.  Both heads
are trained on $\lambda$-returns: averages of the imagined returns of every
length up to the horizon, each completed by the critic, with weights decaying
geometrically in $\lambda\in[0,1]$.  Generalized advantage
estimation \cite{DBLP:conf/iclr/SchulmanMLJA15} turns the critics' prediction
errors into per-agent \emph{advantages} $A_u^{R,i}$ and $A_u^{C,i}$: how much
more reward, or cost, an action led to than the critic expected.  Each is
standardized to zero mean and unit standard deviation over the batch, and the
actor advantage
\[
A_u^{\beta,i}=A_u^{R,i}-\beta A_u^{C,i}
\]
favors an action only if its reward advantage outweighs its cost advantage at
the current price.  After standardization, $\beta$ weights one standard
deviation of cost advantage against one of reward advantage, so $\beta<1$
weights cost below reward.

\noindent \textbf{Actor.}
The actor turns these advantages into a policy update.  Agent $i$ draws $a_u^i$
from the shared actor $\pi_\phi(a_u^i\mid x_u^i)$, so $\phi^1=\cdots=\phi^N$.  Let
$\phi_{\mathrm{old}}$ be the parameters that generated the imagined batch,
$\rho_u^i(\phi)=\pi_\phi(a_u^i\mid x_u^i)/
\pi_{\phi_{\mathrm{old}}}(a_u^i\mid x_u^i)$ agent $i$'s importance ratio,
$\varepsilon>0$ the clipping parameter, and
$\bar\rho_u^i(\phi)=
\operatorname{clip}(\rho_u^i(\phi),1-\varepsilon,1+\varepsilon)$ with
$\operatorname{clip}(v,l,h)=\min\{h,\max\{l,v\}\}$.  The actor
\emph{maximizes} the objective of proximal policy optimization (PPO)
\cite{DBLP:journals/corr/SchulmanWDRK17},
\[
\mathcal J_{\mathrm{clip}}(\phi)
=\mathbb E_{u,i}\!\left[
\min\!\left\{\rho_u^i(\phi)A_u^{\beta,i},
\bar\rho_u^i(\phi)A_u^{\beta,i}\right\}\right],
\]
over imagined steps and agents, plus an entropy bonus that rewards randomness in
the policy to keep it exploring.  Clipping keeps each update near
$\pi_{\phi_{\mathrm{old}}}$ while several gradient steps reuse one imagined
batch; it bounds the policy step, not the model error.

\noindent \textbf{Multiplier.}
The price $\beta$ is the one quantity set from measured rather than imagined
data.  After a warm-up that collects data before any
update, learner step $k$ starts
when one of the $W$ workers, parallel processes that collect episodes, finishes an episode, and $\bar C_k$ is that
episode's undiscounted team cost.  Its worker acted with the parameters
received when its previous episode ended, $\phi_{k-\ell_k}$ for some lag
$\ell_k\geq0$, so $\bar C_k$ is a one-episode estimate of
$J_C(\pi_{\phi_{k-\ell_k}})$, biased only by the order in which episodes finish
(Section~\ref{sec:theory}).  After each of the $E$ policy-update rounds of the
step (each imagines a fresh replay batch and runs several passes over it), the
multiplier takes the projected step
$\beta\gets[\beta+\alpha_\beta(\bar C_k-d)]_+$ with $[v]_+=\max\{0,v\}$ and
step size $\alpha_\beta>0$.  For $\beta\geq0$, repeating this step $E$ times
with the same $\bar C_k$ gives, per learner step,
\begin{equation}
\beta_{k+1}=
\left[\beta_k+E\alpha_\beta(\bar C_k-d)\right]_+ :
\label{eq:dual-update}
\end{equation}
the price rises while the team overspends, falls while it underspends, and stops
at zero.  The actor ascends $\mathcal J_{\mathrm{clip}}$, built from
discounted, standardized \emph{imagined} cost advantages, while the multiplier
reads undiscounted \emph{measured} episode cost.  Even with an exact model these
are not gradients of one objective, so the procedure is an inexact
\emph{primal--dual} method: it alternates policy (primal) updates with price
(dual) updates.  Section~\ref{sec:theory} analyzes the two statistics
separately: Theorem~\ref{thm:cost-transfer} covers the discounted cost the actor
sees, and the dual analysis the episode cost the multiplier measures.

\section{What Model Error Permits}
\label{sec:theory}
CA-MAMBA's actor never sees real cost: it improves the policy on cost that the
world model imagines (Section~\ref{sec:actor-critic}), while the constraint of
Section~\ref{sec:problem} is on real cost.  This raises three questions, which
this section answers in turn.  First, when a policy looks safe in imagination,
how far can its real cost be from its imagined cost, and what does the
difference depend on?  Second, how does the team structure of the model,
per-agent dynamics coupled through messages, enter that difference?  Third,
since the difference cannot be bounded in practice, what does setting the price
from measured cost achieve, and what can it not repair?  We state the main
results here, summarize the others, and prove all of them in
Appendix~\ref{app:theory_proofs}.

\noindent \textbf{The gap between imagined and real cost.}
Consider running one policy $\pi$ twice from the same start for the
$H_{\mathrm{im}}$ steps the actor imagines: once on the real system and once in
the world model.  The difference between the two discounted costs is what the
actor gets wrong about safety.  We compare the two runs in the latent space of
Section~\ref{sec:world-model}.  Let $x_t=(x_t^1,\ldots,x_t^N)$
be the joint latent state, with $x_t^i=(h_t^i,z_t^i)$, and $a_t$ the joint
action drawn from $\pi(a\mid x)=\prod_{i}\pi^i(a^i\mid x^i)$.  On the real
system the latents are computed from observations by the posterior; let
$P_t(\cdot\mid x,a)$ be the resulting distribution of $x_{t+1}$ given
$(x_t,a_t)=(x,a)$, and $c_t(x,a)$ the mean of the team cost $c_t$ given that pair.  In imagination the
world model replaces both: $\widehat P(\cdot\mid x,a)$ is the distribution of
the next latent under the recurrent update and the prior, and $\widehat
c(x,a)$ is the team mean of the cost heads.  Because $P_t$ is defined on the
network's own latents, the results hold whether or not the latents retain all
of the history that predicts the future, which a learned model cannot
guarantee.

\begin{theorem}[Imagination-to-reality cost gap]
\label{thm:cost-transfer}
Fix the policy $\pi$, a horizon $H_{\mathrm{im}}\geq 1$, and a discount
$\gamma\in[0,1]$.  Let the latent space $\mathcal{X}$ be a standard Borel space (regular enough
for conditional distributions to exist),
let both processes draw $x_0$ from the same distribution and their actions from
$\pi$, and let $c_t$ and $\widehat c$ take values in a common interval of length
$c_{\max}<\infty$.  Define the real and imagined discounted costs over
$H_{\mathrm{im}}$ steps as
\[
\begin{aligned}
J_{C,H_{\mathrm{im}}}(\pi)
&=
\mathbb{E}\!\left[
\sum_{t=0}^{H_{\mathrm{im}}-1}\gamma^t c_t(x_t,a_t)\right],\\
\widehat J_{C,H_{\mathrm{im}}}(\pi)
&=
\mathbb{E}_{\widehat P,\pi}\!\left[
\sum_{t=0}^{H_{\mathrm{im}}-1}\gamma^t\widehat c(x_t,a_t)\right],
\end{aligned}
\]
the first along the real process.  Let $\nu_t$ be the distribution of
$(x_t,a_t)$ on the real system, and $\widehat V_k(x)$ the model's expected
discounted cost over $k$ steps from $x$ under $\pi$ (its cost-to-go), with $\widehat V_0\equiv0$.
\begin{enumerate}
\item[(i)] The gap decomposes exactly as
\[
J_{C,H_{\mathrm{im}}}(\pi)-\widehat J_{C,H_{\mathrm{im}}}(\pi)
=\sum_{t=0}^{H_{\mathrm{im}}-1}\gamma^t\bigl(B_c(t)+\gamma B_V(t)\bigr),
\]
where $B_c(t)=\mathbb{E}_{\nu_t}[c_t-\widehat c\,]$ is the cost head's bias and
\[
B_V(t)=\mathbb{E}_{(x,a)\sim \nu_t}\Bigl[
\bigl(\mathbb{E}_{P_t(\cdot\mid x,a)}-\mathbb{E}_{\widehat P(\cdot\mid x,a)}\bigr)
\widehat V_{H_{\mathrm{im}}-t-1}\Bigr]
\]
is the model's one-step bias in predicting the cost-to-go.
\item[(ii)] If, for every $t<H_{\mathrm{im}}$, $|B_c(t)|\leq\epsilon_c$ and the
mean total variation (the largest difference in the probability two distributions give one event)
$\mathbb{E}_{(x,a)\sim \nu_t}\bigl[D_{\mathrm{TV}}\bigl(P_t(\cdot\mid
x,a),\widehat P(\cdot\mid x,a)\bigr)\bigr]\leq\epsilon_P$, then
\begin{equation}
\label{eq:conditional_cost_transfer}
\begin{aligned}
\left|J_{C,H_{\mathrm{im}}}(\pi)-\widehat J_{C,H_{\mathrm{im}}}(\pi)\right|
\leq \Delta_{C,H_{\mathrm{im}}},\\
\Delta_{C,H_{\mathrm{im}}}
:=
\epsilon_c\sum_{t=0}^{H_{\mathrm{im}}-1}\gamma^t +
c_{\max}\epsilon_P\sum_{t=0}^{H_{\mathrm{im}}-1}t\gamma^t ,
\end{aligned}
\end{equation}
and $\widehat J_{C,H_{\mathrm{im}}}(\pi)\leq b-\Delta_{C,H_{\mathrm{im}}}$
implies $J_{C,H_{\mathrm{im}}}(\pi)\leq b$ for any $b$.
\end{enumerate}
\end{theorem}

\noindent Here $D_{\mathrm{TV}}(p,q)=\sup_Y|p(Y)-q(Y)|$ over events $Y$;
$\epsilon_c,\epsilon_P\geq0$ are the average cost and transition errors.  The proof
(Appendix~\ref{app:proof-gap}) adapts the simulation lemma
\cite{DBLP:journals/ml/KearnsS02}, which bounds how one-step errors
accumulate, to the cost-to-go along the real trajectory.

Part~(i) says that the gap is exactly the sum, over the imagined steps, of two
errors the model makes on the states the real system actually visits: an error
in the cost of the current step ($B_c$, the bias of the cost heads) and an
error in where the team goes next, counted only through its effect on the
remaining cost ($B_V$, the bias of the prior).  Both are averaged where the real
system goes, though $B_V$ weighs them by the model's cost-to-go, and errors of
opposite sign cancel.
Part~(ii) turns bounds on the average errors into a safety margin: a policy
whose imagined cost is at least $\Delta_{C,H_{\mathrm{im}}}$ below a threshold
also has real $H_{\mathrm{im}}$-step cost at or below it, without worst-case bounds on every state--action
pair.  At our $\gamma=0.99$ and $H_{\mathrm{im}}=14$ the margin is
$13.1\epsilon_c+83.2c_{\max}\epsilon_P$, so a transition error weighs far more than
a cost-head error.  The actor's $\lambda$-returns do not stop at step
$H_{\mathrm{im}}$ but end in the cost critic's estimate;
Corollary~\ref{app:cor:bootstrap} in Appendix~\ref{app:proof-bootstrap} extends
part~(i) to them, adding one term, the critic's bias at the state where each
return is cut off.

\noindent \textbf{Where the agents enter.}
The transition error $\epsilon_P$ of part~(ii) belongs to the team's joint
model, but the world model of Section~\ref{sec:world-model} samples each
agent's next stochastic state independently given the current joint state and
action, so the error splits across agents.  With the full or self-only
neighborhood of Section~\ref{sec:message}, agent $i$'s next state depends only
on the states and actions in $\mathcal N_i$, written
$x^{\mathcal N_i}=(x^j)_{j\in\mathcal N_i}$ and likewise $a^{\mathcal N_i}$; a
general restriction on whom each agent attends to would replace $\mathcal N_i$
by the agents it reaches through the encoder's layers.

\begin{proposition}[Factorized transition error]
\label{prop:factorized}
Let $\widehat P(\cdot\mid x,a)=\bigotimes_{i=1}^N\widehat P_i(\cdot\mid
x^{\mathcal N_i},a^{\mathcal N_i})$, so agents' next states are drawn
independently, each from a distribution $\widehat P_i$ over agent $i$'s
states, and let $P_{t,i}(\cdot\mid x,a)$ be the $i$-th marginal of
$P_t(\cdot\mid x,a)$.  Then
\[
D_{\mathrm{TV}}(P_t,\widehat P)\leq
D_{\mathrm{TV}}\Bigl(P_t,\textstyle\bigotimes_{i}P_{t,i}\Bigr)
+\sum_{i=1}^{N}D_{\mathrm{TV}}(P_{t,i},\widehat P_i)
\]
at every $(x,a)$, so Theorem~\ref{thm:cost-transfer}(ii) holds with $\epsilon_P=\eta+\sum_i\epsilon_P^i$, where $\eta$ and
$\epsilon_P^i$ are the largest averages over $\nu_t$ ($t<H_{\mathrm{im}}$) of the
first term and the $i$-th summand.  For a fixed state space, enlarging
$\mathcal N_i$ cannot raise the smallest $\epsilon_P^i$ that a model with inputs
$(x^{\mathcal N_i},a^{\mathcal N_i})$ can approach.  With the mean team cost,
$B_c(t)$ averages the agents' cost-head biases.
\end{proposition}

\noindent The proposition, proved in Appendix~\ref{app:proof-factorized}, says
that the team's transition error is at most the sum of one error per agent plus
a team term $\eta$.  Each $\epsilon_P^i$ depends on whom agent $i$ attends to,
so communication enters the safety margin through it.  The team term $\eta$
measures correlations between
agents' next states that a world model sampling agents independently, MAMBA's
included, cannot represent: exact per-agent predictors leave a team error of
exactly $\eta$.

\noindent \textbf{How large the gap is in practice.}
Whether part~(ii) can certify a policy's real cost depends on the error levels,
which we estimate on tasks of Section~\ref{sec:experiments} (\texttt{8m} is a
StarCraft battle, HalfCheetah a simulated robot).  Replay data stand in for $\nu_t$, and $c_{\max}$ is about $1$ for
MAMuJoCo's binary velocity cost and the team size $N$ for SMAC casualties.  On
MAMuJoCo at $d=25$ the cost head's average bias is estimated at no more than
$0.004$ per step (Appendix~\ref{app:cost-accuracy}), so its term in
$\Delta_{C,H_{\mathrm{im}}}$ is about $0.05$.  The transition term cannot be
bounded usefully from logs: through Pinsker's inequality
$D_{\mathrm{TV}}\leq\sqrt{D_{\mathrm{KL}}/2}$, the prior--posterior divergence,
$3.2$ to $6.5$ nats (natural-log units) per agent, bounds each agent's
transition error only by a number above $1$, no better than the trivial bound
(Appendix~\ref{app:model-error}), and the divergence is positive even for a
perfect model, since the posterior also sees the next observation.  Part~(ii) is
therefore uninformative here: any $\epsilon_P$ above
$0.16$ makes $\Delta_{C,H_{\mathrm{im}}}$ exceed
$c_{\max}\sum_{t<14}\gamma^t\approx13.1\,c_{\max}$, the largest gap any model
could produce.  The gap itself can be measured: on $18$ runs
(Appendix~\ref{app:model-error}), we compare
realized and imagined cost over the $14$-step window, a proxy for the left-hand
side of part~(i).  The model under-predicts in $17$ of the $18$ runs, by $33\%$
of realized cost on \texttt{8m} at $d=0$ and by $35\%$ on HalfCheetah
$2{\times}3$ at $d=25$ under the learned multiplier; the exception is a run with
near-zero realized cost.  The cost heads explain little of this (their replay
bias $\hat B_c$ is under $1\%$ of realized cost on \texttt{8m}); the prior's bias
has the matching sign on \texttt{8m} (Appendix~\ref{app:model-error}).  In short, the safety margin of
part~(ii) cannot be computed, and under the learned multiplier the model
under-predicts real cost by about a third.

\noindent \textbf{Why the multiplier reads measured cost.}
The measured gap answers the third question.  A multiplier driven by imagined
cost, as SafeDreamer's is when its policy trains in imagination \cite{DBLP:conf/iclr/0008JXZ024}, would
enforce $\widehat J_{C,H_{\mathrm{im}}}\leq b$ for some threshold $b$, and by
part~(i) the real cost would then miss $b$ by the bias sum, which varies across
configurations and whose measured proxy is positive in 17 of our 18 runs: such a multiplier would
under-enforce the real constraint, and correcting it would need the margin that
part~(ii) cannot supply.  CA-MAMBA's multiplier instead reads the measured
episode cost (Section~\ref{sec:actor-critic}), so if its update settles, it settles where
the real constraint holds, whatever the sign of the model's bias
(Proposition~\ref{app:prop:dual}, proved in Appendix~\ref{app:proof-dual}); whether it settles depends on the model.  The model still sets the actor's
direction.  Because each
advantage is standardized over the batch, a change $A_u^{C,i}\mapsto
sA_u^{C,i}+\delta$ with one $s>0$ and one $\delta$ for the whole batch leaves the
actor's step unchanged; a positive rescaling of $\widehat c$ is such a change
once the cost critic has rescaled with it.  Other errors generally change the step,
including one that keeps each state's ranking of actions but rescales states
differently, and Appendix~\ref{app:counterexample} gives one that no multiplier
repairs.  Section~\ref{sec:dual-ablation} tests the model's cost advantages.

\noindent \textbf{What the measured-cost multiplier guarantees.}
With the true Lagrangian
$\mathcal{L}(\phi,\beta)=J_R(\pi_\phi)-\beta(J_C(\pi_\phi)-d)$ of
Section~\ref{sec:actor-critic}, $\partial\mathcal{L}/\partial\beta=-(J_C-d)$, so
the update of Eq.~\eqref{eq:dual-update} is projected stochastic \emph{descent}
in $\beta$: a noisy gradient step, clipped at zero, on the cost of the policy
that collected the episode.  Episodes reach the learner in the order they
finish, which can depend on their cost, so a single step is not exactly
unbiased; but if an episode's expected cost given the past is $J_C$ of its
starting parameters, the summed bias over $K$ steps is at most $2(W-1)C_{\max}$,
where $C_{\max}$ bounds episode cost (under the scheduling assumptions of Appendix~\ref{app:proof-dual}).  Proposition~\ref{app:prop:dual} then shows that
the average measured cost over $K$ steps exceeds $d$ by at most
$(\beta_K-\beta_0)/(E\alpha_\beta K)$, so if $\beta$ grows slower than
linearly, the long-run average cost, measured and expected, meets the limit.
Any fixed point of the mean update, the projected step taken with the expected
cost, satisfies $J_C\leq d$, with $\beta>0$
only if the constraint holds with equality (complementary slackness).  None of
this bounds the number of violating episodes or proves the iterates converge,
and a cost head blind to an unsafe action can leave no point where the loop settles:
Appendix~\ref{app:counterexample} gives a one-step problem where $\beta$ diverges
while the policy stays unsafe, because the actor has no cost-reducing
direction.  Appendix~\ref{app:beta} shows the loop closing as prescribed where
the cost head is accurate.

\section{Experiments}
\label{sec:experiments}
With the transfer question analyzed, we turn to what CA-MAMBA achieves.  The
evaluation asks three questions, which Sections~\ref{sec:smac}--\ref{sec:nocomm} answer in turn.  \emph{RQ1:} how many environment
interactions does CA-MAMBA save relative to model-free constrained baselines, in
cost and in task performance?  \emph{RQ2:} when CA-MAMBA misses its cost limit, is
the world model or the multiplier responsible?  \emph{RQ3:} what does a
policy optimized with messages lose without them?  Appendices~\ref{app:algorithm}
and~\ref{app:settings} give the training procedure, every hyperparameter, and
every loss; the code and the scripts behind the main tables and figures will be released once the paper is accepted.

\noindent \textbf{Environments and metrics.}
We answer the three questions on the two benchmarks of Section~\ref{sec:related}.  In the StarCraft Multi-Agent Challenge (SMAC)
\cite{DBLP:conf/atal/SamvelyanRWFNRH19} each allied unit is one agent that sees
only a local field of view and picks a discrete move, attack, or stop action
against the game's scripted enemy.  A \emph{map} is a fixed battle scenario
whose name lists each side's units by count and type: \texttt{8m} pits eight
Marines (basic infantry) against eight, and \texttt{3s5z\_vs\_3s6z} three
Stalkers (ranged) and five Zealots (melee) against three and six; \texttt{c}
stands for Colossus, and \texttt{bane\_vs\_bane} is a battle of Zerglings and
Banelings.  Every agent is charged the team cost: the number of allied units
lost in a step, so a unit that sacrifices itself counts against the whole team.
An episode that ends with enemies alive also charges the units still alive,
unless a unit died on that final step.  Episode cost thus counts lost units, plus
survivors when enemies remain: a sparse, discrete, irreversible signal.  A \emph{configuration} is a map paired with a
cost limit $d$.  We report \emph{win rate}, the fraction of battles won, and
cost.  MAMuJoCo partitions a simulated
robot, the four-legged Ant or the two-legged HalfCheetah, among agents that each
actuate and observe a disjoint subset of its joints, so Ant $4{\times}2$ means
four agents of two joints each; the team cost, shared by all agents, accrues
while the body's speed exceeds a threshold \cite{DBLP:conf/nips/JiZZP0SGZD023}.
Because both benchmarks charge every agent the same cost, the team mean of
Section~\ref{sec:problem} equals each agent's cost.  One \emph{interaction} is one joint
environment step.  Each metric is a training-time mean over the episodes in a
window before a checkpoint, averaged over three runs (two if marked), with
sample standard deviations.  The window is CA-MAMBA's last 5k interactions; on SMAC the
baselines log per rollout batch, with wider windows
(Appendix~\ref{app:full-results}).

\noindent \textbf{Baselines.}
We compare with MACPO \cite{DBLP:journals/ai/GuKCDYKY23}, in a discrete-action
adaptation for SMAC (settings in Appendix~\ref{app:baseline-hyperparams}), and
with MAPPO-Lagrangian (MAPPO-Lag in tables and figures), which replaces MACPO's
trust-region step with clipped gradient steps on a Lagrangian.
Both come from the SafePO code base
\cite{DBLP:conf/nips/JiZZP0SGZD023} and run on both benchmarks.  We read each
baseline at the matched budget and at larger budgets from the same runs.  SafeDreamer
is single-agent; in its place, Section~\ref{sec:dual-ablation} runs CA-MAMBA with
the price set from imagined cost, as SafeDreamer sets it.

\subsection{Interaction Efficiency on SMAC}
\label{sec:smac}

We answer RQ1 first on SMAC.  At the matched budget of 100k interactions,
CA-MAMBA has lower mean cost than
both MACPO and MAPPO-Lagrangian on 23 of the 24 configurations and a higher win
rate on 22 (Appendix~\ref{app:full-results}).  On \texttt{3s5z\_vs\_3s6z} at $d=0$
all three methods incur the full cost of $8$; on \texttt{bane\_vs\_bane} at $d=0$
CA-MAMBA trades win rate for far lower cost.  Configurations on one map are not
independent, so we also test per map: CA-MAMBA is ahead of MAPPO-Lagrangian on
every configuration of 11 of the 12 maps in cost and of 10 in win rate
(one-sided sign tests, which compare these counts with fair coin flips:
$p=0.003$ and $p=0.02$).

\begin{figure}[t]
  \centering
  \includegraphics[width=\columnwidth]{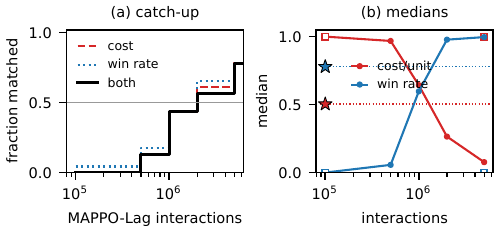}
  \Description{Two panels. Left: fraction of SMAC configurations on which
  MAPPO-Lagrangian matches CA-MAMBA's 100k cost and win rate, against
  MAPPO-Lagrangian's interactions. Right: median casualties per allied unit and
  median win rate against interactions.}
  \caption{How many interactions MAPPO-Lagrangian needs to catch up with
  CA-MAMBA at 100k on SMAC.  \textbf{(a)}~Among the 23 configurations where
  CA-MAMBA wins any episode, the fraction on which MAPPO-Lagrangian has matched
  CA-MAMBA's 100k mean cost (dashed), mean win rate (dotted), or both (solid) by
  each checkpoint; the solid curve is the distribution of $m(g)$, defined in the
  text.  \textbf{(b)}~Medians over the 24 configurations of casualties per
  allied unit (red) and win rate (blue): MAPPO-Lagrangian (lines), MACPO at
  100k and 5M (squares), CA-MAMBA at 100k (stars).}
  \label{fig:catchup}
\end{figure}

Figure~\ref{fig:catchup} turns this into the quantity that matters for a
costly system: how many more interactions the baseline needs.  For each
configuration $g$, the interaction multiple $m(g)$ is the smallest multiple of
100k, among $1$, $5$, $10$, $20$, and $50$, at which MAPPO-Lagrangian's mean cost
is at most and its mean win rate at least CA-MAMBA's at 100k.  Over the 23
configurations where CA-MAMBA wins any episode, the median multiple is $20$ on
this grid, so more than $10$ (Table~\ref{tab:interaction-multiple} in
Appendix~\ref{app:full-results} gives the counts and a check that redraws each configuration's runs at random).
MAPPO-Lagrangian does not match CA-MAMBA within $50$ times on 5
configurations, all on maps where its win rate stays below $0.01$ through 5M.
With ten times
the data the comparison is roughly even (CA-MAMBA lower in cost on 14 of 24,
higher in win rate on 13); at $50$ times MAPPO-Lagrangian is ahead in both on
18.  CA-MAMBA's advantage is therefore the interactions it saves.  Against
MACPO the ordering largely persists at 5M (lower cost on 16 of 24, higher win rate on 17); Appendix~\ref{app:full-results}
lists every configuration for all three methods, and Appendix~\ref{app:curves}
plots learning curves against MACPO on four maps.

Much of the lower cost at 100k comes from learning the task faster; pricing cost
adds a further reduction: on \texttt{8m} at $d=0$ the same learner with
its price fixed at zero incurs cost $5.28$ at win rate $0.63$, against
CA-MAMBA's $2.05$ at $0.95$ (Section~\ref{sec:dual-ablation}),
and removing the cost critic as well raised cost on \texttt{1c3s5z} within 20k
interactions (Appendix~\ref{app:safety-mechanism}).  Under a second, dense SMAC team cost that
counts pairs of moving allied units too close together, CA-MAMBA at 100k still
has a higher win rate than MACPO at 1M on 9 of 10 maps, but lower cost on only
4, with one tie (Appendix~\ref{app:collision}).

\subsection{Interaction Efficiency on MAMuJoCo}
\label{sec:mamujoco}

MAMuJoCo tests whether these savings hold when cost structure and action space both
change: cost is dense and reversible, actions are continuous, and reward opposes
cost.  We use SafePO's default limit $d=25$ and run CA-MAMBA for 1M
interactions.

\begin{table}[t]
\centering
\small
\setlength{\tabcolsep}{3pt}
\begin{tabular}{@{}lcc@{}}
\toprule
Method (interactions) & Reward $\uparrow$ & Cost $\downarrow$ \\
\midrule
\multicolumn{3}{@{}l}{\emph{Ant $2{\times}4$}} \\
\addlinespace[1pt]
MACPO & $268 \mathbin{\pm} 70$ & $0.2 \mathbin{\pm} 0.1$ \\
MAPPO-Lag & $439 \mathbin{\pm} 46$ & $0.1 \mathbin{\pm} 0.2$ \\
\textbf{CA-MAMBA} & $\mathbf{1836} \mathbin{\pm} 678$ & $17.4 \mathbin{\pm} 15.1$ \\
\quad no comm. & $1491 \mathbin{\pm} 80$ & $19.4 \mathbin{\pm} 30.3$ \\
\midrule
\multicolumn{3}{@{}l}{\emph{Ant $4{\times}2$}} \\
\addlinespace[1pt]
MACPO & $374 \mathbin{\pm} 96$ & $0.3 \mathbin{\pm} 0.2$ \\
MAPPO-Lag & $562 \mathbin{\pm} 52$ & $0.3 \mathbin{\pm} 0.3$ \\
\textbf{CA-MAMBA} & $1574 \mathbin{\pm} 304$ & $0.5 \mathbin{\pm} 0.8$ \\
\quad no comm. & $\mathbf{1657} \mathbin{\pm} 162$ & $0.5 \mathbin{\pm} 0.9$ \\
\midrule
\multicolumn{3}{@{}l}{\emph{HalfCheetah $2{\times}3$}} \\
\addlinespace[1pt]
MACPO & $207 \mathbin{\pm} 26^{\ast}$ & $0.0 \mathbin{\pm} 0.0$ \\
MAPPO-Lag & $791 \mathbin{\pm} 66$ & $0.2 \mathbin{\pm} 0.3$ \\
\textbf{CA-MAMBA} & $\mathbf{2527} \mathbin{\pm} 229$ & $16.6 \mathbin{\pm} 2.0$ \\
\quad no comm. & $982 \mathbin{\pm} 1303$ & $64.4 \mathbin{\pm} 103.8^{\dagger}$ \\
\bottomrule
\end{tabular}
\caption{MAMuJoCo at the SafePO default limit $d{=}25$, every method at a matched budget of 1M interactions. \emph{No comm.} acts without messages ($\mathcal N_i=\{i\}$) but learns and imagines with them. Mean $\mathbin{\pm}$ sample standard deviation over three seeds ($\ast$: two). Bold: highest mean reward; $\dagger$: mean cost above $d$.}
\label{tab:mamujoco-main}
\end{table}

At the matched budget CA-MAMBA earns $2.8$ to $12$ times the reward of either
baseline in all three environments (Table~\ref{tab:mamujoco-main}).  At the 1M
checkpoint every method's mean cost is under $d=25$, apart from acting without
messages on HalfCheetah (Section~\ref{sec:nocomm}), so the comparison is
decided on reward; the baselines' near-zero cost reflects policies that have
not yet learned to move.  CA-MAMBA's cost exceeds the limit at some earlier
checkpoints, up to $83$ on Ant $2{\times}4$ at 800k
(Table~\ref{tab:mamujoco-schedule} in Appendix~\ref{app:mamujoco}).  With ten
times the interactions MAPPO-Lagrangian exceeds CA-MAMBA's 1M reward only on
Ant $4{\times}2$ ($1866$ against $1574$); MACPO trails on all three.

\subsection{Why Cost Limits Are Sometimes Missed}
\label{sec:dual-ablation}

Saving interactions does not by itself meet the cost limit, which RQ2 examines.
Only MAPPO-Lagrangian meets a zero-casualty limit within 5M interactions, on
\texttt{2m\_vs\_1z} at 2M.  At the
12 nonzero limits CA-MAMBA's mean cost meets 5 at 100k, as many as
MAPPO-Lagrangian at 1M and MACPO at 5M; MAPPO-Lagrangian meets 8 at 5M.  At
the tight MAMuJoCo limits ($d=0.2$, $1$, and $5$ on Ant $2{\times}4$, Ant
$4{\times}2$, and HalfCheetah) CA-MAMBA at $\alpha_\beta=10^{-5}$ exceeds the limit
in all three environments, while with ten times the interactions
MAPPO-Lagrangian meets all three and MACPO only HalfCheetah's, at exactly the limit
(Appendix~\ref{app:mamujoco}).  The model is partly responsible: on the two
configurations measured under the learned multiplier, its imagined cost falls
short of realized cost by about a third (Section~\ref{sec:theory}).  A multiplier that
reads measured cost targets the real constraint despite that bias, but ours
moves slowly: at $d=0$, $\beta$ reaches only $0.025$ to $0.54$, below parity,
after 100k interactions on five maps (Appendix~\ref{app:beta}).

\noindent \textbf{Multiplier rules on \texttt{8m}.}
To separate the multiplier from the model, we run \texttt{8m} at $d=0$, a limit
CA-MAMBA misses on every map, with six multiplier rules: an imagined input,
which feeds the multiplier the mean imagined cost return instead of $\bar C_k$,
as SafeDreamer does when it trains its policy in imagination
\cite{DBLP:conf/iclr/0008JXZ024}, there with a different step rule;
the measured-cost update of Eq.~\eqref{eq:dual-update}; PID control on measured
cost \cite{DBLP:conf/icml/StookeAA20}, which sets $\beta$ from the current
violation $\bar C_k-d$, its running sum, and its change, with gains $K_P$,
$K_I$, $K_D$; fixed $\beta=1$, which weights the
standardized advantages equally, and fixed $\beta=5$; and $\beta=0$, the
cost-blind learner.  A fixed price is the reward-penalty baseline: it optimizes
reward minus a fixed multiple of cost, applied to standardized advantages.  The
cost-blind learner tests the model: if pricing its cost advantages lowers cost,
those advantages are informative enough.  The fixed prices test the multiplier:
if one lowers cost well below the learned multiplier, the multiplier is too
slow.

\IfFileExists{generated/mechanism_8m.tex}{\begin{table}[t]
\centering
\footnotesize
\setlength{\tabcolsep}{4pt}
\begin{tabular}{@{}lccc@{}}
\toprule
 & \multicolumn{3}{c}{At 100k} \\
\cmidrule(lr){2-4}
Multiplier & Cost $\downarrow$ & Win rate $\uparrow$ & $\beta$ \\
\midrule
Imagined & $2.49 \mathbin{\pm} 0.43$ & $0.95 \mathbin{\pm} 0.02$ & $0.22 \mathbin{\pm} 0.01$ \\
Measured & $2.05 \mathbin{\pm} 0.42$ & $0.95 \mathbin{\pm} 0.02$ & $0.54 \mathbin{\pm} 0.02$ \\
PID & $8.00 \mathbin{\pm} 0.00$ & $0.00 \mathbin{\pm} 0.00$ & $8.28 \mathbin{\pm} 0.01$ \\
$\beta{=}5$ & $8.00 \mathbin{\pm} 0.00$ & $0.00 \mathbin{\pm} 0.00$ & $5.00 \mathbin{\pm} 0.00$ \\
$\beta{=}1$ & $1.75 \mathbin{\pm} 0.20$ & $0.97 \mathbin{\pm} 0.02$ & $1.00 \mathbin{\pm} 0.00$ \\
$\beta{=}0$ & $5.28 \mathbin{\pm} 0.95$ & $0.63 \mathbin{\pm} 0.18$ & $0.00 \mathbin{\pm} 0.00$ \\
\bottomrule
\end{tabular}
\caption{Multiplier rules on SMAC \texttt{8m} at $d=0$. Mean $\mathbin{\pm}$ sample standard deviation over three runs (seeds 1--3) of each run's episodes between 95k and 100k interactions; $\beta$ at 100k. \emph{Imagined} and \emph{Measured} update $\beta$ ($\alpha_\beta=10^{-5}$) from imagined or measured cost; \emph{PID} uses $K_P=1$, $K_I=10^{-5}$, $K_D=1$; the $\beta{=}5,1,0$ rows hold $\beta$ fixed, so $d$ does not enter them.}
\label{tab:mechanism-8m}
\end{table}
}{}

Table~\ref{tab:mechanism-8m} answers both.  The model's cost advantages work:
priced at parity, they cut cost from $5.28$ for the cost-blind learner to
$1.75$ and raise the win rate from $0.63$ to $0.97$.  The measured-cost
multiplier reaches $\beta=0.54$ and cost $2.05\pm0.42$, most of the reduction
parity gives, but no rule meets the limit.  SafeDreamer's imagined input sets a lower price:
its $\beta$ reaches only $0.22$, and its cost ends at $2.49\pm0.43$.  Too high a
price, set before the model has learned the task, fails differently: with gains
chosen for MAMuJoCo ($K_P=1$), PID's current-violation term $K_P(\bar C_k-d)$
is $8$ once episode cost reaches $8$, $\beta$ passes $8$ within $600$
interactions, and the agents learn to avoid combat until nearly every episode
times out at full cost.  A fixed $\beta=5$ fails the same way, so the
reward-penalty baseline works only when its weight is chosen well.

\noindent \textbf{The same test on HalfCheetah.}
To check the \texttt{8m} findings beyond discrete control, we repeat the test on HalfCheetah
$2{\times}3$ (Table~\ref{tab:hc-multiplier} in Appendix~\ref{app:mamujoco}).
At the tight limit $d=5$, the learned multiplier misses by a factor of $3.6$,
while a fixed $\beta=1$ ends all six of its runs within the limit (two sets: cost
$0.6\pm0.3$ and $2.4\pm2.1$) at the same reward ($1954$ and $1985$ against
$1974$).  At $d=25$, which the learned multiplier already meets, the fixed
price is more cautious than needed: cost $2.7$ and $1.2$ for $10\%$ and $25\%$
less reward.  A tenfold $\alpha_\beta$ also meets $d=5$, for $25\%$ less
reward; none of the three multiplier rules we ran on Ant meets the tight Ant
limits (Appendix~\ref{app:mamujoco}).  Without a price, cost ends at $474$ and
$400$, far above both limits; on both tasks the model's cost advantages lower
cost once priced.

\subsection{Acting Without Messages}
\label{sec:nocomm}

Finally, RQ3 asks what a policy optimized with messages loses when it acts
without them.  All our runs learn the world model and imagine with
$\mathcal N_i=\mathcal I$; in the no-communication runs, the agents nevertheless
act with $\mathcal N_i=\{i\}$ in every collected episode
(Section~\ref{sec:message}).  At 1M,
reward does not drop on Ant $4{\times}2$ ($1657$ versus $1574$,
Table~\ref{tab:mamujoco-main}), drops $19\%$ on Ant $2{\times}4$, within
run-to-run spread, and falls from $2527$ to $982$ on HalfCheetah, a gap present
from 500k on (Table~\ref{tab:mamujoco-nocomm} in Appendix~\ref{app:mamujoco}).
Messages thus appear to matter for reward on HalfCheetah, with no consistent effect on cost.

\section{Limitations}
\label{sec:limitations}
Five limitations qualify these results.  First, the evidence supports lower cost
(SMAC) or more reward within the limit (MAMuJoCo) at a matched budget, not a
guarantee that no episode exceeds
the limit, and not safer learning: on Ant $2{\times}4$, CA-MAMBA's cumulative
cost above the limit over the first 1M interactions exceeds MACPO's
(Appendix~\ref{app:remaining-evidence}).  Deployments that cannot tolerate a
single violation would also need a shield (Section~\ref{sec:related}).

Second, the model-error evidence is average-case, and on the sparse SMAC
casualty cost we have not measured how often the cost head predicts the steps
on which a unit is actually lost.  Third, both benchmarks charge every agent
the same team cost, so constraints on costs charged to individual agents remain
untested.  Fourth, the communication experiment removes messages only when
acting; training a self-only model would test whether a domain needs them and
how they change the per-agent errors of Proposition~\ref{prop:factorized}.
Fifth, the multiplier experiments cover one task per benchmark, at one limit on
\texttt{8m} and two on HalfCheetah, with three runs per setting.

\section{Conclusion}
Predicting cost in a multi-agent world model moves most constrained policy
updates into imagination: the stronger model-free baseline needs $10$ to $20$
times CA-MAMBA's interactions on the median SMAC configuration, and CA-MAMBA
earns $2.8$ to $12$ times the baselines' MAMuJoCo reward at equal budget.  In
controlled tests on one task per benchmark, pricing the model's cost advantages
at parity more than halves a cost-blind learner's cost and meets a tight
continuous-control limit the learned multiplier misses, suggesting the model's
cost estimates are not the bottleneck.

\bibliographystyle{ACM-Reference-Format}
\bibliography{references_aaai}

\appendix
\section{Training Procedure}
\label{app:algorithm}

This section expands Algorithm~\ref{alg:ca-mamba}.  Let $\theta$ denote
world-model parameters, $\phi$ actor parameters, $\psi$ the parameters of the
critic network, whose two heads give $V^R_\psi$ and $V^C_\psi$, $\beta\geq0$ the
shared team multiplier, $H_{\mathrm{im}}=14$ the number of costed imagined
steps, $d$ the team-cost limit, and $\mathcal D$ the replay buffer.  $W=4$
workers collect episodes asynchronously; each acts with the parameters it
received when its previous episode ended.

\begin{enumerate}
    \item Receive the next finished episode from any worker.  Store local
    observations, joint actions, shared rewards, local costs, termination
    flags, and agent-availability masks in $\mathcal D$.  Set $\bar C_k$ to
    this episode's undiscounted team cost.
    \item Sample length-$T$ sequences from $\mathcal D$.  Infer posterior
    recurrent states and minimize the observation, reward, local-cost, and
    KL terms in Eq.~\eqref{eq:world-model-loss} with respect to $\theta$.
    \item Repeat $E=4$ times:
    \begin{enumerate}
    \item Freeze $\theta$.  Sample posterior start states from replay and roll
    out $15$ imagined latent states.  Store predicted rewards $\hat r_u$ and
    team costs $\hat c_u=\tfrac1N\sum_{i=1}^{N}\hat c_u^i$ for
    $u\in\{0,\ldots,H_{\mathrm{im}}-1\}$, action log-probabilities, and
    termination predictions.  The mean matches the realized aggregation of
    Section~\ref{sec:method}, so prediction and target share granularity.
    \item Construct per-agent reward and cost $\lambda$-returns that bootstrap
    from the critic at the last imagined state.  Fit $\psi$ by minimizing the
    squared target errors of both heads.  Standardize each advantage over the
    batch, form $A_u^{\beta,i}=A_u^{R,i}-\beta A_u^{C,i}$, and \emph{maximize}
    the per-agent clipped PPO surrogate with respect to $\phi$.
    \item Update the shared multiplier from the measured episode cost:
    $\beta\gets[\beta+\alpha_\beta(\bar C_k-d)]_+$.
    \end{enumerate}
    \item Send the updated $\theta,\phi$ to the worker that finished, and return
    to Step 1.
\end{enumerate}

Over one learner step the $E$ multiplier updates compose to
$\beta_{k+1}=[\beta_k+E\alpha_\beta(\bar C_k-d)]_+$, the dual update of
Section~\ref{sec:method}.
The multiplier update is projected descent in the dual variable for the
reward-maximization Lagrangian, not actor maximization.  Because the episode was
collected before the update, $\bar C_k$ is a statistic of the policy that
generated it, $\pi_{\phi_{k-\ell_k}}$ with lag $\ell_k\geq0$, rather than of the
policy that follows.

\section{Execution Under the Full Neighborhood}
\label{app:execution}

Figure~\ref{fig:execution} draws the execution path the main text describes.
Every agent runs its own copy of the world model and a parameter-tied actor, and
the only inter-agent channel is the attention over the neighborhood
$\mathcal N_i$.  Under $\mathcal N_i=\mathcal I$ the agents exchange their
previous actions and stochastic latents before acting; setting
$\mathcal N_i=\{i\}$ removes that exchange and changes nothing else.  The
communication ablation makes this substitution at action selection in every
collected episode, while world-model learning and imagination keep
$\mathcal N_i=\mathcal I$.  No privileged global state enters either setting.

\begin{figure}[ht]
  \centering
  \includegraphics[width=\columnwidth]{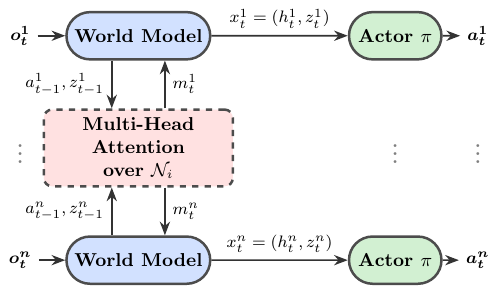}
  \Description{Execution diagram: each agent's world model exchanges previous
  actions and latents through attention and feeds its latent to a shared
  actor.}
  \caption{Execution: every agent runs its own world model and a parameter-tied
  actor.  Each world model sends its previous action and stochastic latent
  $a^i_{t-1},z^i_{t-1}$ to an attention layer over the neighborhood
  $\mathcal N_i$ and receives the message $m^i_t$, which drives its recurrent
  update; the actor then reads the latent $x^i_t=(h^i_t,z^i_t)$.  Under
  $\mathcal N_i=\mathcal I$ the attention is all-to-all; setting
  $\mathcal N_i=\{i\}$ removes the exchange and changes nothing else.  No
  privileged global state is used.}
  \label{fig:execution}
\end{figure}

\section{Verified Implementation Settings}
\label{app:settings}

The values in Tables~\ref{tab:reported-settings} and~\ref{tab:arch-settings}
are read directly from our training code (the configuration classes, the
world-model and actor--critic losses, and the Lagrangian update), so they
describe the runs that produced the reported numbers rather than an intended
configuration.  CA-MAMBA runs on SMAC use the casualty cost
\texttt{dead\_allies\_incremental} (the collision study uses the dense
\texttt{collision} cost), MAMuJoCo runs use the velocity cost of Appendix~\ref{app:mamujoco}, and every
configuration has three independent runs, labeled
seeds $1$ to $3$ (the label does not seed the random number generators; see
below).

\begin{table}[ht]
\centering
\begin{tabular}{lc}
\toprule
Quantity & Value \\
\midrule
World-model learning rate & $2\times10^{-4}$ \\
Actor learning rate (Adam, wd $10^{-5}$) & $5\times10^{-4}$ \\
Critic learning rate & $5\times10^{-4}$ \\
Multiplier step $\alpha_\beta$ (applied $E=4$ times per step) & $10^{-5}$ \\
Initial multiplier $\beta_0$ & $10^{-4}$ \\
PPO clip $\varepsilon$ & $0.2$ \\
Entropy coefficient (anneal $0.99998$) & $0.001$ \\
Model/agent batch size & $40$ \\
Sequence length $T$ & $20$ \\
Replay capacity / min.\ buffer & $250{,}000$ / $500$ \\
Imagined states / costed steps $H_{\mathrm{im}}$ & $15$ / $14$ \\
Discount $\gamma$ & $0.99$ \\
GAE parameter $\lambda$ & $0.95$ \\
Gradient clip (model / policy) & $100$ / $100$ \\
Advantage standardization (per batch) & on \\
KL weight $\omega_{\mathrm{KL}}$ (on the mean over categoricals) & $1$ \\
KL balance (prior / posterior) & $0.8$ / $0.2$ \\
\bottomrule
\end{tabular}
\caption{Optimization and training settings, read from the training code.  Every
term of Eq.~(1), including the divergence and the three terms inherited from MAMBA (the continuation and action-availability heads and, for discrete actions, the action classifier),
enters the world-model loss with weight $1$; KL balancing splits the divergence
$0.8$ toward the direction that updates the prior and $0.2$ toward the
posterior.}
\label{tab:reported-settings}
\end{table}

\paragraph{Update schedule.}
The learner updates once per collected episode.  Each update performs $60$
world-model gradient steps on minibatches of $40$ length-$20$ sequences, then
$4$ policy-update rounds; each round generates a fresh imagined rollout and
runs $5$ PPO epochs over minibatches of $2{,}000$ imagined transitions.  The
imagined reward and cost returns are $\lambda$-returns at $\gamma=0.99$; the
shared multiplier takes one step after each of the $4$ rounds, all
from the same measured undiscounted episode cost.  Data collection uses four
parallel workers.  The divergence in Eq.~(1) is averaged over the $32$
categorical variables of $z_t^i$, so relative to a divergence summed over the
full latent the effective weight is $1/32$.

\begin{table}[ht]
\centering
\small
\begin{tabular}{ll}
\toprule
Component & Specification \\
\midrule
Deterministic state $h$ & $256$ (GRU) \\
Stochastic state $z$ & $32\!\times\!32$ categ.\ (DreamerV2) \\
Feature $x=(h,z)$ & $1280$ \\
Observation embedding & $256$ \\
Reward / cost heads & 2-layer MLP, width $256$ \\
Continuation head & 2-layer MLP, width $256$ \\
Actor (per agent, tied) & 2-layer MLP, width $256$ \\
Critic (augmented) & 1 attn.\ layer, 1-layer heads \\
Transition attention & 8 heads, 3 layers, drop $0.1$ \\
Positional encoding & sinusoidal in agent index, $\leq30$ agents \\
\bottomrule
\end{tabular}
\caption{Network architecture, read from the training code.  The transition
attention is the only inter-agent communication path; the actor is a per-agent
MLP over the communication-conditioned latent.}
\label{tab:arch-settings}
\end{table}

\paragraph{Continuous-control settings.}
The MAMuJoCo runs share the architecture above but override the optimizer
settings on the command line, since a Gaussian policy over a squashed
continuous action is far less tolerant of the SMAC step sizes: world-model rate
$10^{-4}$, actor rate $3\times10^{-5}$, critic rate $10^{-4}$, gradient clip
$10$ (model) and $5$ (policy), with $\mathrm{PPO}$ epochs, policy-update rounds,
imagination horizon, and $\alpha_\beta$ unchanged.  The policy is a Normal
distribution with a $\tanh$ squash; its log-probability is evaluated at the
pre-squash sample rather than the squashed action, without which the PPO
importance ratio is wrong and the policy collapses through entropy death
within roughly $55$k interactions.  Communication mode is selected by a
\texttt{--comm\_mode} flag (\texttt{full} or \texttt{none}).

\paragraph{Compute.}
Every run used a single NVIDIA RTX 3090 on a Slurm cluster, one GPU per job,
with four parallel data-collection workers.  Per-job timing was not retained for
every run, so we report it for three runs of each timed configuration.  On
HalfCheetah $2{\times}3$ at $d=25$, 1M interactions take about $4.5$ hours ($4.9\pm0.5$
hours to 1.075M, three runs).  The learner updates once per finished
episode, so wall-clock time scales with the number of episodes rather than of
interactions: \texttt{8m} episodes are about $30$ times shorter than
HalfCheetah's $1000$ steps, and 100k interactions on \texttt{8m} take
$16.3\pm0.4$ hours (three runs).  In wall-clock time the model-free baselines
are faster per interaction: on \texttt{8m}, MACPO takes $2.5$ hours to 1M and
$12.0$ to 5M, and MAPPO-Lagrangian $6.3\pm1.3$ and $33.4\pm7.4$ hours; on
HalfCheetah, 1M takes $1.1$ and $0.7$ hours.  These times exclude queueing.
CA-MAMBA jobs request $64$ GB of memory; the cluster did not record peak
usage.

\paragraph{Determinism.}
The training entry point accepts a \texttt{--seed} argument and all reported
results use seeds $\{1,2,3\}$, but that value is used to label runs and to seed
an environment curriculum that no reported experiment uses.  It is \emph{not} propagated to Python's
\texttt{random}, NumPy, PyTorch, CUDA, or the SMAC and MAMuJoCo environment
processes, and no deterministic-algorithm flags are set.  Consequently the
three runs per configuration are genuinely independent samples, which is what
the reported sample standard deviations describe, but an individual run is not
bit-reproducible from its seed.  Fixing this requires seeding those five
sources at process start; it is a defect of the implementation, not of the
protocol, and we state it rather than imply stronger reproducibility than the
code supports.  Library versions are likewise unpinned in
\texttt{requirements.txt}.

The hyperparameters of MACPO on SMAC, a discrete-action adaptation, are listed in
Appendix~\ref{app:baseline-hyperparams}.  Its discrete-SMAC adaptation is not
part of the CA-MAMBA code, so its commit and implementation details accompany
that code rather than this one.

\section{MAMuJoCo Details}
\label{app:mamujoco}

This section documents the MAMuJoCo runs behind Sections~\ref{sec:mamujoco} and~\ref{sec:dual-ablation}.

\paragraph{Environments and provenance.}
The three velocity-constrained tasks are
\texttt{Safety2x4AntVelocity-v0} (2 agents, 4 actuators each, 29-dimensional
local observation, velocity threshold $2.522$),
\texttt{Safety4x2AntVelocity-v0} (4 agents, 2 actuators, 31 dimensions,
threshold $2.418$), and
\texttt{Safety2x3HalfCheetahVelocity-v0} (2 agents, 3 actuators, 19 dimensions,
threshold $3.227$).  Continuous actions use a Gaussian policy with a
$\tanh$ squash; the log-probability is evaluated at the pre-squash sample, which
is required for the PPO importance ratio to be correct.  The schedule, step-size
and no-communication tables of this section are generated by
\path{scripts/generate_mamujoco_tables.py} from the aggregated run record
\path{data/mamujoco_experiment8_phase5_full.csv}, reading each run as the mean over its episodes in the last 5k
interactions before a checkpoint; seeds 2 and 3 of HalfCheetah at $d=5$ and $\alpha_\beta=10^{-5}$ are read
from their full training histories by the same rule.  Both
baselines are our own runs of the SafePO implementations
\cite{DBLP:conf/nips/JiZZP0SGZD023} at 10M interactions and three seeds (two for MACPO on HalfCheetah $2{\times}3$ at $d=25$) under
the settings that library ships, so the budget is the only deliberate asymmetry
against CA-MAMBA's 1M.

\paragraph{Interaction budget.}
Table~\ref{tab:mamujoco-schedule} traces CA-MAMBA at $d=25$ across checkpoints
and gives both baselines at the matched 1M budget and at 10M, the two checkpoints
of the same runs.  Mean cost is $0.0$ at 100k interactions on both Ant partitions
and grows as the policy converts constraint slack into reward, which is the
expected behavior of a Lagrangian method below its limit; on Ant $2{\times}4$
it overshoots the limit between 700k and 900k before returning under it.
Reward at 100k already exceeds the 10M-interaction MACPO run on Ant $2{\times}4$
($836$ versus $798$) but not on Ant $4{\times}2$ or HalfCheetah, so the
cross-environment statement in the main text is made at 1M rather than 100k.
The 1M baseline rows are what make the main text's comparison matched: at that
budget neither model-free learner has left the low-velocity regime, so their
near-zero cost reflects an undeveloped gait rather than constraint satisfaction
that CA-MAMBA failed to achieve.

\begin{table}[t]
\centering
\small
\setlength{\tabcolsep}{3pt}
\begin{tabular}{@{}lcccc@{}}
\toprule
Environment & Steps & $n$ & Reward $\uparrow$ & Cost $\downarrow$ \\
\midrule
Ant $2{\times}4$ & 100k & 3 & $836 \mathbin{\pm} 93$ & $0.0 \mathbin{\pm} 0.0$ \\
 & 500k & 3 & $1086 \mathbin{\pm} 412$ & $6.8 \mathbin{\pm} 11.5$ \\
 & 700k & 3 & $1614 \mathbin{\pm} 553$ & $53.1 \mathbin{\pm} 47.1$ \\
 & 800k & 3 & $1794 \mathbin{\pm} 672$ & $82.9 \mathbin{\pm} 72.0$ \\
 & 900k & 3 & $1838 \mathbin{\pm} 583$ & $32.7 \mathbin{\pm} 33.7$ \\
 & 1M & 3 & $1836 \mathbin{\pm} 678$ & $17.4 \mathbin{\pm} 15.1$ \\
\textit{MACPO} & 1M & 3 & $268 \mathbin{\pm} 70$ & $0.2 \mathbin{\pm} 0.1$ \\
\textit{MACPO} & 10M & 3 & $798 \mathbin{\pm} 109$ & $10.5 \mathbin{\pm} 1.7$ \\
\textit{MAPPO-Lag} & 1M & 3 & $439 \mathbin{\pm} 46$ & $0.1 \mathbin{\pm} 0.2$ \\
\textit{MAPPO-Lag} & 10M & 3 & $1580 \mathbin{\pm} 185$ & $1.4 \mathbin{\pm} 0.3$ \\
\midrule
Ant $4{\times}2$ & 100k & 3 & $902 \mathbin{\pm} 49$ & $0.0 \mathbin{\pm} 0.0$ \\
 & 500k & 3 & $997 \mathbin{\pm} 83$ & $0.0 \mathbin{\pm} 0.0$ \\
 & 700k & 3 & $1193 \mathbin{\pm} 87$ & $0.0 \mathbin{\pm} 0.0$ \\
 & 800k & 3 & $1266 \mathbin{\pm} 102$ & $0.0 \mathbin{\pm} 0.0$ \\
 & 900k & 3 & $1407 \mathbin{\pm} 172$ & $0.0 \mathbin{\pm} 0.0$ \\
 & 1M & 3 & $1574 \mathbin{\pm} 304$ & $0.5 \mathbin{\pm} 0.8$ \\
\textit{MACPO} & 1M & 3 & $374 \mathbin{\pm} 96$ & $0.3 \mathbin{\pm} 0.2$ \\
\textit{MACPO} & 10M & 3 & $1287 \mathbin{\pm} 148$ & $23.6 \mathbin{\pm} 6.0$ \\
\textit{MAPPO-Lag} & 1M & 3 & $562 \mathbin{\pm} 52$ & $0.3 \mathbin{\pm} 0.3$ \\
\textit{MAPPO-Lag} & 10M & 3 & $1866 \mathbin{\pm} 42$ & $0.1 \mathbin{\pm} 0.2$ \\
\midrule
HalfCheetah $2{\times}3$ & 100k & 3 & $1136 \mathbin{\pm} 565$ & $30.1 \mathbin{\pm} 30.2$ \\
 & 500k & 3 & $2000 \mathbin{\pm} 121$ & $5.4 \mathbin{\pm} 6.9$ \\
 & 700k & 3 & $2588 \mathbin{\pm} 181$ & $22.3 \mathbin{\pm} 22.2$ \\
 & 800k & 3 & $2545 \mathbin{\pm} 212$ & $37.6 \mathbin{\pm} 9.4$ \\
 & 900k & 3 & $2528 \mathbin{\pm} 163$ & $3.9 \mathbin{\pm} 3.7$ \\
 & 1M & 3 & $2527 \mathbin{\pm} 229$ & $16.6 \mathbin{\pm} 2.0$ \\
\textit{MACPO} & 1M & 2 & $207 \mathbin{\pm} 26$ & $0.0 \mathbin{\pm} 0.0$ \\
\textit{MACPO} & 10M & 2 & $1374 \mathbin{\pm} 86$ & $22.0 \mathbin{\pm} 4.2$ \\
\textit{MAPPO-Lag} & 1M & 3 & $791 \mathbin{\pm} 66$ & $0.2 \mathbin{\pm} 0.3$ \\
\textit{MAPPO-Lag} & 10M & 3 & $1141 \mathbin{\pm} 36$ & $0.0 \mathbin{\pm} 0.0$ \\
\bottomrule
\end{tabular}
\caption{CA-MAMBA at $d{=}25$ against interaction budget, with both model-free baselines at the matched 1M budget and at 10M. The baseline rows are the same runs read at two checkpoints. Cost is not monotone in the budget: it exceeds $d$ on Ant $2{\times}4$ at 700k to 900k and on HalfCheetah $2{\times}3$ at 100k and 800k.}
\label{tab:mamujoco-schedule}
\end{table}

\paragraph{Tight cost limits.}
Table~\ref{tab:mamujoco-pid-tight} repeats the comparison at much tighter
limits, $d=0.2$, $1.0$, and $5.0$ for Ant $2{\times}4$, Ant $4{\times}2$, and
HalfCheetah $2{\times}3$, near the bounds $0.2$, $0.2$, and $5$ (five times looser on Ant $4{\times}2$) that
\citet{DBLP:journals/ai/GuKCDYKY23} set for their obstacle costs.  No CA-MAMBA
multiplier rule meets either Ant limit at 1M.  The model-free baselines meet
these limits at 1M only at very low reward, MACPO on all three tasks and
MAPPO-Lagrangian on two; with ten times the interactions MAPPO-Lagrangian meets
all three, at reward $1549$ and $1723$ on the Ant partitions, and MACPO only
HalfCheetah's.  HalfCheetah is where CA-MAMBA meets the tight limit: at
$\alpha_\beta{=}10^{-4}$ its mean cost is $4.6$ at reward $1481$ after 1M
interactions, against $1317$ for MACPO and $1138$ for MAPPO-Lagrangian after
10M.

\begin{table*}[t]
\centering
\small
\setlength{\tabcolsep}{4pt}
\begin{tabular}{@{}lccccccc@{}}
\toprule
 &  & \multicolumn{2}{c}{Ant $2{\times}4$, $d=0.2$} & \multicolumn{2}{c}{Ant $4{\times}2$, $d=1.0$} & \multicolumn{2}{c}{HalfCheetah $2{\times}3$, $d=5.0$} \\
\cmidrule(lr){3-4}\cmidrule(lr){5-6}\cmidrule(lr){7-8}
Method & Steps & Reward $\uparrow$ & Cost $\downarrow$ & Reward $\uparrow$ & Cost $\downarrow$ & Reward $\uparrow$ & Cost $\downarrow$ \\
\midrule
CA-MAMBA, $\alpha_\beta{=}10^{-5}$ & 1M & $2268 \mathbin{\pm} 134$ & $24.27 \mathbin{\pm} 5.62^{\dagger}$ & $1241 \mathbin{\pm} 566$ & $2.53 \mathbin{\pm} 3.11^{\dagger}$ & $1973 \mathbin{\pm} 506$ & $18.17 \mathbin{\pm} 15.95^{\dagger}$ \\
CA-MAMBA, $\alpha_\beta{=}10^{-4}$ & 1M & $1618 \mathbin{\pm} 641$ & $4.07 \mathbin{\pm} 5.84^{\dagger}$ & $1559 \mathbin{\pm} 359$ & $4.20 \mathbin{\pm} 7.27^{\dagger}$ & $1481 \mathbin{\pm} 98$ & $4.56 \mathbin{\pm} 1.99$ \\
CA-MAMBA, PID & 1M & $1289 \mathbin{\pm} 421$ & $0.91 \mathbin{\pm} 0.54^{\dagger}$ & $1417 \mathbin{\pm} 196$ & $2.40 \mathbin{\pm} 1.60^{\dagger}$ & $2271 \mathbin{\pm} 282$ & $5.22 \mathbin{\pm} 2.87^{\dagger}$ \\
MACPO & 1M & $299 \mathbin{\pm} 42$ & $0.19 \mathbin{\pm} 0.17$ & $315 \mathbin{\pm} 47$ & $0.60 \mathbin{\pm} 0.66$ & $181 \mathbin{\pm} 43$ & $0.00 \mathbin{\pm} 0.00$ \\
MAPPO-Lag & 1M & $520 \mathbin{\pm} 37$ & $0.29 \mathbin{\pm} 0.37^{\dagger}$ & $552 \mathbin{\pm} 58$ & $0.92 \mathbin{\pm} 0.07$ & $775 \mathbin{\pm} 60$ & $0.23 \mathbin{\pm} 0.32$ \\
\midrule
MACPO & 10M & $859 \mathbin{\pm} 33$ & $1.21 \mathbin{\pm} 0.69^{\dagger}$ & $1098 \mathbin{\pm} 97$ & $3.36 \mathbin{\pm} 1.02^{\dagger}$ & $1317 \mathbin{\pm} 116$ & $5.00 \mathbin{\pm} 1.74$ \\
MAPPO-Lag & 10M & $1549 \mathbin{\pm} 43$ & $0.13 \mathbin{\pm} 0.15$ & $1723 \mathbin{\pm} 46$ & $0.43 \mathbin{\pm} 0.59$ & $1138 \mathbin{\pm} 41$ & $0.10 \mathbin{\pm} 0.00$ \\
\bottomrule
\end{tabular}
\caption{Reward and cost on the multi-agent MuJoCo velocity tasks at tight cost limits: $d=0.2$ on Ant $2{\times}4$, $d=1.0$ on Ant $4{\times}2$, $d=5.0$ on HalfCheetah $2{\times}3$. CA-MAMBA updates its multiplier $\beta$ from the measured episode cost with step size $\alpha_\beta$, or sets it with a PID controller on the same signal ($K_P=1$, $K_I=10^{-5}$, $K_D=1$), and is read after 1M interactions; MACPO and MAPPO-Lag are read after 1M and 10M. Each run's reward and cost average its episodes (CA-MAMBA) or logged interval means (MACPO, MAPPO-Lag) in the last 5k interactions before the budget; entries are mean $\mathbin{\pm}$ sample standard deviation over three seeds per cell. $\dagger$: mean cost above $d$. The PID multiplier reached its clip value $\beta=100$ within the first 1M interactions in 1 of 3 runs on HalfCheetah $2{\times}3$.}
\label{tab:mamujoco-pid-tight}
\end{table*}

\paragraph{Multiplier learning rate.}
Table~\ref{tab:mamujoco-laglr} gives the full $\alpha_\beta$ sweep.  The
asymmetry is the useful part: raising the rate helps only where a persistent
positive residual exists.  Under the tight limits it removes $83\%$ of the cost
on Ant $2{\times}4$ for $29\%$ of reward and $75\%$ on HalfCheetah for $25\%$,
making the latter feasible at 1M, though not at 800k or 900k; on Ant
$4{\times}2$ cost does not fall ($2.5$ to $4.2$, within run-to-run spread).
HalfCheetah seeds
2 and 3 at the slow rate were run twice; we use the documented reruns, and the
earlier pair would give a $10\%$ reward drop instead of $25\%$.  A PID controller on the
measured cost ($K_P=1$, $K_I=10^{-5}$, $K_D=1$) reaches, at 1M, cost $0.91$ and
reward $1289$ on Ant $2{\times}4$, $2.40$ and $1417$ on Ant $4{\times}2$, and
$5.22$ and $2271$ on HalfCheetah: the lowest cost of CA-MAMBA's three rules on
both Ant partitions and the highest reward on HalfCheetah, still above every
limit (Table~\ref{tab:mamujoco-pid-tight}).  Under $d=25$, where the constraint is already met, raising $\alpha_\beta$ tenfold reduces
HalfCheetah reward by $55\%$ in exchange for driving an already-satisfied cost
from $16.6$ to $0.0$.  A faster dual is a remedy for violation, not a free
improvement.

\begin{table}[t]
\centering
\small
\setlength{\tabcolsep}{2.5pt}
\begin{tabular}{@{}lccccc@{}}
\toprule
Environment & $d$ & $\alpha_\beta$ & Reward $\uparrow$ & Cost $\downarrow$ & $\bar{c}\leq d$ \\
\midrule
Ant $2{\times}4$ & 0.2 & $10^{-5}$ & $2268 \mathbin{\pm} 134$ & $24.3 \mathbin{\pm} 5.6$ & $\times$ \\
 &  & $10^{-4}$ & $1618 \mathbin{\pm} 641$ & $4.1 \mathbin{\pm} 5.8$ & $\times$ \\
\cmidrule(l){2-6}
 & 25 & $10^{-5}$ & $1836 \mathbin{\pm} 678$ & $17.4 \mathbin{\pm} 15.1$ & \checkmark \\
 &  & $10^{-4}$ & $1552 \mathbin{\pm} 475$ & $22.6 \mathbin{\pm} 18.3$ & \checkmark \\
\midrule
Ant $4{\times}2$ & 1.0 & $10^{-5}$ & $1241 \mathbin{\pm} 566$ & $2.5 \mathbin{\pm} 3.1$ & $\times$ \\
 &  & $10^{-4}$ & $1559 \mathbin{\pm} 359$ & $4.2 \mathbin{\pm} 7.3$ & $\times$ \\
\cmidrule(l){2-6}
 & 25 & $10^{-5}$ & $1574 \mathbin{\pm} 304$ & $0.5 \mathbin{\pm} 0.8$ & \checkmark \\
 &  & $10^{-4}$ & $1504 \mathbin{\pm} 520$ & $2.9 \mathbin{\pm} 4.8$ & \checkmark \\
\midrule
HalfCheetah $2{\times}3$ & 5.0 & $10^{-5}$ & $1974 \mathbin{\pm} 506$ & $18.2 \mathbin{\pm} 15.9$ & $\times$ \\
 &  & $10^{-4}$ & $1481 \mathbin{\pm} 98$ & $4.6 \mathbin{\pm} 2.0$ & \checkmark \\
\cmidrule(l){2-6}
 & 25 & $10^{-5}$ & $2527 \mathbin{\pm} 229$ & $16.6 \mathbin{\pm} 2.0$ & \checkmark \\
 &  & $10^{-4}$ & $1149 \mathbin{\pm} 313$ & $0.0 \mathbin{\pm} 0.0$ & \checkmark \\
\bottomrule
\end{tabular}
\caption{Multiplier learning rate $\alpha_\beta$ at 1M interactions, at both cost limits per environment.}
\label{tab:mamujoco-laglr}
\end{table}

\paragraph{Fixed price and parity start.}
Table~\ref{tab:hc-multiplier} gives the HalfCheetah comparison of the main text's multiplier section at $d=5$ and $d=25$.  The parity-start runs and
fixed-price set (A) were trained on the current code and logged locally; the
imagined-input, PID, fixed-price (B), and $\beta=0$ rows are a second October
2026 set on the same code; \emph{Learned} is the June--July record of the main
MAMuJoCo table.  Learned runs on the current code at $d=25$ give reward
$2511\pm377$ and cost $10.2\pm4.4$ (three runs), within spread in reward and lower in cost.  The PID row
here is a separate set of three runs from that of Table~\ref{tab:mamujoco-pid-tight}, with the same gains;
both miss $d=5$.  Every
endpoint uses the reading rule of this section.

\begin{table}[t]
\centering
\footnotesize
\setlength{\tabcolsep}{3pt}
\begin{tabular}{@{}lcccc@{}}
\toprule
 & \multicolumn{2}{c}{$d=5$} & \multicolumn{2}{c}{$d=25$} \\
\cmidrule(lr){2-3}\cmidrule(lr){4-5}
Multiplier & Reward $\uparrow$ & Cost $\downarrow$ & Reward $\uparrow$ & Cost $\downarrow$ \\
\midrule
Learned & $1974 \mathbin{\pm} 506$ & $18.2^{\dagger} \mathbin{\pm} 15.9$ (1/3) & $2527 \mathbin{\pm} 229$ & $16.6 \mathbin{\pm} 2.0$ \\
Parity start & $2266 \mathbin{\pm} 480$ & $4.0 \mathbin{\pm} 6.1$ (2/3) & $2605 \mathbin{\pm} 235$ & $19.1 \mathbin{\pm} 25.5$ (2/3) \\
Imagined & $1971 \mathbin{\pm} 310$ & $29.1^{\dagger} \mathbin{\pm} 48.0$ (2/3) & $1913 \mathbin{\pm} 395$ & $87.0^{\dagger} \mathbin{\pm} 75.0$ (1/3) \\
PID & $2077 \mathbin{\pm} 45$ & $6.1^{\dagger} \mathbin{\pm} 4.4$ (1/3) & $1555 \mathbin{\pm} 22$ & $25.3^{\dagger} \mathbin{\pm} 9.1$ (2/3) \\
\midrule
$\beta{=}1$ (A) & $1954 \mathbin{\pm} 298$ & $0.6 \mathbin{\pm} 0.3$ (3/3) & $2282 \mathbin{\pm} 219$ & $2.7 \mathbin{\pm} 1.4$ (3/3) \\
$\beta{=}1$ (B) & $1985 \mathbin{\pm} 560$ & $2.4 \mathbin{\pm} 2.1$ (3/3) & $1905 \mathbin{\pm} 713$ & $1.2 \mathbin{\pm} 2.1$ (3/3) \\
$\beta{=}0$ & $2683 \mathbin{\pm} 1402$ & $473.7^{\dagger} \mathbin{\pm} 318.4$ (0/3) & $2550 \mathbin{\pm} 1623$ & $399.5^{\dagger} \mathbin{\pm} 381.3$ (0/3) \\
\bottomrule
\end{tabular}
\caption{Multiplier rules on HalfCheetah $2{\times}3$ after 1M interactions.  Each run's reward and cost average its episodes in the last 5k interactions; cells give mean $\pm$ sample standard deviation over three runs and, in parentheses, how many runs end within $d$ (not available for \emph{Learned} at $d=25$).  \emph{Learned} updates $\beta$ from the measured episode cost ($\alpha_\beta=10^{-5}$, $\beta_0=10^{-4}$; the June--July 2026 runs of the main MAMuJoCo table); \emph{Parity start} starts it at $\beta_0=1$; \emph{Imagined} feeds it the imagined cost return instead; \emph{PID} sets $\beta$ by PID control on the measured cost ($K_P=1$, $K_I=10^{-5}$, $K_D=1$).  The last three rows hold $\beta$ fixed; $\beta=1$ weights the standardized cost and reward advantages equally, and (A) and (B) are two independent sets of runs with the same settings.  With $\beta$ fixed, $d$ does not enter training, so the two columns of a fixed row differ only by run-to-run variation.  $\dagger$: mean cost above $d$.}
\label{tab:hc-multiplier}
\end{table}

\paragraph{Communication ablation.}
The no-communication rows of Table~\ref{tab:mamujoco-main} come from runs with
\texttt{--comm\_mode none}.  In the collection worker this sets the attention
mask to the boolean complement of the identity, so each agent attends only to
itself when selecting an environment action.  The learner path does not receive
a mask, so world-model learning and imagined rollouts remain all-to-all.  The
ablation therefore measures whether a policy optimized with communication still
works when it acts without the channel in every collected episode; it is not a
test of training the world model without communication, which we did not run.
Table~\ref{tab:mamujoco-nocomm} pairs the two settings at every checkpoint from
500k and at both limits.  We report only the MAMuJoCo environments for this
ablation.

\begin{table*}[t]
\centering
\small
\begin{tabular}{@{}llccccc@{}}
\toprule
Environment & $d$ & Steps & Reward, messages & Reward, none & Cost, messages & Cost, none \\
\midrule
Ant $2{\times}4$ & 25 & 500k & $1086 \mathbin{\pm} 412$ & $1083 \mathbin{\pm} 115$ & $6.8 \mathbin{\pm} 11.5$ & $0.0 \mathbin{\pm} 0.0$ \\
 &  & 700k & $1614 \mathbin{\pm} 553$ & $1216 \mathbin{\pm} 244$ & $53.1 \mathbin{\pm} 47.1$ & $1.0 \mathbin{\pm} 1.1$ \\
 &  & 800k & $1794 \mathbin{\pm} 672$ & $1221 \mathbin{\pm} 60$ & $82.9 \mathbin{\pm} 72.0$ & $1.3 \mathbin{\pm} 1.3$ \\
 &  & 900k & $1838 \mathbin{\pm} 583$ & $1506 \mathbin{\pm} 346$ & $32.7 \mathbin{\pm} 33.7$ & $4.7 \mathbin{\pm} 7.6$ \\
 &  & 1M & $1836 \mathbin{\pm} 678$ & $1491 \mathbin{\pm} 80$ & $17.4 \mathbin{\pm} 15.1$ & $19.4 \mathbin{\pm} 30.3$ \\
\midrule
 & 0.2 & 500k & $1410 \mathbin{\pm} 395$ & $816 \mathbin{\pm} 181$ & $2.0 \mathbin{\pm} 1.8$ & $8.6 \mathbin{\pm} 14.8$ \\
 &  & 700k & $1566 \mathbin{\pm} 857$ & $827 \mathbin{\pm} 252$ & $18.0 \mathbin{\pm} 11.8$ & $3.9 \mathbin{\pm} 6.8$ \\
 &  & 800k & $1627 \mathbin{\pm} 159$ & $1102 \mathbin{\pm} 223$ & $43.5 \mathbin{\pm} 8.2$ & $2.6 \mathbin{\pm} 4.5$ \\
 &  & 900k & $1902 \mathbin{\pm} 359$ & $1126 \mathbin{\pm} 282$ & $43.9 \mathbin{\pm} 29.7$ & $1.7 \mathbin{\pm} 3.0$ \\
 &  & 1M & $2268 \mathbin{\pm} 134$ & $1192 \mathbin{\pm} 445$ & $24.3 \mathbin{\pm} 5.6$ & $0.9 \mathbin{\pm} 1.2$ \\
\midrule
Ant $4{\times}2$ & 25 & 500k & $997 \mathbin{\pm} 83$ & $1165 \mathbin{\pm} 176$ & $0.0 \mathbin{\pm} 0.0$ & $0.0 \mathbin{\pm} 0.0$ \\
 &  & 700k & $1193 \mathbin{\pm} 87$ & $1302 \mathbin{\pm} 286$ & $0.0 \mathbin{\pm} 0.0$ & $0.0 \mathbin{\pm} 0.0$ \\
 &  & 800k & $1266 \mathbin{\pm} 102$ & $1335 \mathbin{\pm} 279$ & $0.0 \mathbin{\pm} 0.0$ & $0.0 \mathbin{\pm} 0.0$ \\
 &  & 900k & $1407 \mathbin{\pm} 172$ & $1516 \mathbin{\pm} 254$ & $0.0 \mathbin{\pm} 0.0$ & $0.0 \mathbin{\pm} 0.0$ \\
 &  & 1M & $1574 \mathbin{\pm} 304$ & $1657 \mathbin{\pm} 162$ & $0.5 \mathbin{\pm} 0.8$ & $0.5 \mathbin{\pm} 0.9$ \\
\midrule
 & 1.0 & 500k & $965 \mathbin{\pm} 34$ & $851 \mathbin{\pm} 102$ & $0.3 \mathbin{\pm} 0.6$ & $0.0 \mathbin{\pm} 0.0$ \\
 &  & 700k & $886 \mathbin{\pm} 111$ & $1105 \mathbin{\pm} 194$ & $0.0 \mathbin{\pm} 0.0$ & $0.0 \mathbin{\pm} 0.0$ \\
 &  & 800k & $1134 \mathbin{\pm} 170$ & $1104 \mathbin{\pm} 140$ & $0.1 \mathbin{\pm} 0.1$ & $0.0 \mathbin{\pm} 0.0$ \\
 &  & 900k & $1141 \mathbin{\pm} 33$ & $1159 \mathbin{\pm} 190$ & $0.5 \mathbin{\pm} 0.8$ & $0.0 \mathbin{\pm} 0.0$ \\
 &  & 1M & $1241 \mathbin{\pm} 566$ & $1032 \mathbin{\pm} 57$ & $2.5 \mathbin{\pm} 3.1$ & $0.0 \mathbin{\pm} 0.0$ \\
\midrule
HalfCheetah $2{\times}3$ & 25 & 500k & $2000 \mathbin{\pm} 121$ & $960 \mathbin{\pm} 942$ & $5.4 \mathbin{\pm} 6.9$ & $45.2 \mathbin{\pm} 52.7$ \\
 &  & 700k & $2588 \mathbin{\pm} 181$ & $1582 \mathbin{\pm} 590$ & $22.3 \mathbin{\pm} 22.2$ & $41.8 \mathbin{\pm} 40.3$ \\
 &  & 800k & $2545 \mathbin{\pm} 212$ & $1120 \mathbin{\pm} 468$ & $37.6 \mathbin{\pm} 9.4$ & $18.2 \mathbin{\pm} 17.5$ \\
 &  & 900k & $2528 \mathbin{\pm} 163$ & $824 \mathbin{\pm} 465$ & $3.9 \mathbin{\pm} 3.7$ & $2.9 \mathbin{\pm} 1.8$ \\
 &  & 1M & $2527 \mathbin{\pm} 229$ & $982 \mathbin{\pm} 1303$ & $16.6 \mathbin{\pm} 2.0$ & $64.4 \mathbin{\pm} 103.8$ \\
\midrule
 & 5.0 & 500k & $1642 \mathbin{\pm} 306$ & $717 \mathbin{\pm} 963$ & $2.2 \mathbin{\pm} 2.5$ & $15.4 \mathbin{\pm} 21.4$ \\
 &  & 700k & $1819 \mathbin{\pm} 585$ & $729 \mathbin{\pm} 1041$ & $7.0 \mathbin{\pm} 7.1$ & $0.6 \mathbin{\pm} 0.5$ \\
 &  & 800k & $1736 \mathbin{\pm} 249$ & $839 \mathbin{\pm} 1161$ & $10.6 \mathbin{\pm} 10.0$ & $3.2 \mathbin{\pm} 3.0$ \\
 &  & 900k & $1982 \mathbin{\pm} 388$ & $722 \mathbin{\pm} 439$ & $15.1 \mathbin{\pm} 19.9$ & $4.8 \mathbin{\pm} 7.7$ \\
 &  & 1M & $1974 \mathbin{\pm} 506$ & $1166 \mathbin{\pm} 513$ & $18.2 \mathbin{\pm} 15.9$ & $11.1 \mathbin{\pm} 18.9$ \\
\bottomrule
\end{tabular}
\caption{CA-MAMBA acting with messages ($\mathcal N_i=\mathcal I$) and without them ($\mathcal N_i=\{i\}$ in every collected episode; world-model learning and imagination use $\mathcal N_i=\mathcal I$ in both), $\alpha_\beta=10^{-5}$, mean $\pm$ sample standard deviation over three runs per cell.}
\label{tab:mamujoco-nocomm}
\end{table*}

\section{SMAC Learning Curves}
\label{app:curves}

Figure~\ref{fig:curves-supp} gives the $d=0$ learning curves for four maps.
They separate into two regimes.  On the easier
maps (\texttt{bane\_vs\_bane}, \texttt{2m\_vs\_1z}) MACPO-Discrete eventually
matches CA-MAMBA's win rate, but only after the full 5M interactions, against CA-MAMBA's
100k.  On the harder maps (\texttt{8m}, \texttt{1c3s5z}) it never leaves zero
win rate at any budget we ran, while CA-MAMBA converges within 100k.  The
CA-MAMBA curves use the imagined-cost multiplier input, so their values differ
from the measured-input runs of the main results; they are shown for dynamics,
not as a separate measurement.

\begin{figure*}[t]
    \centering
    \setlength{\tabcolsep}{2pt}
    \renewcommand{\arraystretch}{0.4}
    \begin{tabular}{cc}
        \includegraphics[width=0.49\textwidth]{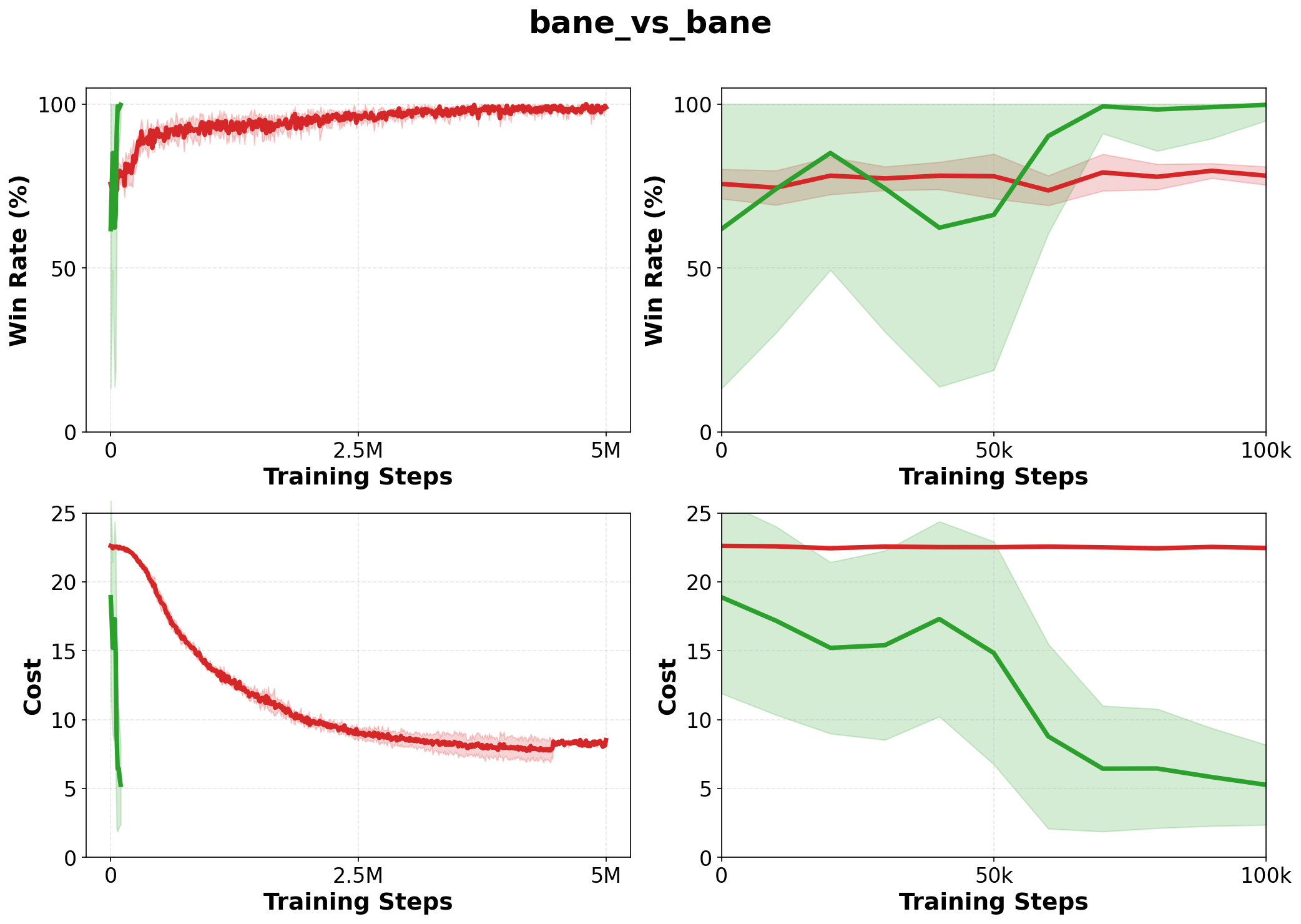} &
        \includegraphics[width=0.49\textwidth]{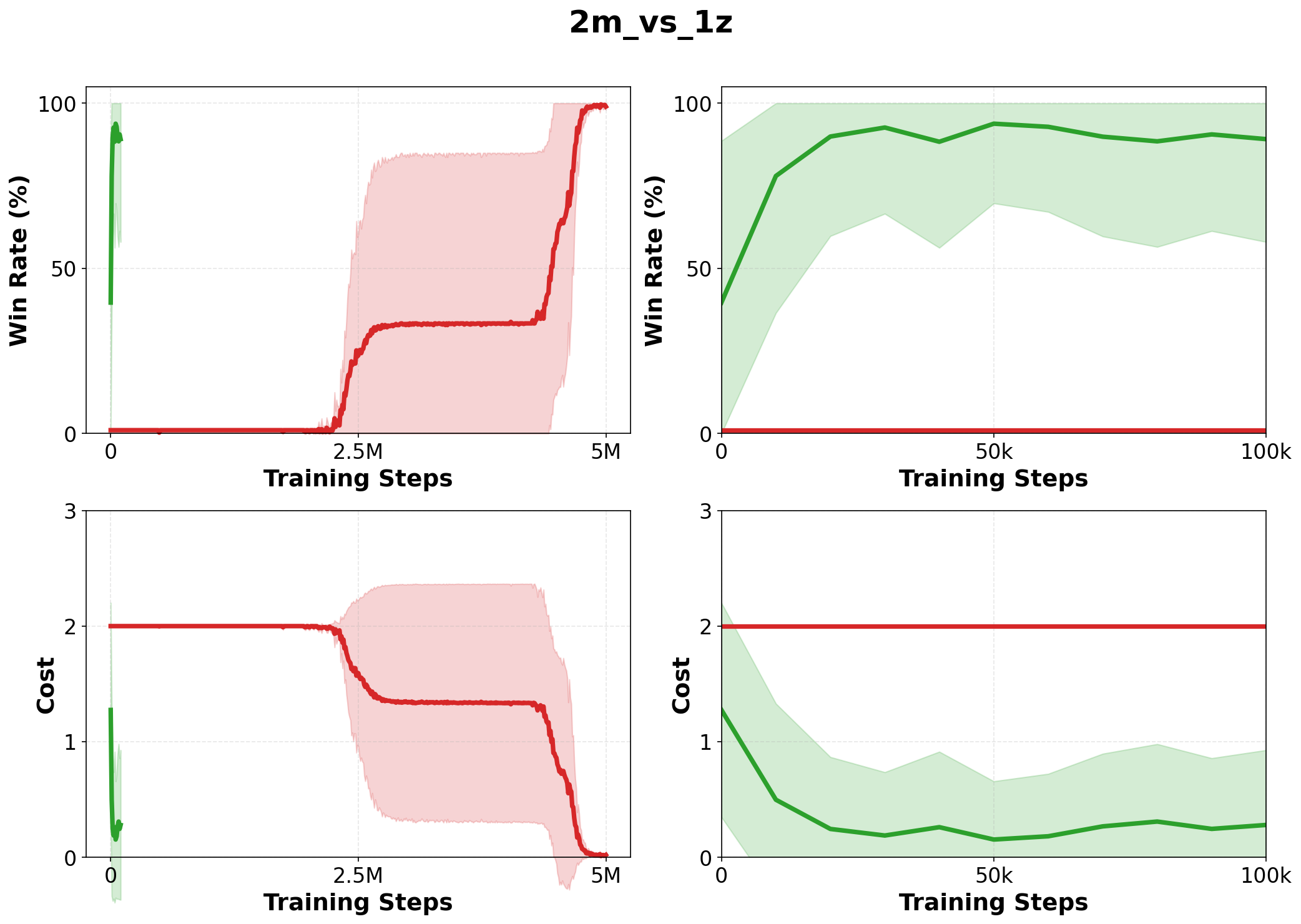} \\
        \includegraphics[width=0.49\textwidth]{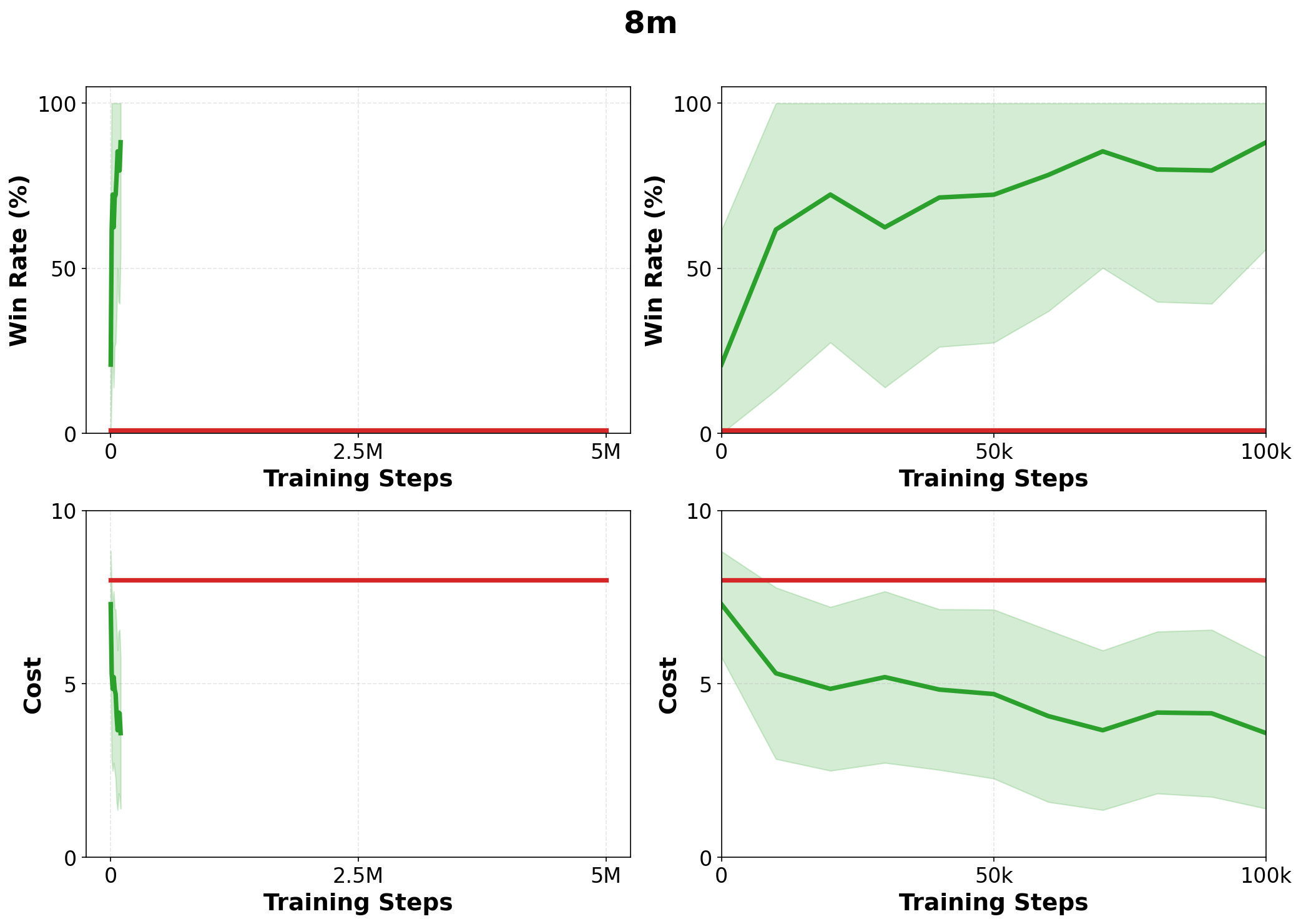} &
        \includegraphics[width=0.49\textwidth]{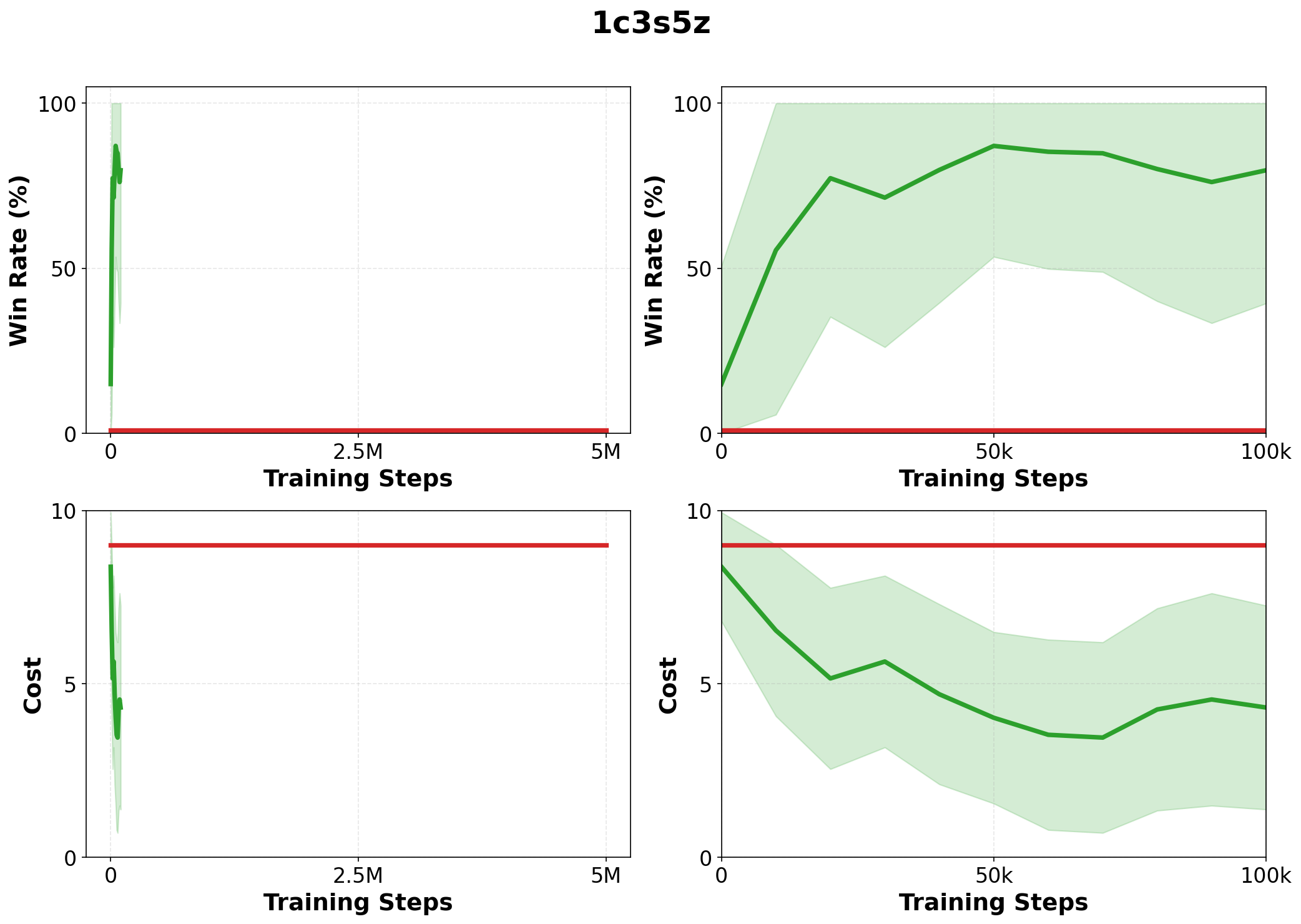} \\
    \end{tabular}
    \Description{Four panels of SMAC learning curves at d=0 for bane_vs_bane,
    2m_vs_1z, 8m and 1c3s5z, each comparing CA-MAMBA (green) with MACPO-Discrete
    (red) over the full 0 to 5M range and a zoomed 0 to 100k range.}
    \caption{Training dynamics at $d=0$.  Top row, the easier maps
    \texttt{bane\_vs\_bane} and \texttt{2m\_vs\_1z}, where MACPO-Discrete
    catches up only at 5M; bottom row, \texttt{8m} and \texttt{1c3s5z}, where it
    stays at zero win rate.  Each panel pairs the full $0$--$5$M baseline range
    with the zoomed $0$--$100$k range.  CA-MAMBA is green, MACPO-Discrete red.}
    \label{fig:curves-supp}
\end{figure*}

\section{Dense Collision-Cost Results}
\label{app:collision}

As a second SMAC cost, dense and reversible unlike casualties, each step
charges one unit for every pair of moving allied units closer than a fixed
distance.  Table~\ref{tab:collision_comparison_100k_1m} gives the per-map
endpoint results at $d=0$, with seed standard deviations.  CA-MAMBA at 100k
attains a higher score and win rate than MACPO at 1M on 9 of 10 maps,
but a higher collision cost on 5: the lead in task performance carries over to
this cost structure, the lead in cost only partly.  At 1M, MACPO still
wins no episode on 7 of the 10 maps.

\begin{table*}[t]
\centering
\small
\setlength{\tabcolsep}{4pt}
\begin{tabular}{@{}lcccccc@{}}
\toprule
& \multicolumn{3}{c}{MACPO (1M)} & \multicolumn{3}{c}{\textbf{CA-MAMBA} (100k)} \\
\cmidrule(lr){2-4}\cmidrule(l){5-7}
Scenario & Score $\uparrow$ & Cost $\downarrow$ & Win rate $\uparrow$ & Score $\uparrow$ & Cost $\downarrow$ & Win rate $\uparrow$ \\
\midrule
\texttt{2s\_vs\_1sc} & $12.23 \mathbin{\pm} 0.48$ & $\mathbf{0.00} \mathbin{\pm} 0.00$ & $0.11 \mathbin{\pm} 0.04$ & $\mathbf{19.86} \mathbin{\pm} 0.22$ & $\mathbf{0.00} \mathbin{\pm} 0.00$ & $\mathbf{0.95} \mathbin{\pm} 0.03$ \\
\texttt{3m} & $11.20 \mathbin{\pm} 0.96$ & $0.12 \mathbin{\pm} 0.03$ & $0.34 \mathbin{\pm} 0.07$ & $\mathbf{19.03} \mathbin{\pm} 0.91$ & $\mathbf{0.02} \mathbin{\pm} 0.02$ & $\mathbf{0.92} \mathbin{\pm} 0.07$ \\
\texttt{2s3z} & $9.24 \mathbin{\pm} 0.21$ & $\mathbf{0.07} \mathbin{\pm} 0.02$ & $0.00 \mathbin{\pm} 0.00$ & $\mathbf{18.72} \mathbin{\pm} 0.44$ & $0.16 \mathbin{\pm} 0.11$ & $\mathbf{0.79} \mathbin{\pm} 0.07$ \\
\texttt{3s\_vs\_3z} & $8.32 \mathbin{\pm} 0.07$ & $\mathbf{0.00} \mathbin{\pm} 0.00$ & $0.00 \mathbin{\pm} 0.00$ & $\mathbf{19.89} \mathbin{\pm} 0.05$ & $0.02 \mathbin{\pm} 0.03$ & $\mathbf{0.97} \mathbin{\pm} 0.01$ \\
\texttt{3s\_vs\_4z} & $5.77 \mathbin{\pm} 0.13$ & $\mathbf{0.00} \mathbin{\pm} 0.00$ & $0.00 \mathbin{\pm} 0.00$ & $\mathbf{15.49} \mathbin{\pm} 3.71$ & $0.26 \mathbin{\pm} 0.23$ & $\mathbf{0.06} \mathbin{\pm} 0.09$ \\
\texttt{3s\_vs\_5z} & $4.78 \mathbin{\pm} 0.26$ & $\mathbf{0.00} \mathbin{\pm} 0.00$ & $0.00 \mathbin{\pm} 0.00$ & $\mathbf{13.72} \mathbin{\pm} 2.97$ & $0.08 \mathbin{\pm} 0.07$ & $\mathbf{0.01} \mathbin{\pm} 0.03$ \\
\texttt{8m} & $7.21 \mathbin{\pm} 0.32$ & $0.70 \mathbin{\pm} 0.05$ & $0.00 \mathbin{\pm} 0.00$ & $\mathbf{17.49} \mathbin{\pm} 0.43$ & $\mathbf{0.58} \mathbin{\pm} 0.06$ & $\mathbf{0.71} \mathbin{\pm} 0.05$ \\
\texttt{MMM} & $6.92 \mathbin{\pm} 0.15$ & $1.50 \mathbin{\pm} 0.02$ & $0.00 \mathbin{\pm} 0.00$ & $\mathbf{15.12} \mathbin{\pm} 0.17$ & $\mathbf{1.20} \mathbin{\pm} 0.52$ & $\mathbf{0.39} \mathbin{\pm} 0.03$ \\
\texttt{3s5z\_vs\_3s6z} & $6.83 \mathbin{\pm} 0.07$ & $\mathbf{0.10} \mathbin{\pm} 0.01$ & $0.00 \mathbin{\pm} 0.00$ & $\mathbf{10.65} \mathbin{\pm} 0.82$ & $0.86 \mathbin{\pm} 0.08$ & $\mathbf{0.01} \mathbin{\pm} 0.01$ \\
\texttt{bane\_vs\_bane} & $\mathbf{19.66} \mathbin{\pm} 0.09$ & $0.65 \mathbin{\pm} 0.00$ & $\mathbf{0.88} \mathbin{\pm} 0.06$ & $18.20 \mathbin{\pm} 0.59$ & $\mathbf{0.15} \mathbin{\pm} 0.13$ & $0.57 \mathbin{\pm} 0.10$ \\
\bottomrule
\end{tabular}
\caption{Dense collision cost at $d=0$: CA-MAMBA after 100k interactions versus
MACPO after 1M.  Score is SMAC's episode return, which credits damage
and kills as well as the win (near $20$ for a clean win).  Mean
$\mathbin{\pm}$ sample standard deviation over three seeds; bold marks the better mean
of the two methods per metric.}
\label{tab:collision_comparison_100k_1m}
\end{table*}

\section{Safety-Mechanism Ablation}
\label{app:safety-mechanism}

Removing the cost critic and the multiplier reduces CA-MAMBA to the
unconstrained MAMBA world-model learner.  Figure~\ref{fig:safety-ablation}
contrasts the two at $d=0$.  On \texttt{1c3s5z} the constrained agent separates
within roughly 20k interactions; the curves meet again near 100k, and it ends two to three
casualties below the unconstrained one.  On \texttt{8m} the two are
comparable until about 150k and separate after it, ending about one casualty
apart.  Return does not drop in either, and on \texttt{1c3s5z} the constrained agent wins
more often: the cost-aware components buy the cost reduction without paying for it in task
performance, and the world model alone does not deliver it.  These are separate runs from
those of the main results, so we read them for shape and do not attach endpoint statistics
to them.

\begin{figure*}[t]
  \centering
  \includegraphics[width=0.49\textwidth]{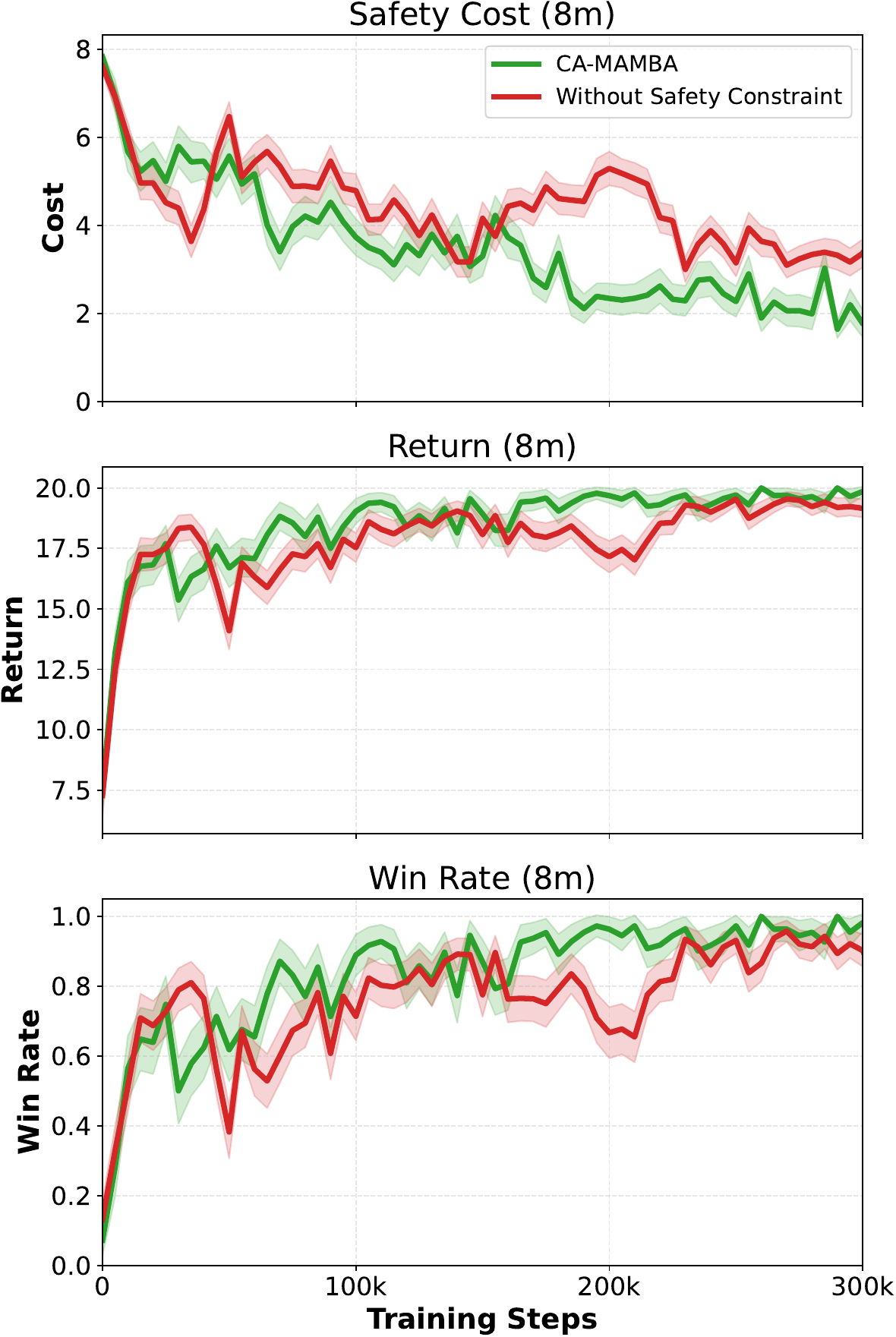}\hfill
  \includegraphics[width=0.49\textwidth]{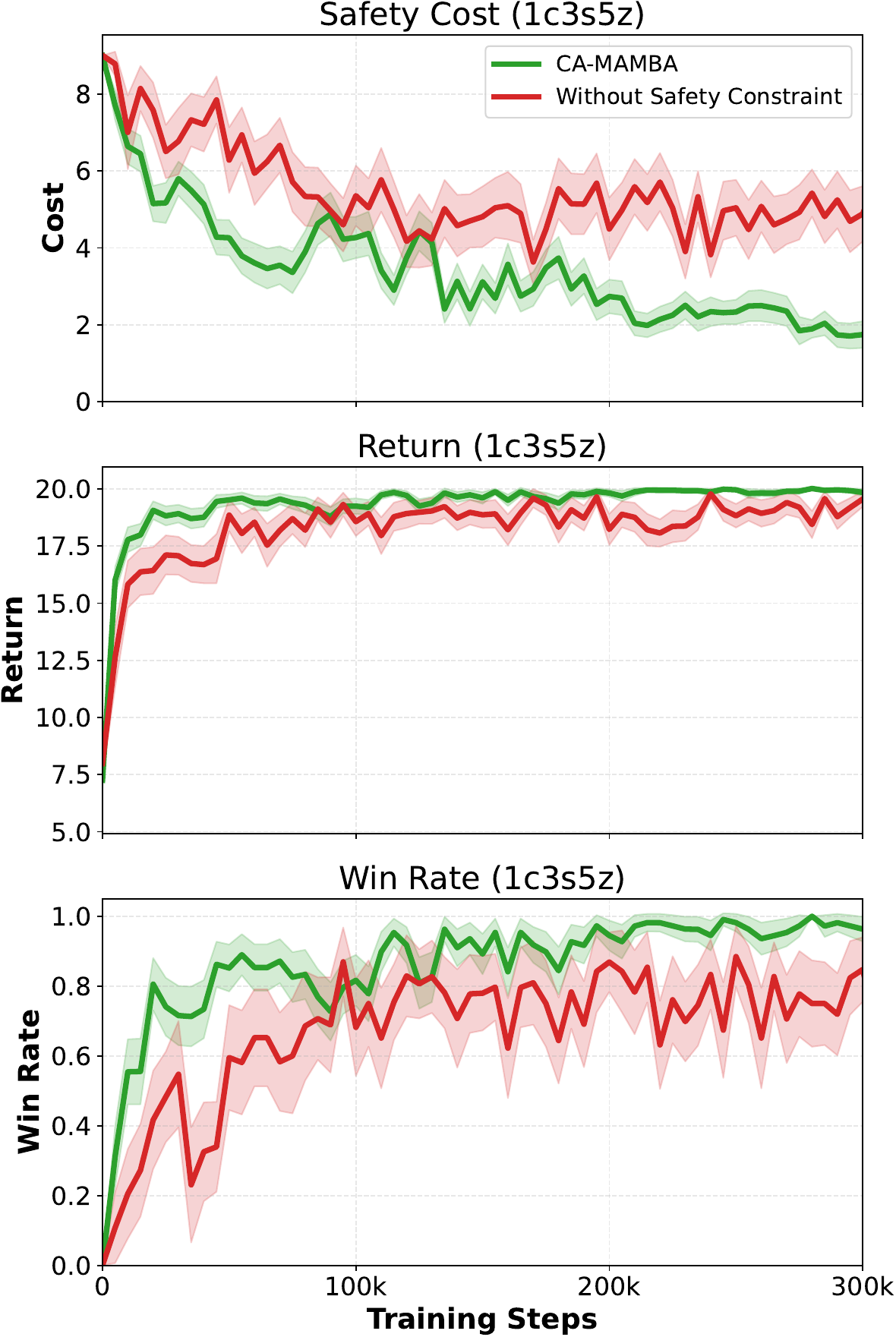}
  \Description{Two panels, 8m and 1c3s5z at d=0, comparing cost, return and win
  rate of constrained CA-MAMBA (green) and the unconstrained MAMBA learner (red)
  over training.}
  \caption{Safety-mechanism ablation on \texttt{8m} (left) and \texttt{1c3s5z}
  (right), $d=0$.  Constrained CA-MAMBA (green) versus the unconstrained
  MAMBA learner (red), over cost, return, and win rate.}
  \label{fig:safety-ablation}
\end{figure*}

\section{Cost-Model Accuracy}
\label{app:cost-accuracy}

Theorem~\ref{app:thm:cost-transfer} below needs the cost head's bias
$B_c=\mathbb E_{\nu_t}[c-\widehat c\,]$ on the states the real process visits, so
this section measures it.  Throughout the first 1M interactions
of the MAMuJoCo runs at $d=25$ and $\alpha_\beta=10^{-5}$, we log the cost head's
predicted team cost alongside the realized team cost on the same training
batches.  Both quantities are per-timestep averages over a batch rather than
episode totals, which keeps the vertical scales small (from $0$ to $0.15$ in
Figures~\ref{fig:costacc-hc} and~\ref{fig:costacc-ant2x4}).  Curves are unsmoothed;
the band is one sample standard deviation over three seeds.

Because each point averages a batch, the difference between the two curves
estimates $B_c$ directly rather than a per-sample error, and averaging its
absolute value over training gives an upper estimate of $|B_c|$, inflated by
batch noise.  That average is $0.004$ on HalfCheetah $2{\times}3$, $0.003$ on Ant
$2{\times}4$, and $0.001$ on Ant $4{\times}2$; on the first two it is roughly
$3\%$ of peak cost.  Two features matter beyond
the magnitude.  The HalfCheetah run traverses a full rise-and-fall arc as the
multiplier takes hold, and the prediction tracks it without visible lag or
systematic offset, so the head is following the cost process rather than
regressing to its mean.  And both Ant partitions spend a long opening phase at
exactly zero cost, which is the sparse regime in miniature: the head neither
invents cost there nor misses the onset when cost appears at roughly 0.4M
interactions on Ant $2{\times}4$ (Figure~\ref{fig:costacc-ant2x4}).  The visible negative excursions on
Ant $4{\times}2$ (Figure~\ref{fig:costacc-ant4x2}) are regression noise around a zero target, of about $0.01$.

Two limits on what this establishes.  Replay batches stand in for the laws
$\nu_t$ of the theorem, and the estimate pools checkpoints across training, so it
fixes the scale of $\epsilon_c$ for these runs rather than certifying it for one
policy.  And it is measured on the continuous velocity cost; the
sparse, irreversible SMAC casualty cost is not covered here, which is why
Appendix~\ref{app:remaining-evidence} still lists positive-event recall as
outstanding.  At $H_{\mathrm{im}}=14$ costed steps the measured value puts the
cost-head term of $\Delta_{C,H_{\mathrm{im}}}$ near $0.05$ per $14$-step window.
The episode budget lives on a different scale: with 1000-step episodes, $d=25$
is $0.025$ per step, or about $0.33$ per window, so even this term is about a
sixth of the budget.  The
transition term cannot be bounded from the logs at all
(Appendix~\ref{app:model-error}), so the measured gap between imagined and
realized cost, not the margin, is the evidence on transfer.

\begin{figure}[t]
  \centering
  \includegraphics[width=\columnwidth]{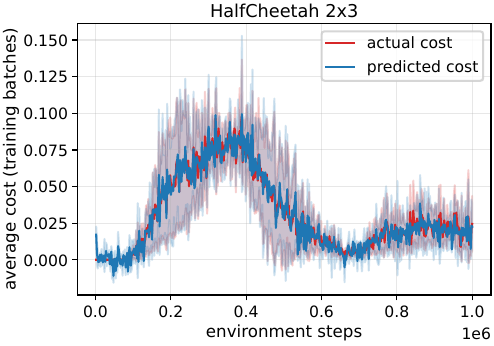}
  \Description{Predicted and realized team cost on HalfCheetah 2x3 at d=25
  against training interactions; the prediction tracks the rise and fall of
  realized cost.}
  \caption{HalfCheetah $2{\times}3$, $d=25$: predicted against realized team cost
  on training batches.  The prediction follows the full rise-and-fall arc.}
  \label{fig:costacc-hc}
\end{figure}

\begin{figure}[t]
  \centering
  \includegraphics[width=\columnwidth]{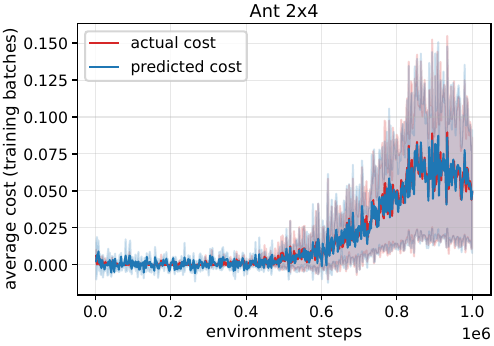}
  \Description{Predicted and realized team cost on Ant 2x4 at d=25 against
  training interactions; both stay at zero until about 0.4M interactions, then
  rise together.}
  \caption{Ant $2{\times}4$, $d=25$: cost is zero until roughly 0.4M
  interactions.  The head predicts the zero phase and the onset.}
  \label{fig:costacc-ant2x4}
\end{figure}

\begin{figure}[t]
  \centering
  \includegraphics[width=\columnwidth]{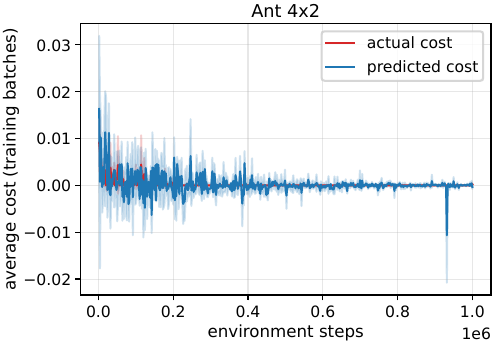}
  \Description{Predicted and realized team cost on Ant 4x2 at d=25 against
  training interactions; both stay near zero throughout.}
  \caption{Ant $4{\times}2$, $d=25$: cost stays near zero throughout and the
  prediction stays with it, within regression noise.}
  \label{fig:costacc-ant4x2}
\end{figure}

\section{Model-Error Estimates}
\label{app:model-error}

Theorem~\ref{app:thm:cost-transfer} relates the gap between imagined and real
cost to two model errors.  This section reports what training logs say about
each, and the gap itself (Table~\ref{tab:model-error}), for the $18$ runs whose
local logs record the exact length of every episode: \texttt{8m} at $d=0$ under
the learned multiplier, and HalfCheetah $2{\times}3$ under three price rules at
$d=25$ and two at $d=5$.  The realized cost per step needs those lengths, and the
exported logs of the other configurations do not record them, so we do not
estimate the gap there.  Every quantity is an average over the last $10\%$ of a
run's interactions.

The transition error is logged as the divergence between the prior and the
posterior over each agent's stochastic latent.  The logged value averages over
the latent's $32$ categorical variables, so the divergence of one agent's full
latent is $32$ times larger, $3.2$ to $6.5$ nats here.  Pinsker's and Jensen's
inequalities then bound each agent's expected total variation only by $1.27$ or
more, above its trivial value $1$, so the logs give no useful $\epsilon_P$.  The
real next-latent law is the posterior averaged over the next observation, so by
joint convexity this divergence does upper-bound the one the theorem needs.  But
the proxy cannot certify a small $\epsilon_P$ even for an exact model: it then
equals the information the next observation adds about the latent, which is
positive whenever that observation is informative.  The imagined and realized costs are both scaled to the discounted
window of $H_{\mathrm{im}}=14$ costed steps: imagined team cost per imagined
step, and realized episode cost divided by episode length, each times
$\sum_{t<14}\gamma^t\approx13.1$.  Their difference is a proxy for the
left-hand side of the theorem.  It treats realized cost as spread evenly over an
episode, which understates bursty windows, so it is a check on average transfer,
not a worst-case one.

\begin{table*}[t]
\centering
\small
\setlength{\tabcolsep}{3pt}
\begin{tabular}{@{}lclccccccc@{}}
\toprule
Env & $d$ & Price & \begin{tabular}[b]{@{}c@{}}KL per\\agent\end{tabular} & \begin{tabular}[b]{@{}c@{}}TV per\\agent\end{tabular} & \begin{tabular}[b]{@{}c@{}}$\hat B_c$\\{[}$10^{-4}${]}\end{tabular} & \begin{tabular}[b]{@{}c@{}}$\hat B_V$\\{[}$10^{-3}${]}\end{tabular} & \begin{tabular}[b]{@{}c@{}}Imagined\\14-step cost\end{tabular} & \begin{tabular}[b]{@{}c@{}}Real\\14-step cost\end{tabular} & \begin{tabular}[b]{@{}c@{}}Real $-$\\imagined\end{tabular} \\
\midrule
\texttt{8m} & 0 & Learned & $4.1 \mathbin{\pm} 0.1$ & $0.66 \mathbin{\pm} 0.01$ & $3.4 \mathbin{\pm} 1.5$ & $63.3 \mathbin{\pm} 2.7$ & $0.728 \mathbin{\pm} 0.017$ & $1.083 \mathbin{\pm} 0.128$ & $0.354 \mathbin{\pm} 0.132$ ($+33\%$) \\
\midrule
HalfCheetah $2{\times}3$ & 5 & Parity start & $5.6 \mathbin{\pm} 1.0$ & $0.76 \mathbin{\pm} 0.04$ & $-1.0 \mathbin{\pm} 1.9$ & $-0.0 \mathbin{\pm} 0.2$ & $0.014 \mathbin{\pm} 0.030$ & $0.068 \mathbin{\pm} 0.112$ & $0.054 \mathbin{\pm} 0.084$ ($+79\%^{\diamond}$) \\
 & 5 & Fixed $\beta{=}1$ & $5.4 \mathbin{\pm} 0.6$ & $0.72 \mathbin{\pm} 0.07$ & $2.0 \mathbin{\pm} 2.5$ & $0.1 \mathbin{\pm} 0.2$ & $-0.011 \mathbin{\pm} 0.009$ & $0.009 \mathbin{\pm} 0.004$ & $0.020 \mathbin{\pm} 0.012$ ($+220\%$) \\
 & 25 & Learned & $4.9 \mathbin{\pm} 0.8$ & $0.73 \mathbin{\pm} 0.04$ & $3.6 \mathbin{\pm} 6.0$ & $-0.2 \mathbin{\pm} 0.5$ & $0.037 \mathbin{\pm} 0.017$ & $0.056 \mathbin{\pm} 0.015$ & $0.020 \mathbin{\pm} 0.020$ ($+35\%^{\diamond}$) \\
 & 25 & Parity start & $4.1 \mathbin{\pm} 1.0$ & $0.68 \mathbin{\pm} 0.06$ & $0.3 \mathbin{\pm} 6.7$ & $0.0 \mathbin{\pm} 1.2$ & $0.092 \mathbin{\pm} 0.046$ & $0.281 \mathbin{\pm} 0.109$ & $0.189 \mathbin{\pm} 0.093$ ($+67\%$) \\
 & 25 & Fixed $\beta{=}1$ & $5.9 \mathbin{\pm} 0.0$ & $0.77 \mathbin{\pm} 0.01$ & $0.1 \mathbin{\pm} 2.3$ & $0.3 \mathbin{\pm} 0.2$ & $0.006 \mathbin{\pm} 0.028$ & $0.035 \mathbin{\pm} 0.025$ & $0.029 \mathbin{\pm} 0.009$ ($+83\%$) \\
\bottomrule
\end{tabular}
\caption{World-model error on the runs whose logs record every episode's exact length, three runs (seeds 1--3) per row.  Each quantity is a run's average over its last $10\%$ of interactions (100k for \texttt{8m}, 1M for HalfCheetah), then mean $\mathbin{\pm}$ sample standard deviation over runs.  \emph{Learned} is the measured-cost multiplier of the main results; \emph{Parity start} runs the same update from $\beta_0=1$; \emph{Fixed $\beta{=}1$} holds the price at parity.  KL is between the prior and the posterior over one agent's latent ($32$ categoricals), TV a Monte Carlo estimate of their total variation.  $\hat B_c$ and $\hat B_V$ are per-step estimates on replay, real minus model as in Theorem~\ref{app:thm:cost-transfer}.  Imagined and real costs are discounted windows of the $14$ costed imagined steps ($\sum_{t<14}0.99^t\approx13.1$): imagined team cost per step, and realized episode cost divided by exact episode length, each times that sum.  The cost head is an unconstrained regressor, so imagined cost can fall slightly below zero where real cost is near zero.  The percentage is $(\bar R-\bar I)/\bar R$ from the row means $\bar R$ and $\bar I$, positive when the model under-predicts; $\diamond$ marks rows whose mean gap is within one standard deviation of the per-run gaps.}
\label{tab:model-error}
\end{table*}

The model under-predicts the window cost in every row of
Table~\ref{tab:model-error} and in $17$ of the $18$ runs; the exception, a
parity-start HalfCheetah run at $d=5$, has a realized window cost of $0.004$.
The three runs of a row are too few to separate the price rules.

\paragraph{Direct diagnostics.}
The table's $\hat B_c$, $\hat B_V$, and total-variation columns, and the
\texttt{8m} runs at $d=4$ below, come from three estimates we log on replay
sequences every $50$ model updates.  $\hat B_c$ is the mean realized minus
predicted team cost.  $\hat B_V$ is the cost critic's value at the posterior
latent minus its mean over $8$ latents sampled from the prior, which estimates
the transition term $B_V$ of the theorem with the critic in place of
$\widehat V_{H_{\mathrm{im}}-t-1}$; the critic attends over all agents, so the
estimate includes the inter-agent dependence $\eta$ of
Proposition~\ref{app:prop:factorized}.  The third is a Monte Carlo estimate of
each agent's total variation between posterior and prior over the $32$
categorical variables.  By convexity it upper-bounds that agent's one-step
transition error under the replay distribution, and like the divergence it is
positive even for an exact model.  The sampling uses its own seeded generator
and restores the training random state, so training is unaffected.

On HalfCheetah $2{\times}3$ at $d=25$ under the learned multiplier (means $\pm$
sample standard deviations over three runs, signs as in the theorem, real minus
model), $\hat B_c=(3.6\pm6.0)\cdot10^{-4}$ per step, $2\%\pm14\%$ of the
realized cost; $\hat B_V=(-2.3\pm5.2)\cdot10^{-4}$ per step, with mean absolute
value $0.0086$ against a cost-critic value of $0.24$; and the per-agent total
variation is $0.73\pm0.04$; this upper estimate of $\epsilon_P^i$ is far above the
$0.16$ beyond which part~(ii) is vacuous, so the logs cannot certify a useful
$\epsilon_P$.  Over the $14$-step window the two one-step terms sum
to $\sum_{t<14}\gamma^t(\hat B_c+\gamma\hat B_V)\approx0.002$, while realized
cost exceeds imagined cost by $0.020\pm0.020$ ($35\%$ of the mean realized
cost; $68\%$, $23\%$, and $4\%$ in the three runs).  The estimates pool replay
states, whereas the theorem's terms are taken along the real process from the
imagination start and use the model's own cost-to-go
$\widehat V_{H_{\mathrm{im}}-t-1}$ rather than the critic, so the difference
bounds neither term.  The multiplier ends at $\beta=0.44\pm0.12$.

On \texttt{8m} at $d=0$ the cost head is as accurate, $\hat B_c=(3.4\pm1.5)
\cdot10^{-4}$ per step against a realized $0.083$, but $\hat B_V$ is
$0.063\pm0.003$ per step, $7\%$ to $8\%$ of the critic's value, and positive in
all three runs: the model under-predicts the cost-to-go of the next state, in
the same direction as the window gap of $0.35\pm0.13$, a third of the realized
cost.  The multiplier ends at $\beta=0.54\pm0.02$.

On \texttt{8m} at $d=4$ (Table~\ref{tab:transition-diag}), averaged over the
second half of training, $\hat B_V$ is positive for every multiplier rule of Section~\ref{sec:dual-ablation}, run at this limit, from
$2.5\%$ to $5\%$ of the critic's cost value; for the measured-cost multiplier it
is $0.069\pm0.005$ per step against a value of $1.85$.  The model therefore
under-predicts the cost-to-go of the next state.  The per-agent total variation
is $0.45$ to $0.66$, again too large to certify $\epsilon_P<0.16$.  The \texttt{8m} export does not
include the cost head's average prediction, so $\hat B_c$ is not available
there.

\begin{table}[t]
\centering
\small
\setlength{\tabcolsep}{4pt}
\begin{tabular}{@{}lccccc@{}}
\toprule
Multiplier & $\hat{B}_V$ & $\overline{|\delta|}$ & $\bar{V}_C$ & $\widehat{\mathrm{TV}}$ & $\widehat{\mathrm{TV}}_{\max}$ \\
\midrule
Imagined & \begin{tabular}[t]{@{}c@{}}$0.06$\\$\pm 0.01$\end{tabular} & \begin{tabular}[t]{@{}c@{}}$0.23$\\$\pm 0.01$\end{tabular} & \begin{tabular}[t]{@{}c@{}}$2.39$\\$\pm 0.38$\end{tabular} & \begin{tabular}[t]{@{}c@{}}$0.65$\\$\pm 0.00$\end{tabular} & \begin{tabular}[t]{@{}c@{}}$0.70$\\$\pm 0.00$\end{tabular} \\
Measured & \begin{tabular}[t]{@{}c@{}}$0.07$\\$\pm 0.00$\end{tabular} & \begin{tabular}[t]{@{}c@{}}$0.21$\\$\pm 0.01$\end{tabular} & \begin{tabular}[t]{@{}c@{}}$1.85$\\$\pm 0.04$\end{tabular} & \begin{tabular}[t]{@{}c@{}}$0.64$\\$\pm 0.01$\end{tabular} & \begin{tabular}[t]{@{}c@{}}$0.69$\\$\pm 0.01$\end{tabular} \\
PID & \begin{tabular}[t]{@{}c@{}}$0.01$\\$\pm 0.00$\end{tabular} & \begin{tabular}[t]{@{}c@{}}$0.04$\\$\pm 0.02$\end{tabular} & \begin{tabular}[t]{@{}c@{}}$0.28$\\$\pm 0.14$\end{tabular} & \begin{tabular}[t]{@{}c@{}}$0.45$\\$\pm 0.05$\end{tabular} & \begin{tabular}[t]{@{}c@{}}$0.50$\\$\pm 0.07$\end{tabular} \\
$\beta{=}5$ & \begin{tabular}[t]{@{}c@{}}$0.01$\\$\pm 0.00$\end{tabular} & \begin{tabular}[t]{@{}c@{}}$0.02$\\$\pm 0.01$\end{tabular} & \begin{tabular}[t]{@{}c@{}}$0.13$\\$\pm 0.07$\end{tabular} & \begin{tabular}[t]{@{}c@{}}$0.63$\\$\pm 0.04$\end{tabular} & \begin{tabular}[t]{@{}c@{}}$0.69$\\$\pm 0.02$\end{tabular} \\
$\beta{=}1$ & \begin{tabular}[t]{@{}c@{}}$0.04$\\$\pm 0.01$\end{tabular} & \begin{tabular}[t]{@{}c@{}}$0.12$\\$\pm 0.03$\end{tabular} & \begin{tabular}[t]{@{}c@{}}$0.82$\\$\pm 0.29$\end{tabular} & \begin{tabular}[t]{@{}c@{}}$0.52$\\$\pm 0.05$\end{tabular} & \begin{tabular}[t]{@{}c@{}}$0.56$\\$\pm 0.05$\end{tabular} \\
$\beta{=}0$ & \begin{tabular}[t]{@{}c@{}}$0.06$\\$\pm 0.01$\end{tabular} & \begin{tabular}[t]{@{}c@{}}$0.22$\\$\pm 0.02$\end{tabular} & \begin{tabular}[t]{@{}c@{}}$2.19$\\$\pm 0.19$\end{tabular} & \begin{tabular}[t]{@{}c@{}}$0.66$\\$\pm 0.01$\end{tabular} & \begin{tabular}[t]{@{}c@{}}$0.71$\\$\pm 0.01$\end{tabular} \\
\bottomrule
\end{tabular}
\caption{World-model transition diagnostics on SMAC \texttt{8m} at $d=4$, logged on replayed sequences during training and averaged over the second half of training (50k to 100k environment interactions); each entry is a mean over the seeds that logged it, with the sample standard deviation ($\pm$) below it (three per row). For a replay step with deterministic state $h$, $\delta=V_C(h,z^{q})-\frac{1}{K}\sum_{k=1}^{K}V_C(h,z^{p,k})$ compares the cost critic $V_C$ at a posterior latent $z^{q}$ with its mean over $K$ prior latents $z^{p,k}$. $\hat{B}_V$ is the mean of $\delta$, an estimate of $B_V$ in Theorem~\ref{app:thm:cost-transfer}(i) with the cost critic as the cost-to-go. $\overline{|\delta|}$ is the mean of $|\delta|$. $\bar{V}_C$ is the mean critic value at posterior latents, the scale for the first two columns. $\widehat{\mathrm{TV}}$ is the Monte Carlo total-variation distance between an agent's posterior and prior over its whole latent, averaged over agents, which upper-bounds the latent-space per-agent transition error $\epsilon_P^i$ of Proposition~\ref{app:prop:factorized} on replay states. $\widehat{\mathrm{TV}}_{\max}$ is the largest per-agent average. Rows list the arms whose runs logged these diagnostics.}
\label{tab:transition-diag}
\end{table}

\section{Multiplier Dynamics}
\label{app:beta}

Proposition~\ref{app:prop:dual} below says the measured-cost update has the right
sign, and the counterexample of Appendix~\ref{app:counterexample} shows it cannot repair a cost-blind actor gradient.  These trajectories show what the
update does when the cost head is accurate, which
Appendix~\ref{app:cost-accuracy} establishes for these runs.  We log $\beta$ at
every update together with the realized episode cost, over the first 1M
interactions at $d=25$ and $\alpha_\beta=10^{-5}$, mean and one standard
deviation over three seeds.

All three environments show the behavior the dual prescribes, in the three
regimes available.  On HalfCheetah $2{\times}3$ (Figure~\ref{fig:beta-hc})
episode cost overshoots the limit early, $\beta$ climbs from zero to $0.64$ while
it does, cost returns below $d$ near 0.4M interactions, and $\beta$ then decays
and settles.  On Ant $2{\times}4$ (Figure~\ref{fig:beta-ant2x4}) the constraint
is inactive for the first 0.4M interactions and $\beta$ correctly stays at zero,
then rises once cost appears.  On Ant $4{\times}2$
(Figure~\ref{fig:beta-ant4x2}) cost never reaches the limit and $\beta$ never
leaves zero, so the constraint costs the policy nothing where it does not bind.
No run shows oscillation or divergence, and $\beta$ stays below $1.1$
throughout.

This is the regime on the favorable side of the counterexample in
Appendix~\ref{app:counterexample}: there the cost head is blind and $\beta$
diverges without improving the policy, here $\epsilon_c$ is small and the loop
closes.  The trajectories demonstrate the mechanism working under a measured
error, not that a cost-blind failure cannot occur.

\begin{figure}[t]
  \centering
  \includegraphics[width=\columnwidth]{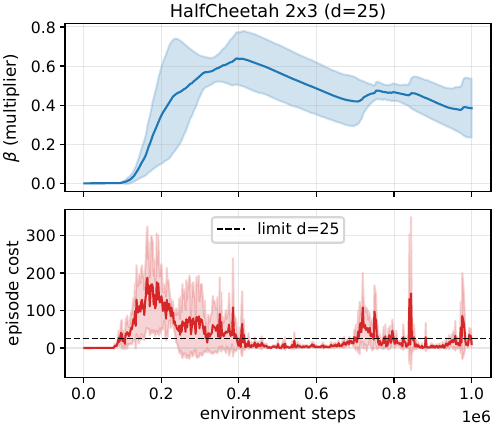}
  \Description{The multiplier beta and episode cost on HalfCheetah 2x3 at d=25
  against training interactions; beta rises while cost exceeds the limit and
  falls once cost is back below it.}
  \caption{HalfCheetah $2{\times}3$, $d=25$: $\beta$ rises while episode cost
  exceeds the limit, then relaxes once cost returns below it.}
  \label{fig:beta-hc}
\end{figure}

\begin{figure}[t]
  \centering
  \includegraphics[width=\columnwidth]{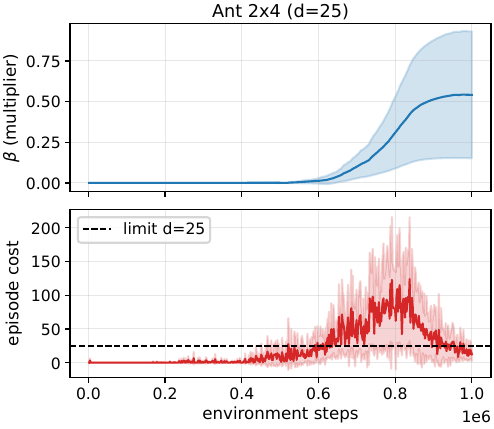}
  \Description{The multiplier beta and episode cost on Ant 2x4 at d=25 against
  training interactions; beta stays at zero while cost is below the limit and
  rises when the limit becomes active.}
  \caption{Ant $2{\times}4$, $d=25$: $\beta$ holds at zero while the constraint
  is inactive and rises when it activates.}
  \label{fig:beta-ant2x4}
\end{figure}

\begin{figure}[t]
  \centering
  \includegraphics[width=\columnwidth]{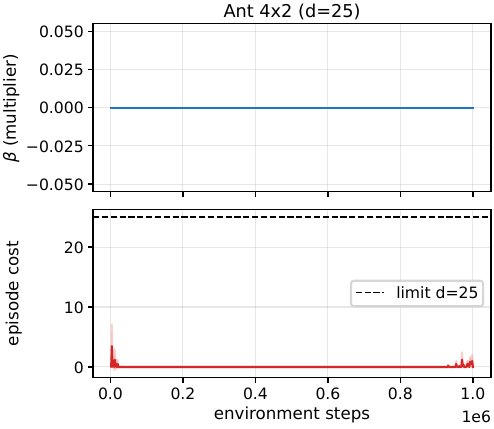}
  \Description{The multiplier beta and episode cost on Ant 4x2 at d=25 against
  training interactions; cost stays below the limit and beta stays at zero.}
  \caption{Ant $4{\times}2$, $d=25$: cost never exceeds the limit and $\beta$
  never leaves zero.}
  \label{fig:beta-ant4x2}
\end{figure}

\paragraph{SMAC at $d=0$.}
Figure~\ref{fig:beta-smac} shows the same quantities on SMAC, from the $d=0$
runs of Table~\ref{tab:full-results-aamas} (Appendix~\ref{app:full-results}), which use the measured-cost multiplier with
$\alpha_\beta=10^{-5}$.  A zero casualty budget is never met in combat, so
$\bar C_k-d$ is never negative and $\beta$ never decreases.  After 100k
interactions it has reached $0.54$ on \texttt{8m}
(Table~\ref{tab:mechanism-8m}), $0.50$ on \texttt{MMM}, $0.30$ on
\texttt{1c3s5z}, $0.13$ on \texttt{3m}, and $0.025$ on \texttt{2m\_vs\_1z},
means over three runs.  The reward and cost advantages are each standardized, so cost is
still weighted below reward when training stops.  Casualty cost falls while $\beta$
rises, which is the mechanism working; on \texttt{8m} the measured-cost multiplier
ends with most of the reduction that the parity price $\beta=1$ gives (Table~\ref{tab:mechanism-8m}), and no rule there meets the zero-casualty limit within 100k.

\begin{figure*}[t]
  \centering
  \includegraphics[width=0.48\textwidth]{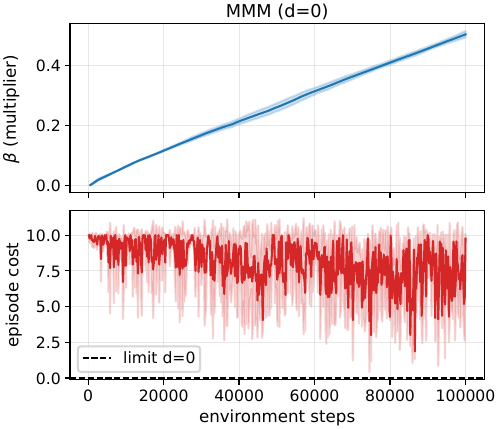}\hfill
  \includegraphics[width=0.48\textwidth]{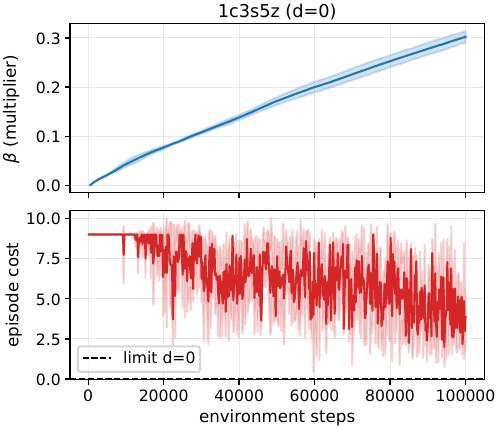}
  \Description{Two panels of multiplier and episode cost against environment
  interactions for MMM and 1c3s5z at a zero casualty limit.}
  \caption{SMAC at $d=0$ with the measured-cost multiplier,
  $\alpha_\beta=10^{-5}$: \texttt{MMM} (left) and \texttt{1c3s5z} (right).  The
  multiplier $\beta$ (top, mean and one standard deviation over three seeds)
  rises linearly and stays below $1$ at 100k interactions, while episode cost
  (bottom) falls but stays above the limit.}
  \label{fig:beta-smac}
\end{figure*}

\section{Evidence Needed for the Remaining Claims}
\label{app:remaining-evidence}

Three questions remain open, and we state what each would take rather than leave
them implicit.

\paragraph{Cost accumulated during learning.}
The endpoint comparisons in this paper say what a policy costs once trained, not
what training itself costs.  We can report the second on MAMuJoCo, and the
answer splits.  Summing episode cost over training on Ant $2{\times}4$ at
$d=25$, CA-MAMBA accumulates roughly $23$k over its 1M interactions against
MACPO's $93$k over 10M, so reaching a stronger policy costs a quarter as much
cost in total.  Restricted to matched interaction counts the ordering reverses:
over the same first 1M interactions MACPO accumulates less total cost and
almost no excess above $d$, while CA-MAMBA's cumulative excess
$E_K=\sum_{e=1}^{K}[C_e-d]_+$ reaches roughly $14.5$k, concentrated between 0.6M
and 0.9M interactions, while the multiplier is still rising from zero
(Table~\ref{tab:mamujoco-schedule}, Figure~\ref{fig:beta-ant2x4}).  Sample efficiency therefore reduces the cost of
\emph{getting} to a policy without making learning safer per
interaction.  Reporting $E_K$ and the violation count
$V_K=\sum_{e=1}^{K}\mathbf 1\{C_e>d\}$ on SMAC, where cost is irreversible, is
the measurement that would tell us whether that overshoot is
specific to a reversible velocity cost.

\paragraph{Positive-event recall under sparse cost.}
Appendix~\ref{app:cost-accuracy} measures $\epsilon_c$ on the continuous
velocity cost.  The SMAC casualty cost is sparser and irreversible, and the
quantity that would settle it is recall on positive events, with area under the
precision--recall curve and the absolute error of the $14$-step cumulative
predicted cost, evaluated on held-out replay without retraining.

\paragraph{Communication trained end to end.}
The communication ablation masks messages only at action selection, from a
model trained with them, and is reported on MAMuJoCo.  Training a self-only
model end to end, with the mask applied in world-model learning and imagination
as well, would show whether a domain needs communication rather than whether a
trained policy depends on it, and would measure how the neighborhood changes the
per-agent transition errors of Proposition~\ref{app:prop:factorized}.

\section{Proofs and Sparse-Cost Boundary Case}
\label{app:theory_proofs}

This section proves the results of Section~\ref{sec:theory} in order: the cost-gap theorem (Appendix~\ref{app:proof-gap}), its extension to bootstrapped returns (Appendix~\ref{app:proof-bootstrap}), the factorized transition error (Appendix~\ref{app:proof-factorized}), the measured-cost dual (Appendix~\ref{app:proof-dual}), and the sparse-cost counterexample (Appendix~\ref{app:counterexample}).

\subsection{Proof of the Imagination-to-Reality Cost Gap}
\label{app:proof-gap}

We restate the main text's imagination-to-reality theorem and prove it.  The
argument is the
telescoping form of the simulation lemma \cite{DBLP:journals/ml/KearnsS02},
applied to the imagined team cost.

\begin{theorem}[Imagination-to-reality cost gap]
\label{app:thm:cost-transfer}
Fix a joint policy $\pi$, an imagination horizon $H_{\mathrm{im}}\geq1$, and a
discount $\gamma\in[0,1]$.  Let $\mathcal X$ be a standard Borel space of joint
information states $x=(x^1,\ldots,x^N)$, where $x^i$ is what agent $i$
conditions on, so that $\pi(a\mid x)=\prod_{i=1}^{N}\pi^i(a^i\mid x^i)$ is well
defined.  Along the real process, let $P_t(\cdot\mid x,a)$ be a regular
conditional law of $x_{t+1}$ given $(x_t,a_t)=(x,a)$, and let $c_t(x,a)$ be the
conditional expectation of the team cost at step $t$ given $(x_t,a_t)=(x,a)$.
Let $\widehat P$ be the learned kernel on $\mathcal X$ and $\widehat c$ the
predicted team cost, and let $c_t$ and $\widehat c$ take values in a common
interval of length $c_{\max}<\infty$.  With
\[
\begin{aligned}
J_{C,H_{\mathrm{im}}}(\pi)
&=\mathbb E\!\left[\sum_{t=0}^{H_{\mathrm{im}}-1}
\gamma^tc_t(x_t,a_t)\right],\\
\widehat J_{C,H_{\mathrm{im}}}(\pi)
&=\mathbb E_{\widehat P,\pi}\!\left[\sum_{t=0}^{H_{\mathrm{im}}-1}
\gamma^t\widehat c(x_t,a_t)\right],
\end{aligned}
\]
the first along the real process, assume the two processes share an
initial-state law and that in both the action $a_t$ is drawn from
$\pi(\cdot\mid x_t)$.  Let $\nu_t$ be the law of $(x_t,a_t)$ under the real
process, and define the model's $k$-step cost-to-go by $\widehat V_0\equiv0$ and
\[
\widehat V_k(x)=\mathbb E_{a\sim\pi(\cdot\mid x)}\Bigl[\widehat c(x,a)
+\gamma\,\mathbb E_{x'\sim\widehat P(\cdot\mid x,a)}\widehat V_{k-1}(x')\Bigr]
\]
for $k\geq1$.  Then \emph{(i)}
\[
J_{C,H_{\mathrm{im}}}(\pi)-\widehat J_{C,H_{\mathrm{im}}}(\pi)
=\sum_{t=0}^{H_{\mathrm{im}}-1}\gamma^t\bigl(B_c(t)+\gamma B_V(t)\bigr),
\]
with $B_c(t)=\mathbb E_{\nu_t}[c_t-\widehat c\,]$ and
$B_V(t)=\mathbb E_{(x,a)\sim \nu_t}[(\mathbb E_{P_t(\cdot\mid x,a)}-\mathbb
E_{\widehat P(\cdot\mid x,a)})\widehat V_{H_{\mathrm{im}}-t-1}]$.  \emph{(ii)} If moreover
$|B_c(t)|\leq\epsilon_c$ and
\[
\mathbb E_{(x,a)\sim\nu_t}\bigl[D_{\mathrm{TV}}(P_t(\cdot\mid x,a),\widehat P(\cdot\mid x,a))\bigr]\leq\epsilon_P
\]
for every $t<H_{\mathrm{im}}$, then
\[
\begin{aligned}
\left|J_{C,H_{\mathrm{im}}}(\pi)
-\widehat J_{C,H_{\mathrm{im}}}(\pi)\right|
&\leq\Delta_{C,H_{\mathrm{im}}},\\
\Delta_{C,H_{\mathrm{im}}}
&:=\epsilon_c\sum_{t=0}^{H_{\mathrm{im}}-1}\gamma^t\\
&\qquad+c_{\max}\epsilon_P\sum_{t=0}^{H_{\mathrm{im}}-1}t\gamma^t,
\end{aligned}
\]
and consequently
$\widehat J_{C,H_{\mathrm{im}}}(\pi)\leq b-\Delta_{C,H_{\mathrm{im}}}$ implies
$J_{C,H_{\mathrm{im}}}(\pi)\leq b$ for any threshold $b$.
\end{theorem}

\noindent Because $P_t$ and $c_t$ are conditional laws along the real process,
the theorem needs no Markov or sufficiency assumption on the latents: it holds
for the network's own latents $x_t^i=(h_t^i,z_t^i)$ of Section~\ref{sec:method}.  Imagination branches from posterior replay states, so the shared
initial law is the replay distribution, and the result governs an
$H_{\mathrm{im}}$-step cost-to-go from visited states.  The assumptions of the
bound are averages under the real process; pointwise bounds on every reachable
state--action pair imply them, so the bound also holds under that stronger
worst-case reading.  The learned cost head is unclipped regression, and the bound
evaluates $\widehat V_{H_{\mathrm{im}}-t-1}$ at next states drawn from $\widehat P$
as well as from $P_t$, so $c_{\max}$ must cover the range of $\widehat c$ on every
state the model's rollouts reach from visited states, including imagined states
the real system never visits, and not only the range of the real cost.

\begin{proof}
\emph{Decomposition.}  Write $H=H_{\mathrm{im}}$ and, for
$t\in\{0,\ldots,H\}$, let $D_t=\gamma^t\,\mathbb E[\widehat V_{H-t}(x_t)]$,
where $x_t$ follows the real process.  The two processes share the initial
law, so $D_0=\mathbb E[\widehat V_H(x_0)]=\widehat J_{C,H}(\pi)$, and
$D_H=0$ because $\widehat V_0\equiv0$.  In the real process $a_t$ is drawn from
$\pi(\cdot\mid x_t)$, the same policy that defines $\widehat V$, so the
definition of $\widehat V_{H-t}$ gives
\[
\mathbb E\bigl[\widehat V_{H-t}(x_t)\bigr]
=\mathbb E_{(x,a)\sim \nu_t}\Bigl[\widehat c(x,a)
+\gamma\,\mathbb E_{\widehat P(\cdot\mid x,a)}\widehat V_{H-t-1}\Bigr].
\]
By the definition of $P_t$ and the tower property,
\[
\mathbb E\bigl[\widehat V_{H-t-1}(x_{t+1})\bigr]
=\mathbb E_{(x,a)\sim \nu_t}\Bigl[\mathbb E_{P_t(\cdot\mid x,a)}\widehat V_{H-t-1}\Bigr].
\]
Subtracting $\gamma^{t+1}$ times the second identity from $\gamma^t$ times the
first,
\[
D_t-D_{t+1}=\gamma^t\,\mathbb E_{\nu_t}\Bigl[\widehat c
+\gamma\bigl(\mathbb E_{\widehat P}-\mathbb E_{P_t}\bigr)\widehat V_{H-t-1}\Bigr].
\]
Summing over $t=0,\ldots,H-1$ telescopes the left side to
$D_0-D_H=\widehat J_{C,H}(\pi)$.  By the tower property,
$J_{C,H}(\pi)=\sum_{t=0}^{H-1}\gamma^t\,\mathbb E_{\nu_t}[c_t]$, so subtraction
gives
\[
J_{C,H}(\pi)-\widehat J_{C,H}(\pi)
=\sum_{t=0}^{H-1}\gamma^t\Bigl(\mathbb E_{\nu_t}[c_t-\widehat c\,]
+\gamma\,\mathbb E_{\nu_t}\bigl[(\mathbb E_{P_t}-\mathbb E_{\widehat P})
\widehat V_{H-t-1}\bigr]\Bigr),
\]
which is the decomposition.

\emph{Bound.}  If $\widehat c$ takes values in $[m,m+c_{\max}]$, induction on
$k$ shows that $\widehat V_k$ takes values in
$[m\sum_{j=0}^{k-1}\gamma^j,(m+c_{\max})\sum_{j=0}^{k-1}\gamma^j]$, an interval
of length $c_{\max}\sum_{j=0}^{k-1}\gamma^j$.  For any measurable $f$ with values
in an interval of length $L$ and probability laws $p$ and $q$,
$|\mathbb E_p[f]-\mathbb E_q[f]|\leq L\,D_{\mathrm{TV}}(p,q)$.  Hence, at every
$(x,a)$,
\[
\bigl|(\mathbb E_{P_t}-\mathbb E_{\widehat P})\widehat V_{H-t-1}\bigr|
\leq c_{\max}\Bigl(\sum_{j=0}^{H-t-2}\gamma^j\Bigr)
D_{\mathrm{TV}}\bigl(P_t(\cdot\mid x,a),\widehat P(\cdot\mid x,a)\bigr),
\]
with an empty sum, and hence zero, at $t=H-1$.  Averaging over $\nu_t$ and using
the assumption, $|B_V(t)|\leq c_{\max}\epsilon_P\sum_{j=0}^{H-t-2}\gamma^j$.
With $|B_c(t)|\leq\epsilon_c$ and the triangle inequality,
\[
\left|J_{C,H}(\pi)-\widehat J_{C,H}(\pi)\right|
\leq\epsilon_c\sum_{t=0}^{H-1}\gamma^t
+c_{\max}\epsilon_P\sum_{t=0}^{H-1}\sum_{j=0}^{H-t-2}\gamma^{t+1+j}.
\]
In the double sum, each exponent $s=t+1+j$ with $1\leq s\leq H-1$ arises from
exactly $s$ pairs $(t,j)$, namely $t=0,\ldots,s-1$.  The double sum therefore
equals $\sum_{s=0}^{H-1}s\gamma^s$, and the right side is
$\Delta_{C,H}$.

For $\gamma<1$, if the assumptions hold for every $t\geq0$, absolute
convergence permits $H\to\infty$, and
\[
\sum_{t=0}^{\infty}\gamma^t=\frac{1}{1-\gamma},
\qquad
\sum_{t=0}^{\infty}t\gamma^t
=
\frac{\gamma}{(1-\gamma)^2}.
\]
Substitution gives the infinite-horizon margin
\[
\Delta_{C,\infty}
=
\frac{\epsilon_c}{1-\gamma}
+
\frac{\gamma c_{\max}\epsilon_P}{(1-\gamma)^2},
\]
which the main text does not use because the implementation imagines a fixed
$H_{\mathrm{im}}=14$ costed steps.  Setting $\gamma=1$ instead gives the undiscounted
specialization over the same finite horizon,
\[
\Delta_{C,H}
=
H\epsilon_c+\frac{H(H-1)}{2}c_{\max}\epsilon_P.
\]
Finally, if $\widehat J_{C,H}(\pi)\leq b-\Delta_{C,H}$, then
\[
J_{C,H}(\pi)
\leq
\widehat J_{C,H}(\pi)+\Delta_{C,H}
\leq b.
\]
Each implication uses only the assumptions stated in the theorem.
\end{proof}

\subsection{Bootstrapped Returns}
\label{app:proof-bootstrap}

The actor's targets are $\lambda$-returns that bootstrap from the cost critic at
the imagination horizon.  This subsection extends the decomposition to them.

\begin{corollary}[Bootstrapped returns]
\label{app:cor:bootstrap}
Keep the setting of Theorem~\ref{app:thm:cost-transfer}, let $\gamma<1$, let
the real team costs be bounded, and let $V:\mathcal X\to\mathbb R$ be bounded
and measurable (the implementation's critic head is unconstrained in sign).  Define $\widehat V^{V}_0=V$ and, for $k\geq1$,
\[
\widehat V^{V}_k(x)=\mathbb E_{a\sim\pi(\cdot\mid x)}\Bigl[\widehat c(x,a)
+\gamma\,\mathbb E_{x'\sim\widehat P(\cdot\mid x,a)}\widehat V^{V}_{k-1}(x')\Bigr],
\]
so that $\widehat J^{(n)}_C(\pi)=\mathbb E[\widehat V^{V}_n(x_0)]
=\mathbb E_{\widehat P,\pi}[\sum_{t<n}\gamma^t\widehat c(x_t,a_t)+\gamma^nV(x_n)]$.
Along the real process, let $J^\gamma_C(\pi)=\mathbb
E[\sum_{t\geq0}\gamma^tc_t(x_t,a_t)]$ and let
$G_n=\sum_{s\geq n}\gamma^{s-n}c_s$ be the realized discounted cost from step
$n$, where $c_s$ is the realized team cost at step $s$.  Then, for every $n\geq1$,
\[
\begin{aligned}
J^\gamma_C(\pi)-\widehat J^{(n)}_C(\pi)
={}&\sum_{t=0}^{n-1}\gamma^t\Bigl(B_c(t)+\gamma\,\mathbb E_{\nu_t}\bigl[
(\mathbb E_{P_t}-\mathbb E_{\widehat P})\widehat V^{V}_{n-t-1}\bigr]\Bigr)\\
&+\gamma^n\,\mathbb E\bigl[G_n-V(x_n)\bigr].
\end{aligned}
\]
For the truncated $\lambda$-return
$\widehat J^{\lambda}_C=\sum_{n=1}^{H_{\mathrm{im}}}w_n\widehat J^{(n)}_C$ with
$w_n=(1-\lambda)\lambda^{n-1}$ for $n<H_{\mathrm{im}}$ and
$w_{H_{\mathrm{im}}}=\lambda^{H_{\mathrm{im}}-1}$,
$J^\gamma_C-\widehat J^{\lambda}_C=\sum_{n=1}^{H_{\mathrm{im}}}w_n\bigl(J^\gamma_C-\widehat
J^{(n)}_C\bigr)$.
\end{corollary}

\begin{proof}
All quantities are finite: the costs and $V$ are bounded and $\gamma<1$.  Let
$D_t=\gamma^t\,\mathbb E[\widehat V^{V}_{n-t}(x_t)]$ along the real process for
$t\in\{0,\ldots,n\}$.  The initial laws agree, so $D_0=\widehat J^{(n)}_C(\pi)$,
and $D_n=\gamma^n\,\mathbb E[V(x_n)]$ because $\widehat V^{V}_0=V$.  The two
one-step identities in the proof of Theorem~\ref{app:thm:cost-transfer} use only
the recursion defining $\widehat V$ for $k\geq1$, not its terminal value, so
they hold here for $t\leq n-1$ and give
\[
D_t-D_{t+1}=\gamma^t\,\mathbb E_{\nu_t}\Bigl[\widehat c
+\gamma\bigl(\mathbb E_{\widehat P}-\mathbb E_{P_t}\bigr)\widehat V^{V}_{n-t-1}\Bigr].
\]
Summing over $t=0,\ldots,n-1$,
\[
\widehat J^{(n)}_C(\pi)-\gamma^n\,\mathbb E[V(x_n)]
=\sum_{t=0}^{n-1}\gamma^t\,\mathbb E_{\nu_t}\Bigl[\widehat c
+\gamma\bigl(\mathbb E_{\widehat P}-\mathbb E_{P_t}\bigr)\widehat V^{V}_{n-t-1}\Bigr].
\]
For the real process, split the discounted sum at $n$ and use linearity, which
absolute convergence allows, and the tower property for the first $n$ terms:
\[
J^\gamma_C(\pi)=\sum_{t=0}^{n-1}\gamma^t\,\mathbb E_{\nu_t}[c_t]
+\gamma^n\,\mathbb E[G_n].
\]
No Markov property is needed.  Subtracting the previous display from this one
gives the stated identity.  For the $\lambda$-return, with $H=H_{\mathrm{im}}$,
$\sum_{n=1}^{H}w_n=(1-\lambda)\sum_{n=1}^{H-1}\lambda^{n-1}
+\lambda^{H-1}=(1-\lambda^{H-1})+\lambda^{H-1}=1$, so
$J^\gamma_C-\widehat J^{\lambda}_C=\sum_nw_nJ^\gamma_C-\sum_nw_n\widehat
J^{(n)}_C=\sum_nw_n(J^\gamma_C-\widehat J^{(n)}_C)$.
\end{proof}

\noindent The implementation also multiplies each imagined step by a predicted
continuation probability.  Treating episode termination as an absorbing state
with zero cost and zero critic value, in both processes, folds that factor into
$\widehat P$ and $P_t$ and leaves the identity unchanged.  Because $P_t$ is
indexed by time, a time limit on episodes needs no further assumption; a model
that cannot see the time step simply carries the resulting error in its
transition term.  The corollary covers the $\lambda$-return from replay start
states; the advantages at later imagined steps start from model states and are
not covered.

\subsection{Factorized Transition Error}
\label{app:proof-factorized}

This subsection proves Proposition~\ref{prop:factorized}.  Write
$\mathcal X=\mathcal X^1\times\cdots\times\mathcal X^N$ with standard Borel
factors, $x^{\mathcal N_i}=(x^j)_{j\in\mathcal N_i}$, and likewise
$a^{\mathcal N_i}$.

\begin{proposition}[Factorized transition error]
\label{app:prop:factorized}
Suppose $\widehat P(\cdot\mid x,a)=\bigotimes_{i=1}^{N}\widehat P_i(\cdot\mid
x^{\mathcal N_i},a^{\mathcal N_i})$ for kernels $\widehat P_i$ into
$\mathcal X^i$, and let $P_{t,i}(\cdot\mid x,a)$ be the $i$-th marginal of
$P_t(\cdot\mid x,a)$.
\begin{enumerate}
\item[(i)] At every $(x,a)$,
\[
D_{\mathrm{TV}}(P_t,\widehat P)\leq D_{\mathrm{TV}}\Bigl(P_t,\textstyle\bigotimes_iP_{t,i}\Bigr)
+\sum_{i=1}^{N}D_{\mathrm{TV}}(P_{t,i},\widehat P_i).
\]
\item[(ii)] With
$\eta=\max_{t<H_{\mathrm{im}}}\mathbb E_{\nu_t}[D_{\mathrm{TV}}(P_t,\bigotimes_iP_{t,i})]$
and $\epsilon_P^i=\max_{t<H_{\mathrm{im}}}\mathbb E_{\nu_t}[D_{\mathrm{TV}}(P_{t,i},\widehat P_i)]$,
the smallest admissible $\epsilon_P$ of Theorem~\ref{app:thm:cost-transfer}
satisfies $\epsilon_P\leq\eta+\sum_i\epsilon_P^i$.
\item[(iii)] If $\mathcal N_i\subseteq\mathcal N_i'$, the infimum of
$\max_{t<H_{\mathrm{im}}}\mathbb E_{\nu_t}[D_{\mathrm{TV}}(P_{t,i},Q)]$ over
kernels $Q$ with inputs $(x^{\mathcal N_i'},a^{\mathcal N_i'})$ is at most the
same infimum over kernels with inputs $(x^{\mathcal N_i},a^{\mathcal N_i})$.
\item[(iv)] If the real and predicted team costs are the means of per-agent
costs, $c_t=\frac1N\sum_ic_t^i$ and $\widehat c=\frac1N\sum_i\widehat c^{\,i}$,
then $B_c(t)=\frac1N\sum_i\mathbb E_{\nu_t}[c_t^i-\widehat c^{\,i}]$.
\end{enumerate}
\end{proposition}

\begin{proof}
\emph{(i).}  Total variation is a metric on probability measures, so
$D_{\mathrm{TV}}(P_t,\widehat P)\leq D_{\mathrm{TV}}(P_t,\bigotimes_iP_{t,i})
+D_{\mathrm{TV}}(\bigotimes_iP_{t,i},\bigotimes_i\widehat P_i)$.  For the second
term, let $R_k=\widehat P_1\otimes\cdots\otimes\widehat P_k\otimes
P_{t,k+1}\otimes\cdots\otimes P_{t,N}$ for $k\in\{0,\ldots,N\}$, so that
$R_0=\bigotimes_iP_{t,i}$ and $R_N=\bigotimes_i\widehat P_i$.  By the triangle
inequality, $D_{\mathrm{TV}}(R_0,R_N)\leq\sum_{k=1}^{N}D_{\mathrm{TV}}(R_{k-1},R_k)$.
The measures $R_{k-1}$ and $R_k$ share every factor except the $k$-th, so
$R_{k-1}-R_k=\alpha\otimes(P_{t,k}-\widehat P_k)\otimes\beta$ with probability
measures $\alpha$ and $\beta$.  The total-variation norm of a product of a signed
measure with probability measures equals the norm of the signed measure (Fubini
applied to densities with respect to a common dominating measure), so
$D_{\mathrm{TV}}(R_{k-1},R_k)=D_{\mathrm{TV}}(P_{t,k},\widehat P_k)$.

\emph{(ii).}  Take expectations of (i) under $\nu_t$; by linearity the right side
becomes $\mathbb E_{\nu_t}[D_{\mathrm{TV}}(P_t,\bigotimes_iP_{t,i})]
+\sum_i\mathbb E_{\nu_t}[D_{\mathrm{TV}}(P_{t,i},\widehat P_i)]$.  Its maximum
over $t$ is at most the sum of the maxima of its terms,
$\eta+\sum_i\epsilon_P^i$.

\emph{(iii).}  A kernel with inputs $(x^{\mathcal N_i},a^{\mathcal N_i})$,
composed with the coordinate projection, is a kernel with inputs
$(x^{\mathcal N_i'},a^{\mathcal N_i'})$ that ignores the extra coordinates.  The
smaller class is therefore contained in the larger one, and an infimum over a
larger set is no larger.

\emph{(iv).}  Linearity of expectation.
\end{proof}

\noindent The product form holds for the model of Section~3 with the full and
self-only neighborhoods used in the paper.  With a general attention mask, each
of the encoder's three layers mixes information from one more hop, so
$\mathcal N_i$ must be replaced by its neighborhood within three hops.
Part~(iii) holds the state space fixed.  In the implementation, changing
$\mathcal N_i$ also changes what the deterministic states $h^i$ summarize, which
the statement does not cover; that is why the main text treats the effect of
communication on a trained model as an open experiment.

\subsection{Proof of the Measured-Cost Dual Statements}
\label{app:proof-dual}

\begin{proposition}[Measured-cost dual]
\label{app:prop:dual}
Let $\mathcal L(\phi,\beta)=J_R(\pi_\phi)-\beta(J_C(\pi_\phi)-d)$ and let
$\beta_{k+1}=[\beta_k+\alpha(\bar C_k-d)]_+$ with $\alpha>0$; in CA-MAMBA
$\alpha=E\alpha_\beta$ (Section~\ref{sec:actor-critic}).  Let $\mathcal F_k$ be
everything observed before the episode whose completion starts learner step $k$ began, and call
$T_\phi(\beta)=[\beta+\alpha(J_C(\pi_\phi)-d)]_+$ the mean dual map.
Then: \emph{(i)} the update is projected stochastic gradient \emph{descent} on
$\mathcal L$ in $\beta$, and is conditionally unbiased whenever
$\mathbb E[\bar C_k\mid\mathcal F_k]=J_C(\pi_{\phi_k})$.  In the implementation
$\bar C_k$ is the cost of the episode whose completion starts learner step $k$,
collected by one of $W$ workers with parameters $\phi_{k-\ell_k}$, and episodes
that finish during an initial warm-up start no learner step.  Suppose each
worker runs at most one episode at a time, starting its first at time zero, the
warm-up ends, at a time decided by the episodes received so far, within the
first $\bar k$ completions for a fixed $\bar k$, episode cost lies in
$[0,C_{\max}]$, and each episode's cost, given everything observed before it
starts, has mean $J_C$ of the parameters it starts with.  Then for every
$K\geq1$,
$\bigl|\sum_{k=0}^{K-1}\mathbb E[\bar C_k-J_C(\pi_{\phi_{k-\ell_k}})]\bigr|\leq2(W-1)C_{\max}$;
\emph{(ii)} for every $K\geq1$,
$\frac1K\sum_{k=0}^{K-1}(\bar C_k-d)\leq(\beta_K-\beta_0)/(\alpha K)$,
so $\beta_K/K\to0$ makes the average signed residual asymptotically
nonpositive; if moreover the assumptions of (i) hold and $\beta_K/K\to0$ almost
surely, then
$\limsup_{K}\frac1K\sum_{k=0}^{K-1}\mathbb E[J_C(\pi_{\phi_{k-\ell_k}})-d]\leq0$;
and \emph{(iii)} any finite fixed point $(\phi^\star,\beta^\star)$ of the mean
dual map, that is, any pair with $\beta^\star=T_{\phi^\star}(\beta^\star)$,
satisfies $J_C(\pi_{\phi^\star})\leq d$ and
$\beta^\star(J_C(\pi_{\phi^\star})-d)=0$.
\end{proposition}

\begin{proof}
For the true Lagrangian
\[
\mathcal{L}(\phi,\beta)
=
J_R(\pi_\phi)-\beta\bigl(J_C(\pi_\phi)-d\bigr),
\]
differentiation with respect to the scalar multiplier $\beta$ gives
\[
\frac{\partial\mathcal{L}(\phi,\beta)}{\partial\beta}
=
-\bigl(J_C(\pi_\phi)-d\bigr).
\]
Projected gradient descent with step size $\alpha>0$ therefore has the mean
update
\[
\beta^+
=
\left[
\beta-\alpha
\frac{\partial\mathcal{L}(\phi,\beta)}{\partial\beta}
\right]_+
=
\left[\beta+\alpha(J_C(\pi_\phi)-d)\right]_+
=T_\phi(\beta).
\]
If
\[
\mathbb{E}[\bar C_k\mid\mathcal{F}_k]
=J_C(\pi_{\phi_k}),
\]
then the measured-cost multiplier update is a conditionally unbiased stochastic
version of this dual step.  The actor's imagined objective omits the additive
term $\beta d$, but that term has zero derivative with respect to $\phi$ and
therefore does not alter the actor direction.

With $W$ workers this conditional statement fails: the learner receives
episodes in the order they finish, and an episode's finishing time can depend
on its cost, as when a team that loses every unit ends its episode early.  To
bound the effect, index episodes by the order in which they \emph{start}, and
let $C_{(j)}$ be the cost of the $j$-th started episode,
$J_{(j)}=J_C(\pi_{\phi_{(j)}})$ the expected cost of the parameters
$\phi_{(j)}$ it starts with, and $\mathcal G_j$ everything observed before
that start.  By assumption the difference $D_j=C_{(j)}-J_{(j)}$ satisfies
$\mathbb E[D_j\mid\mathcal G_j]=0$.  Hence, for any random index set
$\mathcal J\subseteq\{0,\ldots,\bar n-1\}$ with $\{j\in\mathcal J\}\in\mathcal
G_j$ for every $j$, conditioning each term on $\mathcal G_j$ gives
\[
\mathbb E\Bigl[\sum_{j\in\mathcal J}D_j\Bigr]
=\sum_{j<\bar n}\mathbb E\bigl[\mathbf 1\{j\in\mathcal J\}\,
\mathbb E[D_j\mid\mathcal G_j]\bigr]
=0.
\]
For $n\geq0$ let $\mathcal S_n$ be the episodes started before the learner
receives its $n$-th episode, $\mathcal E_n$ the $n$ episodes received up to and
including that one, and $\mathcal R_n=\mathcal S_n\setminus\mathcal E_n$
($\mathcal S_0=\mathcal E_0=\mathcal R_0=\emptyset$).  Each worker has at most one
episode in $\mathcal R_n$ and the worker whose episode is being received has none,
so $|\mathcal R_n|\leq W-1$ and $|\mathcal S_n|\leq n+W-1$; the episodes started
by any time are the first ones in start order.  Let $k_0\leq\bar k$ be the number
of warm-up completions.  Learner steps $0,\ldots,K-1$ receive
$\mathcal E_{k_0+K}\setminus\mathcal E_{k_0}$, so, with $J_e$ the expected cost of
the parameters episode $e$ starts with,
\[
\sum_{k=0}^{K-1}\bigl(\bar C_k-J_C(\pi_{\phi_{k-\ell_k}})\bigr)
=
\sum_{j\in\mathcal J}D_j
-\sum_{e\in\mathcal R_{k_0+K}}\bigl(C_e-J_e\bigr)
+\sum_{e\in\mathcal R_{k_0}}\bigl(C_e-J_e\bigr),
\]
where $\mathcal J=\mathcal S_{k_0+K}\setminus\mathcal S_{k_0}$.
Episode $j$ lies in $\mathcal J$ exactly when, at its start, the learner has
received at least $k_0$ and fewer than $k_0+K$ episodes; since the end of the
warm-up is decided by the episodes received so far, $\{j\in\mathcal J\}\in\mathcal
G_j$, and $\mathcal J\subseteq\{0,\ldots,\bar n-1\}$ with $\bar n:=\bar k+K+W-1$,
so the first sum has expectation zero.  Both $C_e$ and
$J_e$ lie in $[0,C_{\max}]$, so each of the at most $2(W-1)$ terms of the other
two sums lies in $[-C_{\max},C_{\max}]$, and taking expectations proves the
bound of (i).

To prove the pathwise result, use $[u]_+\geq u$ in every update:
\[
\beta_{k+1}
\geq
\beta_k+\alpha(\bar C_k-d).
\]
Because $\alpha>0$, rearrangement preserves the inequality:
\[
\bar C_k-d
\leq
\frac{\beta_{k+1}-\beta_k}{\alpha}.
\]
Summing from $k=0$ through $K-1$ retains both boundary terms and gives
\[
\sum_{k=0}^{K-1}(\bar C_k-d)
\leq
\frac{\beta_K-\beta_0}{\alpha}.
\]
Division by $K>0$ proves the pathwise signed-residual inequality.  If
$\beta_K/K\to0$, taking the upper limit on both sides shows that the average
signed empirical residual is asymptotically nonpositive.  For the expected
statement, episode cost lies in $[0,C_{\max}]$, so
$|\bar C_k-d|\leq M:=\max\{C_{\max},d\}$ and $0\leq\beta_K\leq\beta_0+\alpha KM$.
Then $0\leq\beta_K/K\leq\beta_0+\alpha M$ for $K\geq1$, and $\beta_K/K\to0$
almost surely gives $\mathbb{E}[\beta_K]/K\to0$ by dominated convergence.
Taking expectations in the pathwise inequality gives
$\frac1K\sum_{k<K}\mathbb{E}[\bar C_k-d]\leq(\mathbb{E}[\beta_K]-\beta_0)/(\alpha K)$,
whose upper limit is at most $0$.  By (i),
\[
\frac1K\sum_{k<K}\mathbb{E}\bigl[J_C(\pi_{\phi_{k-\ell_k}})-d\bigr]
\leq
\frac1K\sum_{k<K}\mathbb{E}\bigl[\bar C_k-d\bigr]+\frac{2(W-1)C_{\max}}{K},
\]
and the last term vanishes as $K\to\infty$, which proves the expected
statement.

For completeness, consider a finite fixed point $(\phi^\star,\beta^\star)$ of the deterministic mean dual map.  If $\beta^\star>0$, the positive branch of the projection must be active, and
\[
\beta^\star
=
\beta^\star+\alpha
\bigl(J_C(\pi_{\phi^\star})-d\bigr)
\]
implies $J_C(\pi_{\phi^\star})=d$.  If $\beta^\star=0$, the fixed-point equality
\[
0
=
\left[
\alpha\bigl(J_C(\pi_{\phi^\star})-d\bigr)
\right]_+
\]
implies $J_C(\pi_{\phi^\star})\leq d$.  Both cases establish
\[
J_C(\pi_{\phi^\star})\leq d,
\qquad
\beta^\star\bigl(J_C(\pi_{\phi^\star})-d\bigr)=0.
\]
These are the dual feasibility and complementary-slackness conditions, not a convergence or policy-optimality theorem.
\end{proof}

\subsection{Sparse-Cost Counterexample}
\label{app:counterexample}

Consider a one-step problem with two actions, $a_{\mathrm{s}}$ and
$a_{\mathrm{u}}$, and budget $d=0$.  The safe action has reward and true cost
\[
r(a_{\mathrm{s}})=0,
\qquad
c(a_{\mathrm{s}})=0,
\]
whereas the unsafe action has
\[
r(a_{\mathrm{u}})=1,
\qquad
c(a_{\mathrm{u}})=1.
\]
Suppose the learned reward is exact but the sparse unsafe event was not represented in the model data, so
\[
\widehat c(a_{\mathrm{s}})
=
\widehat c(a_{\mathrm{u}})
=0.
\]
For every finite multiplier $\beta\geq0$, the imagined actor objective assigns value $0$ to $a_{\mathrm{s}}$ and value $1$ to $a_{\mathrm{u}}$.  The actor therefore selects $a_{\mathrm{u}}$.  The real-cost dual update then becomes
\[
\beta_{k+1}
=
\beta_k+\alpha,
\]
so $\beta_k$ diverges while the policy remains unsafe.  Real-cost feedback detects the persistent violation, but it cannot create a corrective actor direction when the cost model assigns equal predicted cost to both actions.  On the action the policy actually takes, the cost head's bias is $B_c=c(a_{\mathrm{u}})-\widehat c(a_{\mathrm{u}})=1$, so $\epsilon_c\geq1$ and the margin of Theorem~\ref{app:thm:cost-transfer} correctly fails to certify the zero-cost budget.  The example shows that a cost head blind to a positive-cost event inflates $\epsilon_c$ by the missed cost, so the margin cannot certify a budget tighter than that cost; positive-cost recall on the states the policy visits is what keeps $\epsilon_c$ small.

\section{Full Casualty-Cost Results}
\label{app:full-results}

These tables give every SMAC configuration for all three methods and are
generated by \path{scripts/generate_aamas_tables.py}.  CA-MAMBA's rows use the
measured-cost multiplier with the settings of Appendix~\ref{app:settings}.  Every value is a training-time statistic with no separate
evaluation pass: for CA-MAMBA, the mean over the episodes a run logged between
95k and 100k interactions.  MACPO and MAPPO-Lagrangian log once per rollout
batch rather than per episode, so their values average wider windows; for
MAPPO-Lagrangian the window is 10k at 100k and up to 50k at later checkpoints.
MAPPO-Lagrangian's logged episode cost can exceed the number of allied units by
a few percent (for example $10.11$ on \texttt{MMM}, which has ten), and we
report it unchanged.  The script filters exact interaction checkpoints, rejects
non-finite or out-of-range metrics, computes sample standard deviations, and
asserts every aggregate claim used in the main text; seed counts below three
are starred.  Table~\ref{tab:main-results-aamas} restates the paired counts
quoted in the main text so they can be read against the per-configuration
tables that follow, and Table~\ref{tab:interaction-multiple} gives the
distribution of the interaction multiple.  Table~\ref{tab:full-results-aamas}
gives CA-MAMBA's per-configuration results at 100k.

\begin{table}[t]
\centering
\small
\setlength{\tabcolsep}{3pt}
\begin{tabular}{@{}lccc@{}}
\toprule
Baseline & \begin{tabular}[b]{@{}c@{}}CA-MAMBA\\lower cost\end{tabular} & \begin{tabular}[b]{@{}c@{}}CA-MAMBA\\higher win rate\end{tabular} & \begin{tabular}[b]{@{}c@{}}Baseline\\never wins\end{tabular} \\
\midrule
MACPO (100k) & 23/24 & 22/24 & 22/24 \\
MACPO (5M) & 16/24 & 17/24 & 16/24 \\
\midrule
MAPPO-Lag (100k) & 23/24 & 22/24 & 22/24 \\
MAPPO-Lag (1M) & 14/24 & 13/24 & 10/24 \\
MAPPO-Lag (5M) & 5/24 & 5/24 & 5/24 \\
\midrule
CA-MAMBA feasible & \multicolumn{3}{c}{5/24} \\
CA-MAMBA win rate $\geq 0.9$ & \multicolumn{3}{c}{10/24} \\
\bottomrule
\end{tabular}
\caption{CA-MAMBA after 100k interactions against MACPO and MAPPO-Lag at the listed interaction budgets, over the 24 SMAC scenario--limit configurations all methods ran. Counts compare unrounded seed means at the same configuration with strict inequalities; a baseline never wins when its mean win rate is zero, and CA-MAMBA is feasible when its mean cost is at most $d$.}
\label{tab:main-results-aamas}
\end{table}

\begin{table}[t]
\centering
\small
\begin{tabular}{@{}lc@{}}
\toprule
MAPPO-Lag budget $m(g)$ & Configurations \\
\midrule
$1\times$ (100k) & 0 \\
$5\times$ (500k) & 3 \\
$10\times$ (1M) & 7 \\
$20\times$ (2M) & 3 \\
$50\times$ (5M) & 5 \\
Not within $50\times$ & 5 \\
\midrule
Excluded ($w^{\mathrm{CA}}_g=0$) & 1 \\
Median $m(g)$ & $20\times$ \\
\bottomrule
\end{tabular}
\caption{The interaction multiple $m(g)$ of a SMAC configuration $g$ is the smallest multiple of CA-MAMBA's 100k interactions, among $1\times$, $5\times$, $10\times$, $20\times$, $50\times$, at which MAPPO-Lag's mean cost is at most and its mean win rate at least CA-MAMBA's at 100k, ties counting as matched. Configurations where CA-MAMBA never wins are excluded, and the median is over the remaining 23, counting unmatched ones as $+\infty$. Resampling runs within each configuration, the median exceeds $10\times$ in $99\%$ of $2000$ resamples and equals $20\times$ in $97\%$.}
\label{tab:interaction-multiple}
\end{table}

\begin{table*}[t]
\centering
\begin{tabular}{lcccc}
\toprule
Scenario & $d$ & Score $\uparrow$ & Cost $\downarrow$ & Win rate $\uparrow$ \\
\midrule
\texttt{1c3s5z} & 0 & $19.29 \mathbin{\pm} 0.14$ & $4.23 \mathbin{\pm} 0.27$ & $0.82 \mathbin{\pm} 0.02$ \\
\texttt{1c3s5z} & 4 & $19.15 \mathbin{\pm} 0.52$ & $4.83 \mathbin{\pm} 0.94$ & $0.79 \mathbin{\pm} 0.12$ \\
\texttt{2m\_vs\_1z} & 0 & $19.30 \mathbin{\pm} 0.49$ & $0.14 \mathbin{\pm} 0.08$ & $0.95 \mathbin{\pm} 0.04$ \\
\texttt{2m\_vs\_1z} & 1 & $19.34 \mathbin{\pm} 0.38$ & $0.14 \mathbin{\pm} 0.09$ & $0.95 \mathbin{\pm} 0.03$ \\
\texttt{2s\_vs\_1sc} & 0 & $19.95 \mathbin{\pm} 0.21$ & $0.23 \mathbin{\pm} 0.11$ & $0.96 \mathbin{\pm} 0.03$ \\
\texttt{2s\_vs\_1sc} & 1 & $19.56 \mathbin{\pm} 0.49$ & $0.31 \mathbin{\pm} 0.10$ & $0.92 \mathbin{\pm} 0.06$ \\
\texttt{2s3z} & 0 & $18.44 \mathbin{\pm} 0.25$ & $2.28 \mathbin{\pm} 0.38$ & $0.77 \mathbin{\pm} 0.04$ \\
\texttt{2s3z} & 4 & $17.37 \mathbin{\pm} 1.38$ & $2.98 \mathbin{\pm} 0.63$ & $0.63 \mathbin{\pm} 0.15$ \\
\texttt{3m} & 0 & $19.58 \mathbin{\pm} 0.25$ & $0.35 \mathbin{\pm} 0.14$ & $0.97 \mathbin{\pm} 0.02$ \\
\texttt{3m} & 1 & $19.46 \mathbin{\pm} 0.42$ & $0.60 \mathbin{\pm} 0.17$ & $0.96 \mathbin{\pm} 0.03$ \\
\texttt{3s\_vs\_3z} & 0 & $19.83 \mathbin{\pm} 0.20$ & $0.95 \mathbin{\pm} 0.18$ & $0.96 \mathbin{\pm} 0.02$ \\
\texttt{3s\_vs\_3z} & 1 & $19.87 \mathbin{\pm} 0.25^{\ast}$ & $0.92 \mathbin{\pm} 0.59^{\ast}$ & $0.95 \mathbin{\pm} 0.05^{\ast}$ \\
\texttt{3s\_vs\_4z} & 0 & $18.19 \mathbin{\pm} 1.58$ & $2.64 \mathbin{\pm} 0.34$ & $0.23 \mathbin{\pm} 0.16$ \\
\texttt{3s\_vs\_4z} & 1 & $16.96 \mathbin{\pm} 0.74$ & $2.47 \mathbin{\pm} 0.35$ & $0.31 \mathbin{\pm} 0.11$ \\
\texttt{3s\_vs\_5z} & 0 & $16.88 \mathbin{\pm} 1.58$ & $2.67 \mathbin{\pm} 0.40$ & $0.15 \mathbin{\pm} 0.19$ \\
\texttt{3s\_vs\_5z} & 1 & $17.44 \mathbin{\pm} 1.49$ & $2.67 \mathbin{\pm} 0.43$ & $0.17 \mathbin{\pm} 0.22$ \\
\texttt{3s5z\_vs\_3s6z} & 0 & $9.60 \mathbin{\pm} 0.84$ & $8.00 \mathbin{\pm} 0.00$ & $0.00 \mathbin{\pm} 0.00$ \\
\texttt{3s5z\_vs\_3s6z} & 4 & $13.28 \mathbin{\pm} 1.10$ & $7.54 \mathbin{\pm} 0.40$ & $0.11 \mathbin{\pm} 0.06$ \\
\texttt{8m} & 0 & $19.58 \mathbin{\pm} 0.14$ & $2.05 \mathbin{\pm} 0.42$ & $0.95 \mathbin{\pm} 0.02$ \\
\texttt{8m} & 4 & $17.26 \mathbin{\pm} 1.04$ & $5.03 \mathbin{\pm} 0.42$ & $0.69 \mathbin{\pm} 0.10$ \\
\texttt{MMM} & 0 & $17.05 \mathbin{\pm} 2.36$ & $6.77 \mathbin{\pm} 1.79$ & $0.60 \mathbin{\pm} 0.25$ \\
\texttt{MMM} & 4 & $17.02 \mathbin{\pm} 1.32$ & $7.88 \mathbin{\pm} 0.16$ & $0.55 \mathbin{\pm} 0.07$ \\
\texttt{bane\_vs\_bane} & 0 & $17.75 \mathbin{\pm} 2.82$ & $13.91 \mathbin{\pm} 7.11$ & $0.62 \mathbin{\pm} 0.39$ \\
\texttt{bane\_vs\_bane} & 4 & $19.91 \mathbin{\pm} 0.11$ & $6.37 \mathbin{\pm} 1.20$ & $0.98 \mathbin{\pm} 0.02$ \\
\bottomrule
\end{tabular}
\caption{CA-MAMBA results at exactly 100k interactions on the 24 SMAC configurations shared with the baselines, with the multiplier updated from the measured episode cost. Score is SMAC's episode return, which credits damage and kills as well as the win (near $20$ for a clean win). Entries are mean $\mathbin{\pm}$ sample standard deviation across seeds; CA-MAMBA uses three valid seeds except the starred configuration, which uses two.}
\label{tab:full-results-aamas}
\end{table*}

Table~\ref{tab:smac-mappolag} adds MAPPO-Lagrangian on the 24 shared
configurations at 100k, 1M, and 5M interactions, generated by
\path{scripts/generate_aamas_tables.py} from per-seed files.  With the 500k and 2M
checkpoints of the same files, it is the source of the MAPPO-Lagrangian counts and of the
interaction multiples in the main text.

\begin{table*}[t]
\centering
\small
\setlength{\tabcolsep}{3pt}
\begin{tabular}{@{}lccccccccc@{}}
\toprule
 &  & \multicolumn{2}{c}{CA-MAMBA (100k)} & \multicolumn{2}{c}{MAPPO-Lag (100k)} & \multicolumn{2}{c}{MAPPO-Lag (1M)} & \multicolumn{2}{c}{MAPPO-Lag (5M)} \\
\cmidrule(lr){3-4} \cmidrule(lr){5-6} \cmidrule(lr){7-8} \cmidrule(lr){9-10}
Scenario & $d$ & Cost $\downarrow$ & Win $\uparrow$ & Cost $\downarrow$ & Win $\uparrow$ & Cost $\downarrow$ & Win $\uparrow$ & Cost $\downarrow$ & Win $\uparrow$ \\
\midrule
\texttt{1c3s5z} & 0 & $4.23 \mathbin{\pm} 0.27$ & $0.82 \mathbin{\pm} 0.02$ & $9.00 \mathbin{\pm} 0.00$ & $0.00 \mathbin{\pm} 0.00$ & $2.22 \mathbin{\pm} 0.64$ & $0.96 \mathbin{\pm} 0.04$ & $0.44 \mathbin{\pm} 0.11$ & $1.00 \mathbin{\pm} 0.00$ \\
\texttt{1c3s5z} & 4 & $4.83 \mathbin{\pm} 0.94$ & $0.79 \mathbin{\pm} 0.12$ & $9.00 \mathbin{\pm} 0.00$ & $0.00 \mathbin{\pm} 0.00$ & $2.08 \mathbin{\pm} 0.66$ & $0.96 \mathbin{\pm} 0.04$ & $0.37 \mathbin{\pm} 0.14$ & $1.00 \mathbin{\pm} 0.00$ \\
\texttt{2m\_vs\_1z} & 0 & $0.14 \mathbin{\pm} 0.08$ & $0.95 \mathbin{\pm} 0.04$ & $2.00 \mathbin{\pm} 0.00$ & $0.00 \mathbin{\pm} 0.00$ & $0.01 \mathbin{\pm} 0.00$ & $1.00 \mathbin{\pm} 0.00$ & $0.00 \mathbin{\pm} 0.00^{\dagger}$ & $1.00 \mathbin{\pm} 0.00^{\dagger}$ \\
\texttt{2m\_vs\_1z} & 1 & $0.14 \mathbin{\pm} 0.09$ & $0.95 \mathbin{\pm} 0.03$ & $2.00 \mathbin{\pm} 0.00$ & $0.00 \mathbin{\pm} 0.00$ & $0.01 \mathbin{\pm} 0.00$ & $1.00 \mathbin{\pm} 0.00$ & $0.00 \mathbin{\pm} 0.00$ & $1.00 \mathbin{\pm} 0.00$ \\
\texttt{2s\_vs\_1sc} & 0 & $0.23 \mathbin{\pm} 0.11$ & $0.96 \mathbin{\pm} 0.03$ & $2.00 \mathbin{\pm} 0.00$ & $0.00 \mathbin{\pm} 0.00$ & $0.21 \mathbin{\pm} 0.01$ & $1.00 \mathbin{\pm} 0.00$ & $0.11 \mathbin{\pm} 0.07$ & $1.00 \mathbin{\pm} 0.00$ \\
\texttt{2s\_vs\_1sc} & 1 & $0.31 \mathbin{\pm} 0.10$ & $0.92 \mathbin{\pm} 0.06$ & $2.00 \mathbin{\pm} 0.00$ & $0.00 \mathbin{\pm} 0.00$ & $0.18 \mathbin{\pm} 0.04$ & $1.00 \mathbin{\pm} 0.00$ & $0.10 \mathbin{\pm} 0.03$ & $1.00 \mathbin{\pm} 0.00$ \\
\texttt{2s3z} & 0 & $2.28 \mathbin{\pm} 0.38$ & $0.77 \mathbin{\pm} 0.04$ & $5.00 \mathbin{\pm} 0.00$ & $0.00 \mathbin{\pm} 0.00$ & $2.88 \mathbin{\pm} 1.86$ & $0.59 \mathbin{\pm} 0.51$ & $0.28 \mathbin{\pm} 0.15$ & $0.99 \mathbin{\pm} 0.01$ \\
\texttt{2s3z} & 4 & $2.98 \mathbin{\pm} 0.63$ & $0.63 \mathbin{\pm} 0.15$ & $5.00 \mathbin{\pm} 0.00$ & $0.00 \mathbin{\pm} 0.00$ & $2.26 \mathbin{\pm} 1.10$ & $0.74 \mathbin{\pm} 0.30$ & $0.18 \mathbin{\pm} 0.06$ & $0.99 \mathbin{\pm} 0.01$ \\
\texttt{3m} & 0 & $0.35 \mathbin{\pm} 0.14$ & $0.97 \mathbin{\pm} 0.02$ & $3.00 \mathbin{\pm} 0.00$ & $0.00 \mathbin{\pm} 0.00$ & $0.16 \mathbin{\pm} 0.13$ & $0.99 \mathbin{\pm} 0.00$ & $0.02 \mathbin{\pm} 0.01$ & $1.00 \mathbin{\pm} 0.00$ \\
\texttt{3m} & 1 & $0.60 \mathbin{\pm} 0.17$ & $0.96 \mathbin{\pm} 0.03$ & $3.00 \mathbin{\pm} 0.00$ & $0.00 \mathbin{\pm} 0.00$ & $0.05 \mathbin{\pm} 0.01$ & $1.00 \mathbin{\pm} 0.00$ & $0.02 \mathbin{\pm} 0.00$ & $1.00 \mathbin{\pm} 0.00$ \\
\texttt{3s\_vs\_3z} & 0 & $0.95 \mathbin{\pm} 0.18$ & $0.96 \mathbin{\pm} 0.02$ & $3.00 \mathbin{\pm} 0.00$ & $0.00 \mathbin{\pm} 0.00$ & $3.02 \mathbin{\pm} 0.00$ & $0.00 \mathbin{\pm} 0.00$ & $0.01 \mathbin{\pm} 0.00$ & $1.00 \mathbin{\pm} 0.00$ \\
\texttt{3s\_vs\_3z} & 1 & $0.92 \mathbin{\pm} 0.59^{\ast}$ & $0.95 \mathbin{\pm} 0.05^{\ast}$ & $3.00 \mathbin{\pm} 0.00$ & $0.00 \mathbin{\pm} 0.00$ & $3.01 \mathbin{\pm} 0.00$ & $0.00 \mathbin{\pm} 0.00$ & $0.01 \mathbin{\pm} 0.01$ & $1.00 \mathbin{\pm} 0.00$ \\
\texttt{3s\_vs\_4z} & 0 & $2.64 \mathbin{\pm} 0.34$ & $0.23 \mathbin{\pm} 0.16$ & $3.00 \mathbin{\pm} 0.00$ & $0.00 \mathbin{\pm} 0.00$ & $3.01 \mathbin{\pm} 0.01$ & $0.00 \mathbin{\pm} 0.00$ & $3.00 \mathbin{\pm} 0.00$ & $0.00 \mathbin{\pm} 0.00$ \\
\texttt{3s\_vs\_4z} & 1 & $2.47 \mathbin{\pm} 0.35$ & $0.31 \mathbin{\pm} 0.11$ & $3.00 \mathbin{\pm} 0.00$ & $0.00 \mathbin{\pm} 0.00$ & $3.02 \mathbin{\pm} 0.01$ & $0.00 \mathbin{\pm} 0.00$ & $3.00 \mathbin{\pm} 0.00$ & $0.00 \mathbin{\pm} 0.00$ \\
\texttt{3s\_vs\_5z} & 0 & $2.67 \mathbin{\pm} 0.40$ & $0.15 \mathbin{\pm} 0.19$ & $3.00 \mathbin{\pm} 0.00$ & $0.00 \mathbin{\pm} 0.00$ & $3.02 \mathbin{\pm} 0.01$ & $0.00 \mathbin{\pm} 0.00$ & $3.01 \mathbin{\pm} 0.01$ & $0.00 \mathbin{\pm} 0.00$ \\
\texttt{3s\_vs\_5z} & 1 & $2.67 \mathbin{\pm} 0.43$ & $0.17 \mathbin{\pm} 0.22$ & $3.00 \mathbin{\pm} 0.00$ & $0.00 \mathbin{\pm} 0.00$ & $3.02 \mathbin{\pm} 0.02$ & $0.00 \mathbin{\pm} 0.00$ & $3.00 \mathbin{\pm} 0.00$ & $0.00 \mathbin{\pm} 0.00$ \\
\texttt{3s5z\_vs\_3s6z} & 0 & $8.00 \mathbin{\pm} 0.00$ & $0.00 \mathbin{\pm} 0.00$ & $8.00 \mathbin{\pm} 0.00$ & $0.00 \mathbin{\pm} 0.00$ & $8.07 \mathbin{\pm} 0.02$ & $0.00 \mathbin{\pm} 0.00$ & $8.00 \mathbin{\pm} 0.00$ & $0.00 \mathbin{\pm} 0.00$ \\
\texttt{3s5z\_vs\_3s6z} & 4 & $7.54 \mathbin{\pm} 0.40$ & $0.11 \mathbin{\pm} 0.06$ & $8.00 \mathbin{\pm} 0.00$ & $0.00 \mathbin{\pm} 0.00$ & $8.08 \mathbin{\pm} 0.00$ & $0.00 \mathbin{\pm} 0.00$ & $8.00 \mathbin{\pm} 0.00$ & $0.00 \mathbin{\pm} 0.00$ \\
\texttt{8m} & 0 & $2.05 \mathbin{\pm} 0.42$ & $0.95 \mathbin{\pm} 0.02$ & $8.00 \mathbin{\pm} 0.00$ & $0.00 \mathbin{\pm} 0.00$ & $5.67 \mathbin{\pm} 0.46$ & $0.61 \mathbin{\pm} 0.08$ & $0.90 \mathbin{\pm} 0.17$ & $1.00 \mathbin{\pm} 0.00$ \\
\texttt{8m} & 4 & $5.03 \mathbin{\pm} 0.42$ & $0.69 \mathbin{\pm} 0.10$ & $8.00 \mathbin{\pm} 0.00$ & $0.00 \mathbin{\pm} 0.00$ & $5.96 \mathbin{\pm} 1.20$ & $0.53 \mathbin{\pm} 0.25$ & $1.10 \mathbin{\pm} 0.36$ & $0.99 \mathbin{\pm} 0.00$ \\
\texttt{MMM} & 0 & $6.77 \mathbin{\pm} 1.79$ & $0.60 \mathbin{\pm} 0.25$ & $10.00 \mathbin{\pm} 0.00$ & $0.00 \mathbin{\pm} 0.00$ & $10.11 \mathbin{\pm} 0.01$ & $0.00 \mathbin{\pm} 0.00$ & $0.97 \mathbin{\pm} 0.61$ & $0.99 \mathbin{\pm} 0.02$ \\
\texttt{MMM} & 4 & $7.88 \mathbin{\pm} 0.16$ & $0.55 \mathbin{\pm} 0.07$ & $10.00 \mathbin{\pm} 0.00$ & $0.00 \mathbin{\pm} 0.00$ & $10.11 \mathbin{\pm} 0.02$ & $0.00 \mathbin{\pm} 0.00$ & $1.41 \mathbin{\pm} 1.07$ & $0.97 \mathbin{\pm} 0.03$ \\
\texttt{bane\_vs\_bane} & 0 & $13.91 \mathbin{\pm} 7.11$ & $0.62 \mathbin{\pm} 0.39$ & $21.36 \mathbin{\pm} 0.10$ & $0.89 \mathbin{\pm} 0.01$ & $10.28 \mathbin{\pm} 0.40$ & $0.96 \mathbin{\pm} 0.00$ & $5.17 \mathbin{\pm} 0.19^{\dagger}$ & $1.00 \mathbin{\pm} 0.00^{\dagger}$ \\
\texttt{bane\_vs\_bane} & 4 & $6.37 \mathbin{\pm} 1.20$ & $0.98 \mathbin{\pm} 0.02$ & $21.46 \mathbin{\pm} 0.20$ & $0.87 \mathbin{\pm} 0.02$ & $9.13 \mathbin{\pm} 1.18$ & $0.98 \mathbin{\pm} 0.01$ & $5.22 \mathbin{\pm} 0.17$ & $1.00 \mathbin{\pm} 0.00$ \\
\bottomrule
\end{tabular}
\caption{CA-MAMBA after 100k environment interactions against MAPPO-Lag after 100k, 1M and 5M interactions on the 24 SMAC scenario--limit configurations both methods ran. Read each row across: at the same scenario and limit $d$, lower cost and higher win rate are better, and a method is feasible when its mean cost is at most $d$. Entries are mean $\mathbin{\pm}$ sample standard deviation across seeds. MAPPO-Lag uses three seeds at every checkpoint; CA-MAMBA uses three valid seeds except the starred configuration, which uses two. Each daggered entry includes one run that stopped before 5M and contributes its last logged window (\texttt{2m\_vs\_1z} at $d=0$, seed 1, stopped at 3.01M; \texttt{bane\_vs\_bane} at $d=0$, seed 1, stopped at 4.52M).}
\label{tab:smac-mappolag}
\end{table*}

Table~\ref{tab:smac-macpo} gives MACPO on the same configurations at 100k and
5M interactions, the source of the MACPO counts in the main text.

\begin{table*}[t]
\centering
\small
\setlength{\tabcolsep}{3pt}
\begin{tabular}{@{}lccccccc@{}}
\toprule
 &  & \multicolumn{2}{c}{CA-MAMBA (100k)} & \multicolumn{2}{c}{MACPO (100k)} & \multicolumn{2}{c}{MACPO (5M)} \\
\cmidrule(lr){3-4} \cmidrule(lr){5-6} \cmidrule(lr){7-8}
Scenario & $d$ & Cost $\downarrow$ & Win $\uparrow$ & Cost $\downarrow$ & Win $\uparrow$ & Cost $\downarrow$ & Win $\uparrow$ \\
\midrule
\texttt{1c3s5z} & 0 & $4.23 \mathbin{\pm} 0.27$ & $0.82 \mathbin{\pm} 0.02$ & $9.00 \mathbin{\pm} 0.00$ & $0.00 \mathbin{\pm} 0.00$ & $9.00 \mathbin{\pm} 0.00^{\dagger}$ & $0.00 \mathbin{\pm} 0.00^{\dagger}$ \\
\texttt{1c3s5z} & 4 & $4.83 \mathbin{\pm} 0.94$ & $0.79 \mathbin{\pm} 0.12$ & $9.00 \mathbin{\pm} 0.00$ & $0.00 \mathbin{\pm} 0.00$ & $9.00 \mathbin{\pm} 0.00^{\dagger}$ & $0.00 \mathbin{\pm} 0.00^{\dagger}$ \\
\texttt{2m\_vs\_1z} & 0 & $0.14 \mathbin{\pm} 0.08$ & $0.95 \mathbin{\pm} 0.04$ & $2.00 \mathbin{\pm} 0.00$ & $0.00 \mathbin{\pm} 0.00$ & $0.02 \mathbin{\pm} 0.02$ & $0.99 \mathbin{\pm} 0.01$ \\
\texttt{2m\_vs\_1z} & 1 & $0.14 \mathbin{\pm} 0.09$ & $0.95 \mathbin{\pm} 0.03$ & $2.00 \mathbin{\pm} 0.00$ & $0.00 \mathbin{\pm} 0.00$ & $0.00 \mathbin{\pm} 0.00$ & $1.00 \mathbin{\pm} 0.00$ \\
\texttt{2s\_vs\_1sc} & 0 & $0.23 \mathbin{\pm} 0.11$ & $0.96 \mathbin{\pm} 0.03$ & $2.00 \mathbin{\pm} 0.00$ & $0.00 \mathbin{\pm} 0.00$ & $2.00 \mathbin{\pm} 0.00$ & $0.00 \mathbin{\pm} 0.00$ \\
\texttt{2s\_vs\_1sc} & 1 & $0.31 \mathbin{\pm} 0.10$ & $0.92 \mathbin{\pm} 0.06$ & $2.00 \mathbin{\pm} 0.00$ & $0.00 \mathbin{\pm} 0.00$ & $0.16 \mathbin{\pm} 0.03$ & $0.99 \mathbin{\pm} 0.01$ \\
\texttt{2s3z} & 0 & $2.28 \mathbin{\pm} 0.38$ & $0.77 \mathbin{\pm} 0.04$ & $5.00 \mathbin{\pm} 0.00$ & $0.00 \mathbin{\pm} 0.00$ & $5.00 \mathbin{\pm} 0.00$ & $0.00 \mathbin{\pm} 0.00$ \\
\texttt{2s3z} & 4 & $2.98 \mathbin{\pm} 0.63$ & $0.63 \mathbin{\pm} 0.15$ & $5.00 \mathbin{\pm} 0.00$ & $0.00 \mathbin{\pm} 0.00$ & $1.94 \mathbin{\pm} 0.28$ & $0.93 \mathbin{\pm} 0.04$ \\
\texttt{3m} & 0 & $0.35 \mathbin{\pm} 0.14$ & $0.97 \mathbin{\pm} 0.02$ & $3.00 \mathbin{\pm} 0.00$ & $0.00 \mathbin{\pm} 0.00$ & $3.00 \mathbin{\pm} 0.00$ & $0.00 \mathbin{\pm} 0.00$ \\
\texttt{3m} & 1 & $0.60 \mathbin{\pm} 0.17$ & $0.96 \mathbin{\pm} 0.03$ & $3.00 \mathbin{\pm} 0.00$ & $0.00 \mathbin{\pm} 0.00$ & $0.27 \mathbin{\pm} 0.29$ & $1.00 \mathbin{\pm} 0.00$ \\
\texttt{3s\_vs\_3z} & 0 & $0.95 \mathbin{\pm} 0.18$ & $0.96 \mathbin{\pm} 0.02$ & $3.00 \mathbin{\pm} 0.00$ & $0.00 \mathbin{\pm} 0.00$ & $3.00 \mathbin{\pm} 0.00$ & $0.00 \mathbin{\pm} 0.00$ \\
\texttt{3s\_vs\_3z} & 1 & $0.92 \mathbin{\pm} 0.59^{\ast}$ & $0.95 \mathbin{\pm} 0.05^{\ast}$ & $3.00 \mathbin{\pm} 0.00$ & $0.00 \mathbin{\pm} 0.00$ & $0.82 \mathbin{\pm} 0.86$ & $0.82 \mathbin{\pm} 0.22$ \\
\texttt{3s\_vs\_4z} & 0 & $2.64 \mathbin{\pm} 0.34$ & $0.23 \mathbin{\pm} 0.16$ & $3.00 \mathbin{\pm} 0.00$ & $0.00 \mathbin{\pm} 0.00$ & $3.00 \mathbin{\pm} 0.00$ & $0.00 \mathbin{\pm} 0.00$ \\
\texttt{3s\_vs\_4z} & 1 & $2.47 \mathbin{\pm} 0.35$ & $0.31 \mathbin{\pm} 0.11$ & $3.00 \mathbin{\pm} 0.00^{\ast}$ & $0.00 \mathbin{\pm} 0.00^{\ast}$ & $3.00 \mathbin{\pm} 0.00^{\ast}$ & $0.00 \mathbin{\pm} 0.00^{\ast}$ \\
\texttt{3s\_vs\_5z} & 0 & $2.67 \mathbin{\pm} 0.40$ & $0.15 \mathbin{\pm} 0.19$ & $3.00 \mathbin{\pm} 0.00$ & $0.00 \mathbin{\pm} 0.00$ & $3.00 \mathbin{\pm} 0.00^{\dagger}$ & $0.00 \mathbin{\pm} 0.00^{\dagger}$ \\
\texttt{3s\_vs\_5z} & 1 & $2.67 \mathbin{\pm} 0.43$ & $0.17 \mathbin{\pm} 0.22$ & $3.00 \mathbin{\pm} 0.00$ & $0.00 \mathbin{\pm} 0.00$ & $3.00 \mathbin{\pm} 0.00^{\dagger}$ & $0.00 \mathbin{\pm} 0.00^{\dagger}$ \\
\texttt{3s5z\_vs\_3s6z} & 0 & $8.00 \mathbin{\pm} 0.00$ & $0.00 \mathbin{\pm} 0.00$ & $8.00 \mathbin{\pm} 0.00$ & $0.00 \mathbin{\pm} 0.00$ & $8.00 \mathbin{\pm} 0.00^{\dagger}$ & $0.00 \mathbin{\pm} 0.00^{\dagger}$ \\
\texttt{3s5z\_vs\_3s6z} & 4 & $7.54 \mathbin{\pm} 0.40$ & $0.11 \mathbin{\pm} 0.06$ & $8.00 \mathbin{\pm} 0.00$ & $0.00 \mathbin{\pm} 0.00$ & $8.00 \mathbin{\pm} 0.00^{\dagger}$ & $0.00 \mathbin{\pm} 0.00^{\dagger}$ \\
\texttt{8m} & 0 & $2.05 \mathbin{\pm} 0.42$ & $0.95 \mathbin{\pm} 0.02$ & $8.00 \mathbin{\pm} 0.00$ & $0.00 \mathbin{\pm} 0.00$ & $8.00 \mathbin{\pm} 0.00$ & $0.00 \mathbin{\pm} 0.00$ \\
\texttt{8m} & 4 & $5.03 \mathbin{\pm} 0.42$ & $0.69 \mathbin{\pm} 0.10$ & $8.00 \mathbin{\pm} 0.00$ & $0.00 \mathbin{\pm} 0.00$ & $8.00 \mathbin{\pm} 0.00$ & $0.00 \mathbin{\pm} 0.00$ \\
\texttt{MMM} & 0 & $6.77 \mathbin{\pm} 1.79$ & $0.60 \mathbin{\pm} 0.25$ & $10.00 \mathbin{\pm} 0.00$ & $0.00 \mathbin{\pm} 0.00$ & $10.00 \mathbin{\pm} 0.00$ & $0.00 \mathbin{\pm} 0.00$ \\
\texttt{MMM} & 4 & $7.88 \mathbin{\pm} 0.16$ & $0.55 \mathbin{\pm} 0.07$ & $10.00 \mathbin{\pm} 0.00$ & $0.00 \mathbin{\pm} 0.00$ & $10.00 \mathbin{\pm} 0.00$ & $0.00 \mathbin{\pm} 0.00$ \\
\texttt{bane\_vs\_bane} & 0 & $13.91 \mathbin{\pm} 7.11$ & $0.62 \mathbin{\pm} 0.39$ & $22.55 \mathbin{\pm} 0.02$ & $0.79 \mathbin{\pm} 0.02$ & $7.79 \mathbin{\pm} 0.59$ & $0.99 \mathbin{\pm} 0.01$ \\
\texttt{bane\_vs\_bane} & 4 & $6.37 \mathbin{\pm} 1.20$ & $0.98 \mathbin{\pm} 0.02$ & $22.40 \mathbin{\pm} 0.08$ & $0.80 \mathbin{\pm} 0.02$ & $11.80 \mathbin{\pm} 3.67^{\dagger}$ & $0.96 \mathbin{\pm} 0.04^{\dagger}$ \\
\bottomrule
\end{tabular}
\caption{CA-MAMBA after 100k environment interactions against MACPO after 100k and 5M interactions on the 24 SMAC scenario--limit configurations all methods ran. Read each row across: at the same scenario and limit $d$, lower cost and higher win rate are better. Entries are mean $\mathbin{\pm}$ sample standard deviation over three seeds, or two where starred. Daggered MACPO (5M) entries include runs that stopped before 5M, whose tracked values predate 5M (\texttt{1c3s5z} at $d=0$, seed 1, last logged at 2.155M; \texttt{1c3s5z} at $d=4$, seed 2, last logged at 2.025M; \texttt{3s\_vs\_5z} at $d=0$, seed 1, last logged at 4.25M; \texttt{3s\_vs\_5z} at $d=0$, seed 2, last logged at 4.28M; \texttt{3s\_vs\_5z} at $d=1$, seed 1, last logged at 4.05M; \texttt{3s5z\_vs\_3s6z} at $d=0$, seed 2, last logged at 2.55M; \texttt{3s5z\_vs\_3s6z} at $d=4$, seed 2, last logged at 2.295M; \texttt{bane\_vs\_bane} at $d=4$, seed 2, last logged at 0.735M; \texttt{bane\_vs\_bane} at $d=4$, seed 3, last logged at 1.29M).}
\label{tab:smac-macpo}
\end{table*}

\FloatBarrier

\section{Baseline Hyperparameters}
\label{app:baseline-hyperparams}

Table~\ref{tab:macpo-hyperparams} lists the hyperparameters of the
MACPO baseline on SMAC: the defaults of its code base and the values
our SMAC runs override.

\par\medskip\noindent
\begin{minipage}{\columnwidth}
\makeatletter\def\@captype{table}\makeatother
\centering
\small
\caption{Hyperparameters of the MACPO baseline on SMAC.  SMAC-specific overrides are listed separately; all other parameters use the defaults shown.}
\label{tab:macpo-hyperparams}
\begin{tabular}{@{}lcc@{}}
\toprule
Parameter & Default & SMAC override \\
\midrule
\multicolumn{3}{@{}l}{\emph{Network}} \\
Hidden size & 512 & 128 \\
Layers & 2 & 1 \\
Activation & ReLU & -- \\
Orthogonal init (gain) & 0.01 & -- \\
Feature normalization & yes & -- \\
\midrule
\multicolumn{3}{@{}l}{\emph{Optimisation}} \\
Actor lr & $9 \times 10^{-5}$ & -- \\
Critic lr & $5 \times 10^{-3}$ & -- \\
Adam $\epsilon$ & $10^{-5}$ & -- \\
Weight decay & 0 & -- \\
Learning iters / update & 5 & 15 \\
Mini-batches & 1 & -- \\
Max grad norm & 10 & -- \\
\midrule
\multicolumn{3}{@{}l}{\emph{RL}} \\
$\gamma$ & 0.96 & 0.99 \\
GAE $\lambda$ & 0.95 & -- \\
Entropy coeff & 0.0 & 0.01 \\
Clip $\epsilon$ & 0.2 & -- \\
Value loss coeff & 1 & -- \\
\midrule
\multicolumn{3}{@{}l}{\emph{Safety / MACPO}} \\
Cost limit & 0 & per configuration ($d$) \\
Safety $\gamma$ & 0.09 & 0.2 \\
Target KL & 0.016 & 0.01 \\
Searching steps & 10 & -- \\
CG iterations & 10 & -- \\
Accept ratio & 0.5 & -- \\
Step fraction & 0.5 & -- \\
$g$-step dir coeff & 0.1 & -- \\
$b$-step dir coeff & 0.1 & -- \\
Fraction coeff & 0.1 & -- \\
\midrule
\multicolumn{3}{@{}l}{\emph{Rollout}} \\
Rollout length per update & 8 & 1000 \\
Rollout threads & 5 & 5 \\
Env steps & $10^{8}$ & $10^{7}$ \\
\bottomrule
\end{tabular}
\end{minipage}\par

\FloatBarrier

\end{document}